\documentclass[twocolumn,envcountsect,envcountsame]{svjour3}    
\usepackage[utf8]{inputenc}

\smartqed  
\usepackage{graphicx}
\usepackage{soul}
\usepackage{natbib}
\usepackage[colorlinks,linkcolor=blue,citecolor=blue,urlcolor=blue]{hyperref}
\usepackage{times}
\usepackage[mathscr]{eucal}
\usepackage{amsmath,amsfonts,amssymb,bbm}
\usepackage{graphicx,epsfig,subfigure,cases,multirow}
\usepackage[linesnumbered,ruled,vlined,commentsnumbered]{algorithm2e}
\usepackage{enumerate}
\usepackage[title]{appendix}
\usepackage[displaymath,mathlines,switch]{lineno}
\usepackage{booktabs}
\usepackage{todonotes}

\makeatletter 
\def\cl@chapter{\@elt {chapter}}
\makeatother
\usepackage[capitalize]{cleveref}

\makeatletter
\newcommand{\labelx}[1]{
    \relax
    \ifmmode
        \label{#1} 
    \else 
        \ifnum\pdfstrcmp{\@currenvir}{document}=0
            \label{#1}
        \else
            \label[\@currenvir]{#1}
        \fi
    \fi
}
\makeatother

\newlength\savewidth
\newcommand\shline{\noalign{\global\savewidth\arrayrulewidth\global\arrayrulewidth 1.0pt}\hline\noalign{\global\arrayrulewidth\savewidth}}

\newlength\savedwidth

\newcommand{\alphareg}{\varsigma}
\newcommand{\bS}{\mathbb S}
\newcommand{\bR}{\mathbb R}
\newcommand{\bZ}{\mathbb Z}
\newcommand{\bD}{\mathbb D}

\newcommand{\bM}{\mathbb M}
\newcommand{\bT}{\mathbb T}

\newcommand{\rW}{\mathscr W}
\newcommand{\rV}{\mathscr V}
\newcommand{\rD}{\mathscr D}
\newcommand{\rG}{\mathscr G}
\newcommand{\rT}{\mathscr T}

\newcommand{\cI}{\mathcal I}
\newcommand{\cM}{\mathcal M}

\newcommand{\cX}{\mathcal X}
\newcommand{\cL}{\mathcal L}
\newcommand{\cC}{\mathcal C}
\newcommand{\cF}{\mathcal F}
\newcommand{\cR}{\mathcal R}

\newcommand{\cG}{\mathcal G}
\newcommand{\cP}{\mathcal P}

\newcommand{\cZ}{\mathcal Z}
\newcommand{\cK}{\mathcal K}
\newcommand{\cT}{\mathcal T}
\newcommand{\cH}{\mathcal H}
\newcommand{\cD}{\mathfrak D}
\newcommand{\cS}{\mathcal S}
\newcommand{\cO}{\mathcal O}
\newcommand{\cQ}{\mathcal Q}

\newcommand{\kb}{\mathfrak b}
\newcommand{\kc}{\mathfrak c}

\newcommand{\kg}{\mathfrak g}
\newcommand{\kF}{\mathfrak F}
\newcommand{\kL}{\mathfrak L}

\newcommand{\fH}{\mathbf H}
\newcommand{\fI}{\mathbf I}
\newcommand{\fp}{\mathbf p}
\newcommand{\fs}{\mathbf s}
\newcommand{\fx}{\mathbf x}
\newcommand{\fy}{\mathbf y}
\newcommand{\fq}{\mathbf q}
\newcommand{\fz}{\mathbf z}
\newcommand{\dfe}{\dot{\mathbf{e}}}
\newcommand{\dfx}{\dot{\mathbf{x}}}
\newcommand{\dft}{\dot{\mathbf{t}}}

\DeclareMathOperator\minimize{minimize}
\DeclareMathOperator\Length{Length}
\DeclareMathOperator\Id{I_d}
\DeclareMathOperator\Lip{Lip}
\DeclareMathOperator\Leb{Leb}
\DeclareMathOperator\diver{div}

\DeclareMathOperator\TV{TV}

\DeclareMathOperator\Dist{Dist}

\DeclareMathOperator\curl{curl}
\DeclareMathOperator\lfs{lfs}
\DeclareMathOperator\supp{supp}
\DeclareMathOperator\Jac{jac}
\DeclareMathOperator\JacM{\mathbf{Jac}}

\newcommand\cfx{\hat{\mathbf{x}}}
\newcommand\dfy{\dot{\mathbf{y}}}
\newcommand\sm{\setminus}
\newcommand\gS{\mathfrak{S}}
\newcommand{\aS}{S}

\newcommand\diff{\mathrm{d}}
\newcommand\ve{\varepsilon}
\newcommand\vp{\varphi}
\newcommand\bpsi{\boldsymbol{\psi}}
\newcommand\bPsi{\boldsymbol{\Psi}}

\newcommand\bzero{\mathbf{0}}
\newcommand\rmax{\mathrm{m}}

\newcommand\<{\langle} \def\>{\rangle}
\newcommand\rC{\boldsymbol{\mathscr{C}}} 
\newcommand\rN{\boldsymbol{\mathscr{N}}}

\journalname{}
\begin{document}

\title{A Region-based Randers Geodesic Approach for Image Segmentation
}


\author{Da~Chen \and Jean-Marie~Mirebeau \and Huazhong Shu \and Laurent~D.~Cohen
}


\institute{Da~Chen\at
Shandong Artificial Intelligence Institute, Qilu University of Technology (Shandong Academy of Sciences), Jinan, China.\\
\email{dachen.cn@hotmail.com}
\and Jean-Marie Mirebeau \at
Department of Mathematics, Centre Borelli, ENS Paris-Saclay, CNRS, University Paris-Saclay, 91190, Gif-sur-Yvette, France.\\
\email{jean-marie.mirebeau@math.u-psud.fr}\\
\and Huazhong Shu \at
Laboratory of Image Science and Technology (LIST), Key Laboratory of New Generation Artificial Intelligence Technology and Its Interdisciplinary Applications (Southeast University), Ministry of Education, Nanjing 210096, China.\\
Jiangsu Provincial Joint International Research Laboratory of Medical Information Processing, Southeast University, Nanjing 210096, China.\\
\email{shu.list@seu.edu.cn}\\   
\and Laurent~D.~Cohen \at
University Paris Dauphine, PSL Research University, CNRS, UMR 7534, CEREMADE, 75016 Paris, France\\ 
\email{cohen@ceremade.dauphine.fr}\\   
}

\date{Received: date / Accepted: date}

\maketitle
\begin{abstract}
The geodesic model based on the eikonal partial differential equation (PDE) has served as a fundamental tool for the applications of image segmentation and boundary detection in the past two decades. However, the existing approaches commonly only exploit the image edge-based features for computing minimal geodesic paths, potentially  limiting their performance in complicated segmentation situations.  In this paper, we introduce a new variational image segmentation model based on the minimal geodesic path framework and the eikonal PDE, where the region-based appearance term that defines then regional homogeneity features  can be taken into account for estimating  the associated minimal geodesic paths. This is done by constructing a Randers geodesic metric interpretation of the region-based active contour energy functional. As a result, the minimization of the active contour energy functional is transformed into finding the solution to the Randers eikonal PDE. 

We also suggest a practical interactive image segmentation strategy, where the target boundary can be delineated by the concatenation of several piecewise geodesic paths. We invoke  the Finsler variant of the fast marching method to estimate the geodesic distance map, yielding an efficient implementation of the proposed region-based Randers geodesic model for image segmentation. Experimental results on both synthetic and real images exhibit that our model indeed achieves  encouraging  segmentation performance.      

\end{abstract}

\section{Introduction}
\label{intro}
\subsection{Related Work}
Active contour models are commonly established upon the mathematical frameworks of PDE, energy minimization and numerical analysis, thus featuring a quite solid theoretical background. These models are able to take advantage of efficient geometry priors and reliable image features to design flexible energy functionals, whose minimizers yield suitable solutions to a great variety of image segmentation tasks.

The classical snakes model~\citep{kass1988snakes} can be regarded as one of the earliest active contour approaches, which rely on continuous curves to delineate image edges of interest. In its basic formulation, active curves move and deform according to an internal regularization force and an image external force, both of which are derived from the Euler-Lagrange equation of an energy functional. Great efforts have been devoted to improvements of the classical snakes model so as to obtain satisfactory solutions in various image segmentation scenarios. However, an important shortcoming of the classical snakes model is its sensitivity to the curve initialization. One possible way to overcome this issue is to design new external force fields with sufficiently large capture range. Significant examples along this research line may involve the approaches introduced in~\citep{cohen1991active,cohen1993finite,xu1998snakes,jalba2004cpm,xie2008mac,li2007active}. 
In contrast to the classical snakes model which is non-intrinsic, geodesic active contour models~\citep{caselles1997geodesic,yezzi1997geometric,melonakos2008finsler} make use of intrinsic energy functionals, in the sense that they are independent of the curve parametrization. In addition, these intrinsic functionals only involve to the first-order derivative of the curves, and sometimes their curvature as well using recent techniques introduced by~\cite{mirebeau2018fast}. On the negative side this simplification limits the expressive power of the model, but on the positive side it allows for an efficient and global minimization of the energy functionals, as discussed in~\citep{osher1988fronts,goldenberg2001fast,ma2021FastGAC}. One essential common point of the active contour models reviewed above is that they only make use of locally-defined edge-based features, usually derived from the gradient of the processed image, to build the energy functionals. While these features are effective and easy to extract, their local nature increases the risk of finding segmentations corresponding to unexpected local minima of the energy functional. Comparing to the local edge-based features, the regional appearance models are capable of penalizing the global homogeneity of the image data in each subregion.
 
The region-based active contour models differ from the edge-based approaches by the use of image  regional appearance models,  thus are insensitive to local spurious edges, for instance, generated by image noise. A pioneering region-based active contour approach is the Mumford-Shah piecewise smooth model~\citep{mumford1989optimal}, whose prior is that the image appearance data (e.g.\ gray levels or colors) can be well approximated by a piecewise smooth data-fitting function. Since then, the Mumford-Shah model has motivated a series of research works on active contour models, for instances~\citep{vese2002multiphase,tsai2001curve,duan2015l0,dougan2008variational,zhang2021topologyPreserv}. Among them, the approaches which impose different regional appearance priors are able to handle more complicated image segmentation issues. The region competition model~\citep{zhu1996region} and its generalized variant~\citep{brox2009local} interpreted the Mumford-Shah energy functional in a Bayes framework, by assuming that the image gray levels follow a parameter-dependent probability distribution function (PDF). In practice, the Gaussian distributions are very often chosen to model the image data within each subregion. Approaches~\citep{chan2001active,cohen1997avoiding} applied with a piecewise constant appearance constraint on image data are considered as a practical reduction of the original Mumford-Shah model. From the viewpoint of statistics, these piecewise constant reduction models can also be treated as an instance of the region competition model~\citep{zhu1996region}, with a particular assumption that image gray levels in different subregions can be modeled by Gaussian  distributions of identical variance.  

Nonparametric active contour models~\citep{ni2009local,michailovich2007image,houhou2009fast} usually construct the energy functional by measuring the statistical distances between the histograms of the image data in different subregions,  thus allowing to work in the absence of prior knowledge on the image data distribution. In~\citep{kimmel2003fast,bresson2007fast,jung2017piecewise}, the $L^1$-norm was taken into account for measuring the approximation errors between the image data and the fitting terms. Active contour approaches based on the pairwise homogeneity criteria are capable of successfully handling more complicated segmentation situations~\citep{sumengen2006graph,bertelli2008variational,jung2012nonlocal,bresson2008non,desquesnes2013eikonal}, but the estimation of the associated gradient descent flows has a high computational complexity. Hybrid active contour models simultaneously integrate the region-based image appearance models and the edge-based features~\citep{paragios2002geodesic,kimmel2003regularized,zach2009globally,bresson2007fast,chen2021generalized} for building the energy functionals, so as to cumulate the benefits of both types of features. 
The segmentation models in conjunction with prescribed shape constraints (e.g. \ statistical shape priors, convexity and star convexity shape priors) have proven to generate accurate segmentation results~\citep{cremers2002diffusion,bresson2006variational,chen2022geodesic,chen2021elastica,shi2021convexity,siu2020image,luo2022convex,zhang2021topology,prevost2014tagged}, especially for the case that the features derived from the image data are not reliable enough. In~\citep{sundaramoorthi2007sobolev,sundaramoorthi2009new,yang2015shape}, the curve evolution is driven by a type of Sobolev gradient flows, allowing imposing various flexible geometry priors to enhance the image segmentation process.

Level set formulations~\citep{osher1988fronts} represent a closed curve by means of the zero-level line of a scalar-valued function (e.g. signed Euclidean distance map to the curve), which is considered as a quite effective solution to the contour evolution problems~\citep{cremers2007review,dubrovina2015multi,brox2006level,li2008minimization,zhang2019resls,alvarez2018level}. However, the active contour models relying on the level set schemes often suffer from numerical difficulties in practical segmentation applications, due to the non-convex nature of the associated energy functionals. 
The approaches~\citep{chan2006algorithms,bresson2007fast,chambolle2012convex,bae2011global,yuan2014spatially} based on a convex relaxation of the level set method allow finding the global minimum of a region-based active contour energy reformulated in a convex set, thus are less demanding on initialization, and also benefit from efficient numerical techniques for convex optimization. Instead of working with the PDE framework, graph-based optimization algorithms have been exploited to address various active contour problems~\citep{boykov2003computing,grady2009piecewise,mishra2011decoupled}. 

\begin{figure*}[t]
\centering
\includegraphics[height=4.5cm]{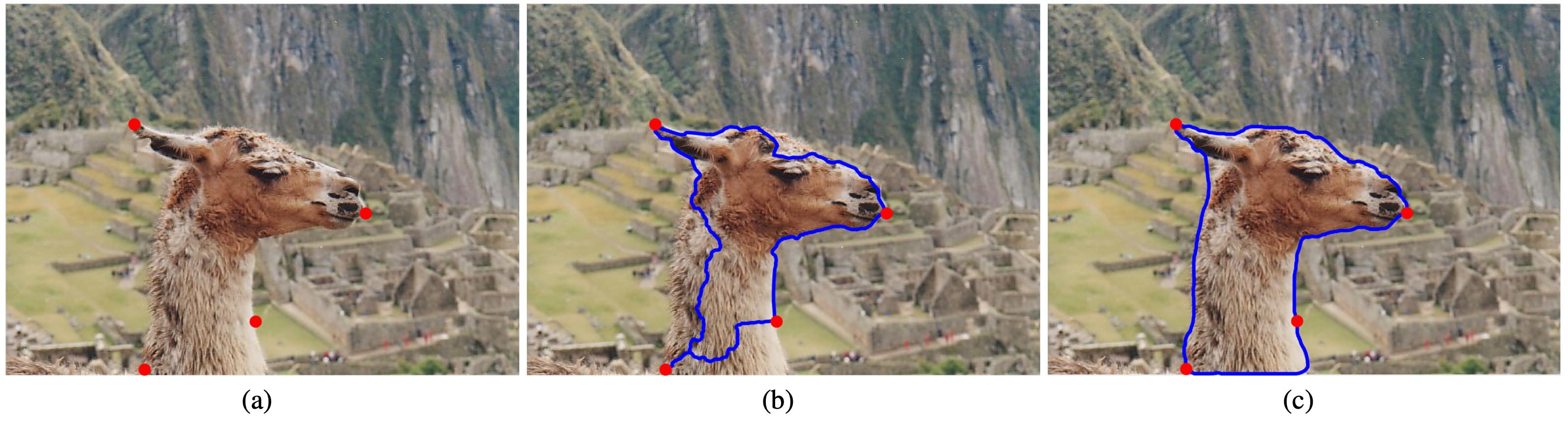}
\caption{An example of interactive image segmentation using the combination of piecewise geodesic paths model~\citep{mille2015combination} and the proposed region-based Randers geodesic model. \textbf{a} The original image with user-provided landmark points are indicated by red dots. The image is obtained form the Grabcut dataset~\citep{rother2004grabcut}. \textbf{b} The result from the  combination of piecewise geodesic paths model. \textbf{c} The result obtained using the proposed model.}
\label{fig:DemoComb}
\end{figure*}

\subsection{Optimal Paths-based Image Segmentation Approaches}
\label{subsec:minimal_paths}

In geodesic active contour models~\citep{caselles1997geodesic,yezzi1997geometric}, the minimization  of the energy (i.e.\ the cost of paths) is implemented through a curve evolution scheme in a gradient descent manner. Unfortunately, there is no guarantee for these approaches to find the globally minimizing curves. In order to address this issue, \citet{cohen1997global} introduced a minimal path model whose \emph{global} minimizer can be extracted from the solution to a nonlinear eikonal PDE. The energy of this model is defined as the Euclidean length of the path modulated by a local positive coefficient. The scope of this approach was later extended to encompass energies defined by a Riemannian or a Finslerian metric, allowing to favor paths aligned with certain directions and orientations in an anisotropic manner~\citep{melonakos2008finsler,bougleux2008anisotropic,chen2021geodesic}, and to degenerate sub-Riemannian and sub-Finslerian metrics which can take path curvature into account~\citep{chen2017global,duits2018optimal,mirebeau2018fast}. Minimal geodesic paths are a basic tool for image segmentation~\citep{peyre2010geodesic}, involved in a great variety of algorithms which benefit from the efficient numerical solvers of the eikonal PDE~\citep{sethian1999fast,mirebeau2014efficient,mirebeau2019riemannian}, and from the global optimality of the paths extracted from its solution.

In the classical geodesic models, the extraction of a minimal geodesic path requires a pair of fixed points, which serve as its endpoints. In the eikonal PDE framework, one endpoint is used as the source point defining the boundary condition  of the eikonal PDE, and the other is the target point, which initializes the geodesic backtracking ordinary differential equation (ODE). \citet{cohen1997global} suggested a saddle point method for image segmentation using a single source point. Taking the saddle point as the target point, one can track two geodesic paths which connect to the source point from different ways.  This saddle point-based strategy is then extended in~\citep{mille2015combination,cohen2001multiple} to closed contour detection in the scenario of multiple landmark points as source points. Accordingly, the objective contours are generated using the concatenation of a family of piecewise geodesic paths, thus passing through all given points.  \citet{benmansour2009fast} proposed an iterative keypoints sampling scheme for generating minimal paths from a single source point. In each iteration, a new keypoint is treated as a target point, form which a geodesic path can be backtracked until another keypoint is reached. 
\citet{chen2017global} proposed to connect a set of prescribed points of unknown order through the curvature-penalized geodesic paths. These paths are tracked in an orientation-lifted space and are derived using the image boundary features extracted by a steerable filter. A common feature of the geodesic approaches mentioned above is to build the segmentation contours as the concatenation of piecewise geodesic paths. Alternatively, \citet{appleton2005globally} introduced a circular geodesic model with a particular topology constraint, extracting a single closed geodesic path with an identical source and target point. 

In contrast to geodesic models which solve the eikonal PDE to track optimal paths as segmentation contours, a crucial point for discrete paths-based  models~\citep{falcao2000ultra,miranda2012riverbed,boykov2003computing,windheuser2009beyond} is to map the processed image  to a regular graph of nodes and edges, where a discrete shortest path comprised of finite ordered nodes can be efficiently computed by the well-established graph-based optimization algorithms~\citep{dijkstra1959note,falcao2004image,boykov2004experimental}.

Most of the existing optimal paths-based segmentation models, using either the eikonal PDE framework or the graph-based optimization schemes,  only leverage edge-based features such as image gradient features to construct the local geodesic metrics. Therefore, the resulting minimal geodesic paths in essence fall into the image \emph{edge-driven} limitation, despite their possible combination with region-based homogeneity features for image segmentation as introduced in~\citep{mille2015combination,appia2011active}. Unfortunately, geodesic paths derived from the image edge-based features sometimes fail to accurately depict the boundaries of interest, especially when the processed images involve complex structures or when the image gradients are not reliable. In this paper, we introduce a region-based Randers geodesic path model for addressing the problems of active contours and image segmentation. In our model, the considered local geodesic metrics are an instance of Finslerian geometry which are constructed in terms of a region-based appearance model, i.e.\ regional homogeneity features, and/or anisotropic edge-based features. 

In Fig.~\ref{fig:DemoComb}, we take the combination of piecewise geodesic paths  model proposed by~\citet{mille2015combination} as an example to illustrate the limitations of using edge-driven geodesic paths for segmentation. This model is an interactive segmentation method, for which a set of ordered vertices or landmark points on the target boundary should be provided. In Fig.~\ref{fig:DemoComb}a, we show the original image with four landmark points, displayed as red dots, lying on the objective boundary. One of the crucial points for the combination of piecewise geodesic paths model is the construction of a potential used for estimating the geodesic distance maps.  In Fig.~\ref{fig:DemoComb}b, we show the segmentation contour (visualized by blue lines) from the combination of piecewise geodesic paths model. 
Clearly, some portions of the desired boundary are not extracted, and the underlying reason is that these missed boundaries are not well characterized by the edge-based features. In contrast, the geodesic paths derived from the proposed geodesic model, which blends the benefits of both region-based and edge-based features, follow the desired boundary, as illustrated in Fig.~\ref{fig:DemoComb}c.

\subsection{Contribution and Paper Outline}
Overall, the main contribution in this paper is threefold:
\begin{itemize}
\item[-] Firstly, we introduce a new geodesic model which bridges the gap between the region-based active contours which involve a regional image appearance term, and the geodesic paths based on the eikonal PDE framework. The proposed model, named \emph{region-based Randers geodesic model}, is an instance of Finslerian geometry, which relies on minimal geodesic paths with respect to an asymmetric Randers metric. 
\item[-] Secondly, we present a convergence analysis of the contour evolution scheme which is proposed in this paper. We also discuss the numerical implementation of scheme, whose iterations involve the solution to a Randers eikonal PDE and to a curl PDE.
\item[-] Finally, we introduce a landmark points-based scheme to apply the proposed region-based Randers geodesic model to interactive image segmentation. The core idea is to extract the target boundary as a concatenation of minimal geodesic paths. Each geodesic path connects a pair of landmark points.
\end{itemize}

The manuscript is organized as follows. In Section \ref{sec_Background}, we briefly introduce the  minimal path models and the corresponding Eikonal PDEs and geodesic backtracking ODEs, including the Riemannian cases and the asymmetric Randers case. Preliminary results on the hybrid active contour energy functionals and on the iterative minimization scheme for these functionals are presented in Section~\ref{sec_AC}. In particular we transform in \cref{subsec_RandersInterpretation} the hybrid active contour energy functional into a weighted curve length associated to a Randers metric. This transformation relies on the solution to a divergence equation.
In Section~\ref{sec_SegmentationAlgs}, we introduce two methods to exploit the constructed Randers metrics and the associated minimal paths for the application of image segmentation. The numerical implementation details and the experimental results are presented in Sections \ref{sec_PracticalRanders} and \ref{sec_Experiments}, respectively. The conclusion can be found in Section~\ref{sec_conclusion}.

The preliminary short versions of this work  were first presented in the conferences~\citep{chen2016finsler,chen2017anisotropic}, upon which more contributions such as the theoretical convergence analysis have been added.

\section{Background on Randers Minimal Geodesic Paths}
\label{sec_Background}
In this section, we introduce the background on the computation of paths minimizing a length defined by a Randers geodesic metric.

\paragraph{Notations.} We denote by $\Omega\subset\bR^2$ a connected open bounded domain, with a smooth boundary. Points are denoted $\fx \in \Omega$, vectors $\dfx \in \bR^2$, and co-vectors $\cfx \in \bR^2$. The Euclidean scalar product and norm on $\bR^2$ are denoted by $\<\cdot,\cdot\>$ and $\|\cdot\|$. For any $\dfx=(a,b) \in \bR^2$ we denote by $\dfx^\perp := (-b,a)$ the counter-clockwise orthogonal vector.

We denote by $\bS_2^{++}$ the set of symmetric positive definite matrices of shape $2 \times 2$, and we associate to each $M \in \bS_2^{++}$ the norm $\|\dfx\|_M := \sqrt{\<\dfx,M\dfx\>} = \| M^\frac 1 2 \dfx\|$. Note that $\|\dfx\| \|M^{-1}\|^{-\frac 1 2} \leq \|\dfx\|_M \leq \|\dfx\| \|M\|^\frac 1 2$, where $\|M\|$ denotes the spectral norm (i.e.\ the largest eigenvalue, for a symmetric positive matrix).

We denote by $\Lip(X,Y)$ the set of Lipschitz functions from a metric space $(X,d_X)$ to another $(Y,d_Y)$. For such a function $f : X \to Y$ we denote by $\Lip(f)\in [0,\infty[$ the smallest constant such that $d_Y(f(x),f(\tilde x)) \leq \Lip(f) d_X(x,\tilde x)$ for all $x,\tilde x \in X$. Note that a function defined over a compact space is Lipschitz iff it is locally Lipschitz.

Throughout the document, the wording ``curve'' refers to a \emph{parametrized curve}, also known as a path. (This is in opposition to the concept of \emph{geometrical curve}, not considered here, which is defined as an equivalence class of paths modulo reparametrizations.)

\subsection{Randers Geodesic Metric}
\label{subsec_FinslerIntro}
Randers geometry is a generalization of Riemannian geometry, and a special case of Finslerian geometry~\citep{randers1941asymmetrical,bao2012introduction}.
In the general framework of path-length geometry, the energy of a path $\gamma \in \Lip([0,1],\overline \Omega)$ reads
\begin{equation}
\label{eq_RandersLength}
\Length_{\cF}(\gamma):=\int_0^1 \cF(\gamma(t),\gamma'(t)) \, \diff t,
\end{equation}
where $\cF:\overline \Omega\times\bR^2\to[0,\infty[$ is the geodesic metric. 
Note that $\gamma'$ is well defined as an element of $L^\infty([0,T], \bR^2)$ by Rademacher's theorem.
The energy~\eqref{eq_RandersLength} is also known as the length associated to the geodesic metric $\cF$. The geodesic distance $\Dist_\cF$ between two points $\fx,\fy \in \overline \Omega$ is defined as the length of the shortest path joining them  
\begin{align}
\label{eq_GeodesicDistance}
\Dist_\cF(\fx,\fy)=\inf \{\Length_\cF(\gamma) \mid\, &\gamma \in \Lip([0,1],\overline \Omega), \\
&\gamma(0)=\fx,\, \gamma(1)=\fy\}.\nonumber
\end{align}

Randers geodesic metrics have, by assumption, the following structure~\citep{randers1941asymmetrical}: for all $\fx \in \overline \Omega$, $\dfx \in \bR^2$
\begin{equation}
\label{eq_RandersForm}
\cF(\fx,\dfx)=\|\dfx\|_{\cM(\fx)}+\<\omega(\fx),\dfx\>.
\end{equation}
A Randers metric thus involves a positive definite tensor field $\cM \in \Lip( \overline \Omega , \bS_2^{++})$, similarly to the Riemannian setting. In addition, a Randers metric features a linear term defined by a vector field $\omega \in \Lip(\overline\Omega,\bR^2)$, which is subject to the compatibility assumption $\<\omega(\fx),\cM(\fx)^{-1} \omega(\fx)\> < 1$ for all $\fx \in \overline \Omega$.
Because of this second term, a Randers metric is in general asymmetric: more precisely,  
$\cF(\fx,\dfx)\neq \cF(\fx,-\dfx)$ whenever $\<\omega(\fx),\dfx\> \neq 0$. As a result, the corresponding path length and geodesic distances are also asymmetric: in general $\Length_\cF(\gamma) \neq \Length_\cF(\gamma(1-\cdot))$ and $\Dist_\cF(\fx, \fy) \neq \Dist_\cF(\fy, \fx)$.

\begin{figure*}[t]
\setlength{\fboxsep}{0pt}%
\centering
\includegraphics[height=4.8cm]{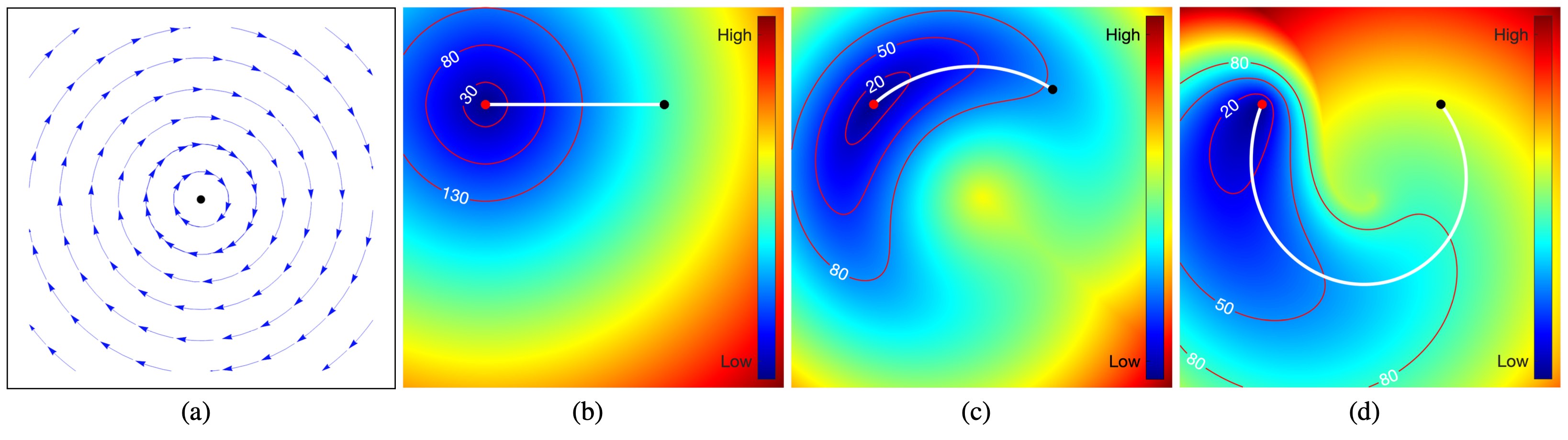}
\caption{Geodesic paths derived from different geodesic models. (\textbf{a}) Illustration of the vector field $\varpi$ defined over the whole domain, where the black dot indicates the origin of the domain.  (\textbf{b}) to (\textbf{d}) Geodesic paths obtained respectively  using  the isotropic,  anisotropic Riemannian  and asymmetric Randers metrics, where the white lines are the extracted minimal geodesic paths. The red lines  are the level set lines the geodesic distances.}
\label{fig_DifferentMetrics}	
\end{figure*} 

\paragraph{Minimal geodesics.}
Let us briefly justify the existence of a minimizer to \eqref{eq_GeodesicDistance}, referred to as a minimal geodesic from $\fx$ to $\fy$, since variants of this classical argument, see the references at the end of this paragraph, are used in the mathematical analysis of our model \cref{sec_AC}. 
The two main ingredients are the Arzelà-Ascoli theorem, and the lower-semi continuity of the length functional. The former states that the set 
\begin{equation}
\label{eq:ArzelaAscoli}
	\{f \in \Lip(X,Y)\mid \Lip(f) \leq K\},
\end{equation}
equipped with $d(f,\tilde f) := \max_{x \in X} d_Y(f(x), \tilde f(x))$ which is the distance associated to uniform convergence, is compact for any compact metric spaces $(X,d_X)$ and $(Y,d_Y)$ and any constant $K$. The latter affirms that, for any sequence of curves 
$\gamma_n \in \Lip([0,1],\overline \Omega)$ converging uniformly $\gamma_n \to \gamma_*$ as $n \to \infty$, for some $\gamma_* \in \Lip([0,1],\overline \Omega)$, one has 
\begin{equation}
\label{eq:length_lsc}
	\Length_\cF(\gamma_*) \leq \liminf_{n \to \infty} \Length_\cF(\gamma_n).
\end{equation}

Now regarding the minimal geodesic problem \eqref{eq_GeodesicDistance}, 
our assumptions on the domain $\Omega$ (namely openness, connectedness, boundedness, smooth boundary, as stated in the notations of \cref{sec_Background}), imply that there exists at least one Lipschitz path from $\fx$ to $\fy$. 
Denote by $(\gamma_n)_{n \geq 0}$ a sequence of
 paths $\gamma_n \in \Lip([0,1], \overline \Omega)$ such that $\gamma_n(0) = \fx$, $\gamma_n(1)=\fy$, and $\Length_\cF(\gamma_n) \to \Dist_\cF(\fx,\fy)$ as $n \to \infty$.
 Without loss of generality, we can assume that these paths are parametrized at constant speed w.r.t.\ the Euclidean metric, and thus $\Lip(\gamma_n) \leq \Length_\cF(\gamma_n)/ \rho_{\min}(\cF)$ is bounded, where the constant $\rho_{\min}(\cF)>0$ is defined below in Eq.~\eqref{eq:anisotropy_ratios}.
 By the Arzelà-Ascoli theorem \eqref{eq:ArzelaAscoli}, applied to $X=[0,1]$ and $Y = \overline \Omega$, there exists a subsequence $\gamma_{\vp(n)} \to \gamma_*$ converging uniformly as $n \to \infty$, and the limiting curve $\gamma_*$ is also Lipschitz. By the lower-semi-continuity of $\Length_\cF$, the limiting curve $\gamma_*$ is a minimal geodesic.

A generalization of the above proof can be found in \cite[Section 2.5]{burago2022metric}. 
The fact that length by integral of velocity \eqref{eq_RandersLength} equals length by total variation is proved in \cite[Proposition 1.7]{younes2010shapes}, and this implies the lower semi-continuity of the length \citep{burago2022metric}.

\paragraph{Anisotropy ratios, and duality.}
One often needs to compare the (possibly asymmetric) norm $F(\dfx) := \cF(\fx,\dfx)$ defined by a Randers metric at a given point $\fx \in \overline \Omega$, with the standard Euclidean norm $\|\dfx\|$. For that purpose, we introduce the minimum and maximum ratios
\begin{align}
\label{eq:anisotropy_ratios}
	\rho_{\min}(F) &:= \min_{\|\dfx\|=1} F(\dfx),&
	\rho_{\max}(F) &:= \max_{\|\dfx\|=1} F(\dfx).
\end{align}
By positive homogeneity of $F$, one has $\rho_{\min}(F) \leq F(\dfx)/\|\dfx\| \leq \rho_{\max}(F)$ for all $\dfx \neq \mathbf{0}$. 
The anisotropy ratio is defined as $\rho(F) := \rho_{\max}(F) / \rho_{\min}(F)$. 
By convention, we also set
\begin{align*}
	\rho_{\min}(\cF) &:= \min_{\fx \in \overline \Omega} \rho_{\min}(\cF(\fx,\cdot)), \\
	\rho_{\max}(\cF) &:= \max_{\fx \in \overline \Omega} \rho_{\max}(\cF(\fx,\cdot)).
\end{align*}
Another object of interest is the dual norm, defined as
\begin{equation}
\label{eq:dual_norm}
	F^*(\cfx) := \max\{\<\cfx,\dfx\> \mid \dfx \in \bR^2,  F(\dfx) = 1\},
\end{equation}
for all $\cfx \in \bR^2$. The dual norm plays a central role in the generalized eikonal PDE obeyed by the distance map, and in geodesic backtracking ODE, see~\cref{subsec:eikonal}.
The following result, which concludes this subsection, provides estimates of the ratios~\eqref{eq:anisotropy_ratios}, sharp up to a factor two, as well as the explicit form of the dual~\eqref{eq:dual_norm} of a Randers metric.
 
\begin{proposition}
\labelx{prop:randers_properties}
Let $M \in S_2^{++}$ and $\omega\in \bR^2$ be such that $\<\omega,M^{-1}\omega\> < 1$. Define $F(\dfx) := \|\dfx\|_M + \<\omega,\dfx\>$ for all $\dfx \in \bR^2$. Then $F$ is definite, positively homogeneous, and obeys the triangular inequality. 

The dual norm \eqref{eq:dual_norm} has the same structure: $F^*(\cfx) = \|\cfx\|_A+\<\cfx,\kb\>$, where $A\in S_2^{++}$, 
$\<\kb,A^{-1} \kb\> < 1$, and 
\begin{align*}
\begin{pmatrix}
	M & \omega\\
	\omega^\top & 1
\end{pmatrix}
^{-1}
&=
\begin{pmatrix}
	\delta A & \kb\\
	\kb^\top & 1/\delta
\end{pmatrix}
&
\text{with }
\delta = 1-\<\omega,M^{-1}\omega\>.
\end{align*}
The anisotropy ratios obey 
\begin{align}
\label{eq:mumax_bounds}
	\sqrt{\|M\|} &\leq \rho_{\max}(F) = \rho_{\min}(F^*)^{-1} < 2\sqrt{\|M\|}, \\
\label{eq:mumin_bounds}
	\tfrac 1 2 \sqrt{\|A\|} &< \rho_{\min}(F) = \rho_{\max}(F^*)^{-1} \leq \sqrt{\|A\|},
\end{align}
and thus $\sqrt{\|M\|\|A\|} \leq \rho(F)=\rho(F^*) < 4 \sqrt{\|M\| \|A\|}$.
Also $\rho_{\min}(F) \leq \|M^{-1}\|^{-\frac 1 2}$ and $\rho(F) \geq 1/(1-\|\omega\|_{M^{-1}})$.
\end{proposition}

The rest of this section is devoted to the proof.
The positive homogeneity $F(\lambda \dfx) = \lambda F(\dfx)$, and the triangular inequality $F(\dfx+\dfy) \leq F(\dfx)+F(\dfy)$, for all $\lambda\geq 0$ and $\dfx,\dfy \in \bR^2$, follow from the same properties of the norm $\|\cdot\|_M$ and of the linear form $\<\omega,\cdot\>$.
One has 
\begin{equation*}
	|\<\omega,\dfx\>| = |\<M^{-\frac 1 2} \omega, M^\frac 1 2 \dfx\>| \leq 
	\|\omega\|_{M^{-1}} \|\dfx\|_M 
\end{equation*}
by the Cauchy-Schwartz inequality, and therefore 
\begin{equation}
\label{eq:Randers_bounds}
	\|\dfx\|_M (1-\|\omega\|_{M^{-1}}) \leq F(\dfx) \leq \|\dfx\|_M (1+\|\omega\|_{M^{-1}}),
\end{equation}
showing that $F$ is definite, i.e.\ $F(\dfx)>0$ whenever $\dfx \neq 0$. 
On the other hand, observing that $F(-M^{-1} \omega) = \|\omega\|_{M^{-1}} (1-\|\omega\|_{M^{-1}})$, we find that the compatibility condition $\|\omega\|_{M^{-1}}<1$ is necessary for definiteness of $F$. Also 
$
\rho(F) \geq F(M^{-1} \omega)/F(-M^{-1} \omega) 
\geq 1/(1-\|\omega\|_{M^{-1}})
$
as announced in the last sentence of the proposition.

Denote by $\lambda = \|M\|$ the largest eigenvalue of $M$, and by $\dfe$ the corresponding unit eigenvector. One has $\max \{F(\dfe),F(-\dfe)\} = \sqrt{\lambda} + |\<\omega,\dfe\>| \geq \sqrt{\lambda}$, hence $\rho_{\max}(F) \geq \sqrt{\lambda}$. On the other hand, the equation~\eqref{eq:Randers_bounds} yields $\rho_{\max}(F) \leq (1+\|\omega\|_{M^{-1}})  \sqrt{\lambda} < 2 \sqrt{\lambda}$ as announced in Eq.~\eqref{eq:mumax_bounds}. Conversely, if $\lambda = \|M^{-1}\|^{-1}$ is the smallest eigenvalue of $M$, and if $\dfe$ is the associated eigenvector, then $\min\{F(\dfe), F(-\dfe)\} = \sqrt \lambda - |\<\omega,\dfe\>| \leq \sqrt \lambda$, hence $\rho_{\min}(F) \leq \sqrt \lambda$ as announced in the last sentence of the proposition.

The dual norm $F^*$ is given\footnote{
Randers norms are presented in \cite{mirebeau2014efficient} with the convention $F(\dfx) = \|\dfx\|_M - \<\tilde \omega, M \dfx\>$ corresponding to $\omega = -M\tilde \omega$.
} 
in \cite[Proposition 4.1]{mirebeau2014efficient} by $A = \delta^{-2} M^{-1} \omega \omega^\top M^{-1} + \delta^{-1} M^{-1}$ and $\kb = - \delta^{-1} M^{-1} \omega$.
On the other hand, the $2 \times 2$ block matrix 
$
	\begin{pmatrix}
		M & \omega\\
		\omega^\top & 1
	\end{pmatrix}
$
has a non-singular top left block $M$ and Schur complement $\delta$, hence its inverse is known explicitly \citep{lu2002inverses}.
Comparing these expressions we obtain the announced expression of the dual norm, and also that $F^{**} = F$, which alternatively could be proved in a more general setting using convex duality theory.
Since $F^*$ is definite, one must have $\<\kb,A^{-1} \kb\> < 1$. 

Consider $\dfx,\cfx \in \bR^2$. If $F(\dfx)=1$, then $\rho_{\min}(F) \| \dfx\| \leq 1$, and inserting this estimate in \eqref{eq:dual_norm} yields
\begin{equation*}
F^*(\cfx) \leq \<\cfx,\dfx\> \leq \|\cfx\| \|\dfx\| \leq \|\cfx\| /\rho_{\min}(F),	
\end{equation*}
hence $\rho_{\max}(F^*) \rho_{\min}(F) \leq 1$.
On the other hand, assume that $\cfx \neq \mathbf{0}$ and take $\dfx= \cfx/F(\cfx)$, which satisfies $F(\dfx)=1$, in~\cref{eq:dual_norm}. Then we obtain $F^*(\cfx) F(\cfx) \geq \|\cfx\|^2$, hence $\rho_{\min}(F^*) \rho_{\max} (F) \geq 1$. Exchanging the roles of the primal norm $F$ and of the dual norm $F^*$ leads to 
\begin{equation*}
\rho_{\max}(F) \rho_{\min}(F^*) \leq 1\text{~~and~~}\rho_{\min}(F) \rho_{\max} (F^*) \geq 1	,	
\end{equation*}
hence 
\begin{equation*}
\rho_{\max}(F) = \rho_{\min}(F^*)^{-1}\text{~~and~~}\rho_{\min}(F) = \rho_{\max}(F^*)^{-1},
\end{equation*}
as announced in Eqs.~\eqref{eq:mumax_bounds} and~\eqref{eq:mumin_bounds}. Thus Eq.~\eqref{eq:mumax_bounds} is established, and Eq.~\eqref{eq:mumin_bounds} follows by duality. Taking the ratio of these estimates yields the announced bounds for $\rho(F)$, which concludes the proof.

\subsection{Characterization of Randers Distances and Geodesics} 
\label{subsec:eikonal} 

The framework of optimal control \citep{bardi2008optimal} allows characterizing the Randers geodesic distance from a given source point as the viscosity solution to an eikonal PDE. This paves the way for the numerical computation of this distance, which is presented in~\cref{subsec_FMSolver}. Once the geodesic distance is known, the minimal Randers geodesic paths from the source point to the target points can be backtracked using an appropriate ODE.

Denote by $\rD : \overline \Omega \to \bR$ the geodesic distance defined by $\rD(\fx) := \Dist_\cF(\fs,\fx)$, where the source point $\fs \in \Omega$ and the Randers metric $\cF$ are fixed in this subsection. We also refer to $\rD$ as the geodesic distance map or minimal action map. Following \citep{bardi2008optimal}, the map  $\rD$ is the unique viscosity solution to the following static Hamilton-Jacobi-Bellman PDE, which is a variant of the eikonal PDE:
\begin{equation}
\label{eq_FinslerEikonal}
\begin{cases}
\cF^*(\fx,\diff \rD(\fx)) = 1,&\forall\fx\in\Omega \sm \{\fs\},\\
\rD(\fs) = 0,	
\end{cases}
\end{equation}
with outflow boundary conditions on $\partial \Omega$. Using the structure of Randers metrics, the eikonal PDE~\eqref{eq_FinslerEikonal} can be stated more explicitly as 
\begin{equation*}
\|\diff\rD(\fx)\|_{A(\fx)}+\<\kb(\fx), \diff \rD(\fx)\>=1,
\end{equation*}
where the dual metric parameters $A(\fx)$ and $\kb(\fx)$ are obtained from  \cref{prop:randers_properties}. 
There are other mathematically equivalent forms of the eikonal PDE \eqref{eq_FinslerEikonal}, such as $\|\diff \rD_\fs(\fx)-\omega(\fx)\|_{\cM^{-1}(\fx)}=1$, see~\citep{mirebeau2019riemannian,chen2018fast}.
\vspace{-0.5\baselineskip}\\

\noindent\emph{Geodesic backtracking.} 
Optimal feedbacks, and optimal paths, can be extracted from the solution of the eikonal PDE~\eqref{eq_FinslerEikonal}, following the framework of optimal control. We briefly give the heuristic of these techniques, and refer the reader to~\citep{bardi2008optimal} for a complete and rigorous presentation, in a much more general setting in terms of geometry and regularity (the metric needs not be of Randers form, and the related coefficients can be merely continuous), which is outside the scope in this paper. Let us emphasize in particular that $\rD$ is not a classical solution to the PDE~\eqref{eq_FinslerEikonal}, but only a viscosity solution, hence the term $\diff \rD$ appearing below is in~\citep{bardi2008optimal} replaced with a suitable weak and in general multivalued notion of derivative, which is not discussed here.

The geodesic flow, or optimal feedback map, is the vector field $\mathbf{V} : \overline \Omega \to \bR^2$ obtained as 
\begin{equation*}
\mathbf{V}(\fx):=\arg\max\{\< \diff \rD(\fx),\dfx \> \mid \dfx \in \bR^2, \cF(\fx,\dfx)=1\}.
\end{equation*}
Equivalently, in terms of the dual metric, one has 
\begin{equation}
\label{eq:geodesic_flow}
\mathbf{V}(\fx) = \diff \cF^*_\fx(\diff \rD(\fx)) =\frac{A(\fx)\diff \rD(\fx)}{\|\diff \rD(\fx)\|_{A(\fx)}} + \kb(\fx).
\end{equation}
A globally optimal geodesic path $\cG_\fx$ from the fixed source point $\fs$ to an arbitrary target $\fx$, satisfies
\begin{align}
\label{eq:backtracking_ODE}
\begin{cases}
	\cG_\fx'(t) = \mathbf{V}(\cG_\fx(t)), &\quad  0 < t \leq T := \rD(\fx), \\ 
	\cG_\fx(T) = \fx.
\end{cases}
\end{align}
This is known as the backtracking ODE, and it is solved backwards in time from the target point $\fx$ until the source point $\fs$ is reached. Indeed, assuming differentiability, one easily checks that $\rD(\cG_\fx(t)) = t$ for all $0<t \leq T$, and thus $\cG_\fx(t) \to \fs$ as $t \to 0^+$.  

Let us emphasize however that the geodesic flow $\mathbf{V}$ has not sufficient regularity (and may in fact be multivalued like $\diff \rD$) to apply the Cauchy-Lipschitz/Picard-Lindelöf existence and stability theorem for ODEs. Indeed, geodesic paths between two given points are not unique in general, and their existence is instead ensured by the argument presented in \cref{subsec_FinslerIntro}. The backtracking ODE allows  characterizing these optimal paths, and  computing them under suitable assumptions, see respectively Sections III.2.5 and A.1.2 of~\citep{bardi2008optimal}.

\begin{remark}
We fix throughout the paper a Riemannian metric, defined as 
\begin{equation}
\label{eq_RiemannMetric}
\cR(\fx,\dfx):=\sqrt{\<\dfx,\cM(\fx)\dfx\>}=\|\dfx\|_{\cM(\fx)}.
\end{equation}
where $\cM \in C^1(\overline \Omega , S_2^{++})$ is a tensor field, whose practical choice is discussed later in \cref{subsec_PracticalMetric}. As can be observed, Riemannian metrics $\cR$ are a special case of Randers metrics, whose asymmetric term  identically vanishes, i.e. $\omega\equiv \mathbf{0}$. The dual metric of $\cR$ also has a Riemannian form, defined by $A = \cM^{-1}$ and $\kb \equiv \mathbf{0}$, in view of \cref{prop:randers_properties}. The Riemannian eikonal PDE \eqref{eq_FinslerEikonal} and the flow of the backtracking ODE \eqref{eq:geodesic_flow} are specialized accordingly and become $\|\diff \rD(\fx)\|_{\cM(\fx)^{-1}} = 1$ and $\mathbf{V}(\fx) = \cM(\fx)^{-1} \diff \rD(\fx)/\|\diff \rD(\fx)\|_{\cM(\fx)^{-1}}$, respectively. 
\end{remark}

\subsection{Eikonal Solver: the Finsler Variant of the Fast Marching Method}
\label{subsec_FMSolver}
The fast marching method (FMM)~\citep{tsitsiklis1995efficient,sethian1999fast} is an efficient algorithm for solving the eikonal PDE associated to an isotropic metric, of the form $\cF(\fx,\dfx) = c(\fx) \|\dfx\|$ where $c: \overline \Omega \to ]0,\infty[$ is a given cost function. This numerical scheme requires a Cartesian grid $\Omega_h := \Omega \cap h \bZ^2$ where $h>0$ denotes the grid scale, containing the source point $\fs$. The FMM numerically solves the discretized eikonal PDE in a single pass over $\Omega_h$ using a front propagation. 
Adapting the FMM to Riemannian and Finsler metrics is a non-trivial task, which has led to a continued line of research \citep{sethian2003ordered,mirebeau2014anisotropic,mirebeau2014efficient,mirebeau2019riemannian}.
The numerical method used in this paper relies on a semi-Lagrangian scheme, based on the following update operator: for any discrete map\footnote{The choice of font distinguishes the continuous mapping $\rD : \overline \Omega \to \bR$ from the discrete one $\cD : \Omega_h \to \bR$ used in the numerical scheme.} $\cD : \Omega_h \to \bR$ and point $\fx \in \Omega_h$ of the discretization grid
\begin{equation}
\label{eq_HopfLax}
\Lambda \cD(\fx) := \min\big\{\cD(\fy)+\cF(\fx,\fx-\fy) \mid \fy\in\partial \cS(\fx)\big\},
\end{equation}
where $\cS(\fx)$ is a polygonal neighborhood of the point $\fx$ referred to as the stencil, whose vertices lie on the grid $h \bZ^2$. The table of distances $\cD$ is linearly interpolated between these vertices, and extended by $+\infty$ outside of $\Omega$, which implements outflow boundary conditions.
The scheme solution is defined as the unique solution to the fixed point problem $\cD(\fx) = \Lambda \cD(\fx)$, which mimics Bellman's optimality principle 
\begin{equation*}
\rD(\fx) = \min\{ \rD(\fy) + \Dist_\cF(\fy,\fx) \mid \fy\in\partial \cS(\fx)\}	
\end{equation*}
obeyed by the solution to the eikonal PDE \eqref{eq_FinslerEikonal}.

The original isotropic FMM relies on a basic diamond shaped stencil \citep{tsitsiklis1995efficient,sethian1999fast}, independent of the grid point $\fx \in \Omega_h$.
In contrast, extensions of the isotropic FMM to anisotropic Riemannian, Randers or Finslerian metrics require more sophisticated stencils $\cS(\fx)$ obeying an acuteness property \citep{sethian2003ordered,mirebeau2014anisotropic} depending on the local metric $\cF(\fx,\cdot)$. In this paper, we rely on \emph{fast marching using anisotropic stencil refinement} (FM-ASR)~\citep{mirebeau2014efficient}, so that $\cS(\fx)$ contains $\cO (\ln^2 \rho)$ grid points in average and $\cO(\rho \ln\ln \rho)$ in the worst case~\citep{mirebeau2019worst}, where $\rho := \rho(\cF(\fx,\cdot))$. 
The FM-ASR method is presented in \cref{alg_FM}, and has an overall complexity $\cO(K N \ln N)$, where $N = \#(\Omega_h)$ is the number of grid points, and $K$ is the mean number of grid points per stencil. In addition, computation time can be slightly shortened by stopping early as soon as the desired endpoint is accepted.

\begin{algorithm}[t]
\caption{\textsc{Fast Marching Method}}
\label{alg_FM}
\SetKwInOut{Input}{Input}
\SetKwInOut{Output}{Output}
\Input{A Finsler metric $\cF$ and a source point $\fs\in \Omega_h$.}
\Output{Approximated distance map $\cD(\fx) \approx \Dist_\cF(\fs,\fx)$.}
\SetKwInOut{Ini}{Initialization}
\Ini{ 
\begin{itemize}
\item Set $\cD(\fs)=0$.
\item Tag each grid point $\fx\in \Omega_h$ as \emph{Trial}.
\item Construct the stencils $\cS$ adapted to the metric $\cF$.
\end{itemize}
}
\While{a trial point remains}{ 
Find a \emph{Trial} point $\fx_{\min}$ with minimal distance $\cD(\fx_{\min})$.\\
Tag $\fx_{\min}$ as \emph{Accepted}.\\

\For{each Trial point $\fy$ such that $\fx_{\min}\in \cS(\fy)$}{
\label{algLineNeigh}
Set $\cD(\fy) \gets \Lambda \cD(\fy)$, using the update operator \eqref{eq_HopfLax}.\\
}
}
\end{algorithm}
In \cref{fig_DifferentMetrics}, we illustrate some numerically computed geodesic distance maps and the corresponding minimal paths with respect to different geodesic metrics $\cF$. Fig.~\ref{fig_DifferentMetrics}a illustrates a vector field, denoted by $\varpi(\fx) := -(\fx-\fx_0)^\perp / \|\fx-\fx_0\|$ for any $\fx\in\overline\Omega\backslash\{\fx_0\}$ and $\varpi(\fx_0)=\mathbf{0}$,  using blue arrows, where $\fx_0$ is shown as a black dot at the image center. 
In other words, $\varpi(\fx)$ is tangent to the circle centered at $\fx_0$ and passing through $\fx$. In this example, the tensor field $\cM$ and the vector field $\omega$, defining the metric as in Eq.~\eqref{eq_RandersForm}, are respectively defined as
\begin{equation}
\cM(\fx):=\Id-a_1^2\varpi(\fx)\varpi(\fx)^\top, \text{~and~}\omega(\fx):=a_2\varpi(\fx),
\end{equation}
where $\Id$ is the identity matrix and where $a_1,\,a_2\in [0,\infty[$ are two constants. 
In the case where $a_1=a_2=0$, one obtains the Euclidean metric $\cF(\fx,\dfx)=\|\dfx\|$, for which the minimal geodesic path is a straight segment between the given points, as depicted in Fig.~\ref{fig_DifferentMetrics}b. In Fig.~\ref{fig_DifferentMetrics}c, we choose $a_1=0.95$ and $a_2=0$ to construct an anisotropic and symmetric Riemannian metric $\cR$, yielding a minimal geodesic path whose tangents are almost aligned with the orientation of the vector field $\varpi$.  Fig.~\ref{fig_DifferentMetrics}d illustrates the geodesic distances and the associated minimal geodesic path between the given points, associated to the Randers metric defined by the parameters $a_1=0.3$ and $a_2=0.8$. 
Note that the smallest eigenvalue $1-a_1^2>0$ of the matrix $\cM(\fx)$ is positive, and that the Randers  compatibility condition $\|\omega(\fx)\|_{\cM(\fx)^{-1}} = a_2/\sqrt{1-a_1^2} < 1$ is satisfied. One can observe that the tangents to the Randers minimal geodesic path are almost opposite to the direction of the vector field $\varpi$.

\section{Region-based Randers Geodesic Formalism}
\label{sec_AC}
The general goal of region-based active contour models is to seek shapes minimizing a well-chosen energy functional, so as to delineate image features of interest. In this section, we focus on the minimization of a region-based active contour energy based on a contour evolution scheme. 

\paragraph{Notations.} 
We denote by $\bT := \bR/\bZ$ the $1$-periodic interval, in other words $\bT = [0,1]$ with periodic boundary conditions.

Recall that $\Omega$ refers to a simply connected open and bounded domain. Subregions are denoted $S \subset \Omega$, are referred to as \emph{shapes}, and are assumed to be measurable. The characteristic function $\chi_S$ of a shape $S$ is regarded as an element of $\cX :=L^1(\Omega,\{0,1\})$. The distance between two shapes is defined as $\|\chi_S - \chi_\gS\|_1$, i.e.\ the area of their symmetric difference. 
We say that a shape $S$ is \emph{simple} if it is homeomorphic to a ball and has a rectifiable boundary, and in this case we denote by $\cC_S \in \Lip(\bT,\Omega)$ an arbitrary counter-clockwise parametrization of $\partial S$. We also fix $\cM \in C^1(\overline \Omega, S_2^{++})$ and denote by $\cR$ the associated Riemannian metric, see~\cref{eq_RiemannMetric}. 

We denote by $\|f\|_1$ the $L^1$ norm of a measurable function $f$ on its set of definition, and likewise $\|f\|_2$ and $\|f\|_\infty$ denote the $L^2$ and $L^\infty$ norms. The notation ``$C=C(a,b)$'' means that the constant $C$ only depends on the parameters $a$ and $b$. We denote by  $|\xi|_{C^\alpha}$ the $\alpha$-H\"{o}lder semi-norm of function $\xi : \overline \Omega \to \bR$, which is defined as the smallest constant such that $|\xi(\fx)-\xi(\fy)| \leq |\xi|_{C^\alpha} \, \|\fx-\fy\|^\alpha$ for all $\fx,\fy \in \overline \Omega$, where $0<\alpha<1$. 

\subsection{Region-based Energy Functional}
A typical region-based energy functional $E_0:\cX\to\bR$ is defined as the sum of a region-based image appearance term $\Psi:\cX\to\bR$ and of a regularization term $\TV:\cX\to\bR$, which can be expressed as 
\begin{equation}
\label{eq_RegionalEnergyTV}
E_0(S)=\alphareg\,\Psi(\chi_{\aS})+\TV(\chi_{\aS}),
\end{equation}
where $\alphareg\in[0,\infty[$ is a constant parameter modulating the importance of $\Psi$. 
In applications, the region-based appearance term $\Psi$ often measures the homogeneity of image features within and outside of the shape $S$, see~\cref{subsec:GradientExamples} for examples. The regularization term $\TV(\chi_{\aS})$ is the total variation norm of $\chi_{\aS}$ which penalizes the perimeter, or the Euclidean curve length, of the boundary $\partial{\aS}$. In this way, the regularization term is independent of the image data.

Hybrid active contour models such as~\citep{kimmel2003regularized,zach2009globally,bresson2007fast} provide a simple and effective avenue to integrate a region-based appearance model and an edge-driven weighted curve length associated to a Riemannian metric. 
In the following we limit our attention to \emph{simple shapes} $S$ (see notations), whose boundary can be parametrized by a Lipschitz curve $\cC_S$. This admittedly limits the scope of our method: we cannot segment several disconnected regions in a domain, or an annulus shaped region, in a single run of the curve evolution scheme.
The hybrid active contour energy is defined as 
\begin{equation}
\label{eq_HybridEnergy}
E(S)=\alphareg\,\Psi(\chi_{\aS})+
\Length_\cR(\cC_S),
\end{equation}
where $\cR$ is a Riemannian metric, defined in terms of a tensor field $\cM \in C^1(\overline \Omega, S_2^{++})$, and which is fixed in this paper, see Eq.~\eqref{eq_RiemannMetric}.  The tensor field $\cM$ is in practice data-driven in terms of the magnitude and the directions of the image gradient features, see~\cref{subsec_PracticalMetric}. For simplicity we set the parameter $\alphareg=1$ in this section, w.l.o.g.\ up to considering the region term $\alphareg \, \Psi$. The tuning of parameter $\alphareg$ does however play a significant role in practice, as discussed in~\cref{sec_SegmentationAlgs}. The image segmentation functional~\eqref{eq_HybridEnergy} is thus enhanced by the image anisotropy features~\citep{kimmel2003regularized}. Note that for images whose gradients are unreliable to define the edges of interest, the Euclidean curve length of $\cC_{\aS}$ should be taken as the regularization term, in other words $\cM\equiv \Id$ and one recovers the energy functional~\eqref{eq_RegionalEnergyTV}.  

We present a contour evolution scheme based on the numerical computation of successive Randers geodesic paths, which is summarized in~\cref{alg_ContourEvolution}. Our main result~\cref{th:summary} summarizes the analysis of this scheme, to which this whole section is devoted. It establishes that the proposed scheme yields a critical point of the hybrid active contour energy functional~\eqref{eq_HybridEnergy}, within the set of simple regions $S$ whose boundary $\partial S$ is contained within a thin tubular shaped subdomain $\rT \subset \Omega$, see Eq.~\eqref{eqdef:tubular_domain}. Our algorithm is entirely constructive, implementable, and numerically efficient, based on the eikonal solver of \cref{subsec_FMSolver} and as illustrated in the experiments of  Section~\ref{sec_Experiments}. The limitation of the search space however means that the optimal shape $S$ found by~\cref{alg_ContourEvolution} may satisfy $\partial S \cap \partial \rT\neq \emptyset$, in other words the evolution of the contour of the shape $S$ can be stuck by the boundary of the search space $\rT$, and the global minimizer of~\cref{eq_HybridEnergy} (over all possible shapes within $\Omega$) can be different. In practice, this shortcoming is addressed by dynamically updating the tubular search region as the algorithm progresses, see the optional \cref{line:update_tubular} of \cref{alg_ContourEvolution}, and \cref{fig:FixedPoints}. However, the theoretical analysis of this enhancement to the method appears to be complex and is outside the scope of this paper. Thus our approach does \emph{not} establish the existence of a globally optimal shape minimizing~\cref{eq_HybridEnergy} without restriction on the search space, which can be addressed using different (usually non-constructive) techniques as discussed in \cref{rem:minimization}.

In summary, the successive steps of our segmentation method \cref{alg_ContourEvolution} are described in the following places: \Cref{sec_PracticalRanders} for \cref{line:update_tubular},~\cref{subsec:GradientExamples} for \cref{line:compute_shape_gradient}, \cref{subsec_ExistenceVectorField} for \cref{line:solve_curl}, \cref{subsec_RandersInterpretation} for \cref{line:make_metric}, \cref{subsec_SearchSpace} for \cref{line:closed_geodesic}, and \cref{subsec_convergence_analysis} for the convergence analysis.

\begin{remark}[Existence of global minimizers for image segmentation functionnals]
\labelx{rem:minimization}
The mathematical well-posedness of variational image segmentation methods is for a large part founded on two celebrated results establishing the existence of minimizers to the piecewise smooth Mumford-Shah model~\citep{mumford1989optimal}, and to the piecewise constant approximation models~\citep{zhu1996region,cohen1997avoiding,chan2001active}, see~\citep{morel2012variational} for a proof. These works also encompass the more difficult problem of multi-region segmentation, and describe their triple point junctions. On the other hand, these results limit their scope to only two specific types of region-based homogeneity penalty terms $\Psi$, related to either piecewise smooth approximation, or to piecewise constant  approximation with a gradient squared penalization. The authors are not aware of a theory sufficiently general to accommodate the various non-linear and non-local region terms typically found in applications such as \citep{cremers2007review,jung2012nonlocal,michailovich2007image,bresson2008non}, based on e.g.\ their regularity properties. Nevertheless, the existence of the minimizers for these variants is generally regarded as a consequence of the two fundamental cases treated. Finally, let us also mention that the existence of minimizers is established for the fuzzy~\citep{li2010variational} and convex~\citep{chambolle2012convex,chan2006algorithms} relaxations of the segmentation problem.
\end{remark}

\begin{algorithm}[t]
\caption{\textsc{Contour Evolution Scheme}}
\label{alg_ContourEvolution}
\SetKwInOut{Input}{Input}
\SetKwInOut{Output}{Output}
\Input{An initial shape $S_0$, with $\partial S_0\subset \rT$ a tubular domain.}
\Output{An approximation of the optimal shape.}
\While{shape $S_n$ is not stabilized}{
Compute the Euclidean distance $d_{\partial S_n}$ from the shape boundary $\partial S_n$.\\
\label{line:eucl_distance}
Update the tubular neighborhood $\rT$ if needed. \hfill (Optional)\\ 
\label{line:update_tubular}
Compute the gradient $\xi_{S_n}$ of the region-based energy functional $\Psi$.\\
\label{line:compute_shape_gradient}
Compute the vector field $\omega_{S_n}$, obeying the PDE~\eqref{eq:curl_omega}.\\
\label{line:solve_curl}
Construct the Randers metric $\cF^{S_n}$, by formula \eqref{eq:rander_stokes}.\\
\label{line:make_metric}
Choose a point $\fx_n \in \partial S_n$, far from $\fx_{n-1}$ if $n \geq 1$.\\ 
Compute the geodesic curve $\cC_{n+1}$ from $\fx_n$ to itself, circling once around $\rT$, with respect to the metric $\cF^{S_n}$.\\
\label{line:closed_geodesic}
Define the shape $S_{n+1}$ as the region enclosed by $\cC_{n+1}$.\\
$n \gets n+1$.
}
\end{algorithm}

\begin{theorem}[Summary of the results of this section]
\label{th:summary}
Let $\Omega \subset \bR^2$ be a simply connected bounded open domain. Let $\cM \in C^1(\overline \Omega, S_2^{++})$ be a Riemannian metric. 
Consider a region-based energy functional $\Psi$ obeying 
\begin{equation*}
\Big| \Psi(\chi_\gS) + \int_\Omega (\chi_S-\chi_\gS) \xi_\gS \,\diff \fx - \Psi(\chi_S) \Big| 
\leq K_\Psi \|\chi_S - \chi_\gS\|_1^2,	
\end{equation*}
for any\footnote{We may require in addition that $\partial\gS$ and $\partial S$ are curves homotopic to $\rC$ within the tubular neighborhood $\rT$, see the definitions below.}
measurable shapes $\gS$ and $S$, where $\xi_\gS \in C^0(\overline \Omega)$ is (slightly abusively) referred to as the gradient of $\Psi$ at $\gS$.
 Assume that $\xi_\gS$ depends continuously on $\gS$ (where the distance between two shapes is defined as $\|\chi_S-\chi_\gS\|_1$), and that $\|\xi_\gS\|_\infty$ and $|\xi_\gS|_{C^\alpha}$ are bounded independently of $\gS$ for some $0<\alpha\leq 1$. 
Let also $\rC \in C^3(\bT,\Omega)$ be a fixed curve parametrized at unit Euclidean speed, and let $\rT$ denote the tubular neighborhood of $\rC$ of width $U>0$. 
Denote by $\Gamma_0$ the collection of curves of the form $\cC = \rC + \mu \rN$, for some $\mu \in \Lip(\bT,[-U,U])$, where $\rN := (\rC')^\perp$.

Then the following holds if $U$ is small enough. There exists a vector field $\omega_\gS \in C^1(\rT,\bR^2)$, depending continuously on the shape $\gS$, such that $\curl \omega_\gS(\fx) = \xi_\gS(\fx)$ and $\|\omega_\gS(\fx)\|_{\cM(\fx)^{-1}} < 1$ for all $\fx \in \rT$. 
Define the metric 
$$
\cF^\gS_\fx(\dfx) := (1+2\lambda\, d_{\partial\gS}(\fx)^2) \| \dfx\|_{\cM(\fx)} + \<\omega_\gS(\fx),\dfx\>,
$$
where $\lambda>0$ is a suitable constant specified later.
Define by induction a sequence of curves $\cC_n$, regions $S_n$, and points $\fx_n$ as follows. 
Initialization : let $\cC_0 = \rC+\mu_0 \rN \in \Gamma_0$ with $\Lip(\mu_0) \leq 1$, let $S_0$ be the bounded region whose boundary is $\cC_0$, and let $\fx_0$ be a point of $\cC_0$. 
Induction for $n \geq 1$ : let $\cC_n$ be a curve from $\fx_{n-1}$ to itself, of minimal length w.r.t.\ the metric $\cF^{S_{n-1}}$ in the homotopy class of $\rC$ within $\rT$. We prove that $\cC_n = \rC + \mu_n \rN \in \Gamma_0$ with $\Lip(\mu_n) \leq 1$. Let $S_n$ be the bounded region whose boundary is $\cC_n$, and let $\fx_n$ be a point of $\cC_n$ such that $\|\fx_n - \fx_{n-1}\| \geq \delta > 0$.

Then one has $\sum_{n \geq 0} \|\mu_{n+1}-\mu_n\|_2^2 < \infty$, and there exists a uniformly converging subsequence $\mu_n \to \mu_* \in \Lip(\bT,[-U,U])$. Denote $\cC_* := \rC+\mu_* \rN\in \Gamma_0$ and let $S_*$ be the corresponding region. 
This limit region has the following critical point property : let $\eta : \bT \to \bR^2$ be any perturbation such that $\cC_\ve := \cC_* + \ve \eta$ takes values in $\rT$ for all sufficiently small $\ve>0$, and let $S_\ve$ be the corresponding bounded region. Then as $\ve \to 0$
$$
	\Psi(\chi_{S_\ve}) + \Length_\cR(\cC_\ve) \geq \Psi(\chi_{S_*}) + \Length_\cR(\cC_*) + o(\ve).
$$
\end{theorem}

\begin{proof}
	The properties of $\omega_\gS$ are established in \cref{th:exists_omega}. The properties of the curve $\cC_n$ are established in \cref{corol:minS_wellposed}. The sequence $(\mu_n)_{n\geq 0}$ is studied in \cref{th:successive_min}, and the critical point property is established in \cref{th:critical}.
\qed
\end{proof}

\begin{remark}
	We discuss in~\cref{subsec:GradientExamples} various energy functionals $\Psi$, and the conditions under which they obey the regularity assumption of \cref{th:summary}. 
	A possible construction of the metric $\cM$ is suggested in \cref{subsec_PracticalMetric}, and its regularity is briefly discussed in \cref{rem:reg_M}.
	A number of ad-hoc enhancements and variants of the curve evolution scheme are also discussed in \cref{sec_PracticalRanders}, including adaptive tubular neighbourhoods. However we make no attempt at their theoretical study, and instead regard \cref{th:summary} as a generic and self contained foundation for our approach.
\end{remark}

\subsection{Construction of the Tubular Search Space}
\label{subsec_SearchSpace}

Our segmentation method, presented in~\cref{alg_ContourEvolution}, looks for shapes whose contour fits within a tubular domain $\rT$. In this section, we describe this search space, and we establish its elementary properties.
For that purpose, we fix a simple and periodic curve $\rC \in C^3(\bT, \Omega)$, referred to as the \emph{centerline} of the tubular domain, parametrized counter-clockwise and at unit Euclidean speed\footnote{
This assumption implies that $\rC$ has Euclidean length one, but there is no loss of generality since the problem can be rescaled.
}.
Let $\kappa \in C^1(\bT, \bR)$ denote the signed curvature of the curve $\rC$, and let $\kappa_{\max} := \|\kappa\|_\infty$. 
Define also $\rN(t) := \rC'(t)^\perp$, for all $t \in \bT$, so that $\rC'' = \kappa \rN$. Fig.~\ref{Fig_offsetCurves}a depicts an example of such a tubular domain $\rT$ as well as its centerline $\rC$.

We refer to~\citep[Sections 1.4, 2.2 and 2.4]{abate2012curves} respectively for the general properties of curvature, tubular neighborhoods, and the rotation index of a curve around a point, also known as the winding number, which are used below. See also~\citep{thale2008reach} on the local feature size of a curve, also known as the reach. We recall that the curvature $\kappa$ of a planar curve $\gamma$, such that $\gamma'$ is non-vanishing and Lipschitz, is defined as 
\begin{equation}
\label{eqdef:curvature}
	\kappa := \frac{\det(\gamma',\gamma'')}{\|\gamma'\|^3}.
\end{equation}

A tubular neighborhood of the curve $\rC$ can be parametrized using as coordinates the curvilinear abscissa $t$ and the deviation from the centerline, as follows
\begin{align*}
	\Phi &: \bT \times ]-\lfs(\rC),\lfs(\rC)[ \to \bR^2, &
	\Phi(t,u) &:= \rC(t) + u \rN(t).
\end{align*}
We denoted by $\lfs(\rC)>0$ the \emph{local feature size} of the curve $\rC$, which is the minimal Euclidean distance from $\rC$ to its medial axis \citep{choi1997mathematical}. In particular, the mapping $\Phi$ is injective and $\kappa_{\max} \lfs(\rC) \leq 1$. The other properties of $\Phi$ are summarized in the next proposition.

\begin{proposition}
\labelx{prop:Phi}
The Jacobian matrix of the mapping $\Phi$ reads
\begin{align}
\label{eq:JacM_Phi}
\JacM_\Phi(t,u) &=
R(t)
\cdot
\begin{pmatrix}
1-\kappa(t) u &~0\\
0~ &~1	
\end{pmatrix}, &	
\end{align}
where $R(t)$ is the rotation matrix of columns $(\rC'(t),\rN(t))$, for all $t \in \bT$ and $|u| < \lfs(\rC)$. In particular, the Jacobian determinant is $\Jac_\Phi(t,u) = 1-\kappa(t) u$, the mapping $\Phi$ is a $C^2$ diffeomorphism on its image, and likewise $\Phi^{-1}$.
\end{proposition}

\begin{proof}
A closely related result is proved in~\citep[Theorem 2.2.5]{abate2012curves}.
	Note that $\rN'(t) = (\rC''(t))^\perp = \kappa \rN(t)^\perp = -\kappa\rC'(t)$, for any $t \in \bT$.
	Direct differentiation thus yields $\partial_t \Phi(t,u) = \rC'(t) + u \rN'(t) = (1-\kappa(t) u) \rC'(t)$, and $\partial_u \Phi(t,u) = \rN(t)$, which establishes Eq.~\eqref{eq:JacM_Phi}.
	Since $\rC'(t)$ is a unit vector, and $\rN(t)$ is its image by a counterclockwise rotation, we obtain that $R(t)$ is a rotation matrix as announced. In particular $\det R(t) = 1$ and the announced expression of $\Jac_\phi$ follows.
	
	The mapping $\Phi$ has $C^2$-regularity, since $\rN'' = (\rC''')^\perp$ is continuous by assumption. It is globally injective by definition of $\lfs(\rC)$, and has a locally invertible Jacobian matrix since $\kappa_{\max} \lfs(\rC)\leq 1$, which implies $\Jac_\Phi(t,u)>0$. Therefore $\Phi$ defines a $C^2$ diffeomorphism on its image, and so does $\Phi^{-1}$ by the inverse function theorem.
\qed
\end{proof}

\begin{figure*}[t]
\setlength{\fboxsep}{8pt}%
\centering
\includegraphics[height=4.5cm]{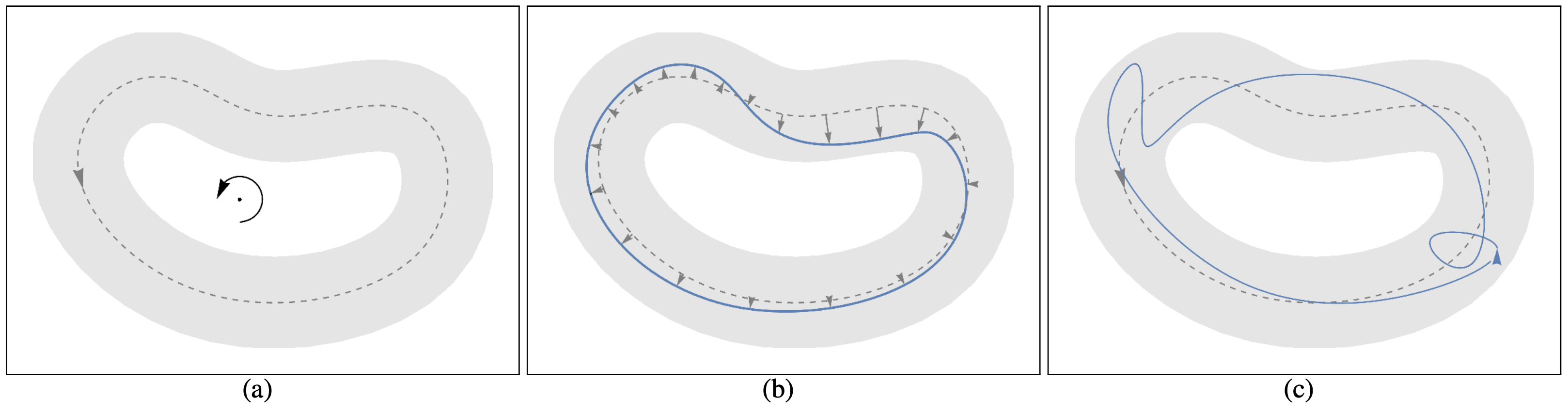}
\caption{\textbf{a} Illustration of the tubular domain $\rT$ (grayed region), the centerline $\rC$ (dashed line), and a point (black dot) within the bounded component of $\bR^2 \sm \rT$. Arrows indicate counterclockwise orientation. \textbf{b} A curve $\cC=\rC+\mu \rN \in \Gamma_0$ (blue solid line), defined by its deviation from the centerline $\mu$ (shown as arrows). Such a curve always bounds a shape $S \in \cX_0$. \textbf{c} A curve $\cC \in \Gamma_1$ which is indicated by the solid line.}
\label{Fig_offsetCurves}
\end{figure*}

We fix a tubular neighborhood $\rT$ of width $U$ of the reference curve $\rC$, defined as  
\begin{align}
\label{eqdef:tubular_domain}
	\rT &:= \Phi(\bT \times [-U,U]), & 
	0 &< U \leq \lfs(\rC)/3.
\end{align}
We assume that $\rT \subset \Omega$, up to reducing the width $U$, and note that further restrictions on $U$ are imposed in some results, such as \cref{th:geodesic_tubular} below.
Two sets of curves are of interest:
\begin{align*}
	\Gamma_1 &:= \{\cC \in \Lip(\bT, \rT)\mid \cC \text{ is homotopic to } \rC \text{ within } \rT\},\\
	\Gamma_0 &:= \{\cC = \rC+\mu \rN\mid \mu \in \Lip(\bT,[-U,U])\}.
\end{align*}
Clearly one has $\Gamma_0 \subset \Gamma_1$, using the homotopy $(\rC+s \mu \rN)_{s \in [0,1]}$, which continuously deforms the centerline $\rC$ into a curve of interest $\cC = \rC+\mu \rN \in \Gamma_0$ as the parameter $s$ varies from $0$ to $1$, while remaining within $\rT$ . In Figs.~\ref{Fig_offsetCurves}b and~\ref{Fig_offsetCurves}c, we illustrate two instances of curves in the sets $\Gamma_0$ and $\Gamma_1$, respectively.
Homotopy to the centerline $\rC$ within $\rT$, which defines the set of curves $\Gamma_1$, is the rigorous meaning of the informal statement ``circling once around $\rT$'' which appears in line~\ref{line:closed_geodesic} of~\cref{alg_ContourEvolution}.

The convergence analysis of algorithms is simpler within $\Gamma_0$, but the set of curves $\Gamma_1$ is the natural search space for geodesics. Fortunately, if the tube width $U$ is small enough, then the geodesic paths of interest in this paper belong to $\Gamma_0$, as shown by the following result established in Appendix~\ref{Appendix:geodesic_tubular}. 
\begin{theorem}
\label{th:geodesic_tubular}
Let $\cF$ be a Randers metric on $\rT$, with Lipschitz parameters, and let $\fx \in \rT$. Then there exists a minimizer to 
\begin{equation}
\label{eq:argmin_path}
	\min\left\{\Length_\cF(\cC)\mid \cC \in \Gamma_1, \,\cC(0) = \fx\right\}.
\end{equation}
Such minimizers $\cC$ have a curvature bounded by $\kappa_\cF$, except at $\cC(0)$, and if  $U\leq U_\cF$ they can be parametrized as  $\cC = \rC+\mu \rN \in \Gamma_0$ where $\Lip(\mu) \leq 1$. The constants $\kappa_\cF$ and $U_\cF>0$ only depend on $\kappa_{\max}$, $\rho_{\min}(\cF)$, $\rho_{\max}(\cF)$, and the Lipschitz constant of the coefficients of $\cF$.
\end{theorem}

\noindent\emph{Introduction of a wall for the numerical implementation.}
The numerical computation of a minimizer to the problem~\eqref{eq:argmin_path} is entirely practical using an eikonal solver such as the FM-ASR presented in~\cref{subsec_FMSolver}, which is a crucial point for our implementation of~\cref{alg_ContourEvolution}. For that purpose, we define 
\begin{equation}
	\label{eqdef:wall}
	\rW(t) := \{\rC(t)+u\rN(t)\mid u \in [-U,U]\},
\end{equation}
which is referred to as the \emph{wall} at abscissa $t \in \bT$. Assume that $\fx = \rC(t)+u \rN(t)$ for some $u \in [-U,U]$ and compute the shortest path from $\fx_\ve^+ := \rC(t+\ve)+u \rN(t)$ to $\fx_\ve^- := \rC(t-\ve)+u \rN(t)$ in the domain $\rT \sm \rW(t)$. Then letting $\ve\to 0$ we obtain a solution to~\cref{eq:argmin_path}. In the discrete setting, $\fx_\ve^+$ and $\fx_\ve^-$ may instead be chosen as the grid points immediately adjacent to $\fx$ and on the correct side of $\rW(t)$. In contrast if the constraint $\cC(0) = \fx$ is removed from~\cref{eq:argmin_path}, then this problem remains well posed mathematically, but it becomes more complex numerically (it amounts to the computation of a closed minimal path, as opposed to a minimal path between two points) and typically requires an iterative method~\citep{appleton2005globally}.

\begin{definition}
\labelx{def:cX01}
	We denote by $\cX_1$ (resp.\ $\cX_0$) the collection of closed shapes whose boundary, suitably parametrized, belongs to $\Gamma_1$ (resp.\ $\Gamma_0$). 
	If $S \in \cX_1$, then we let $\cC_S\in \Lip(\bT,\rT)$ denote the corresponding parametrization of the shape boundary $\partial{S}$. If $S \in \cX_0$, then the boundary parametrization can be chosen with the form $\cC_S(t)=\rC(t)+\mu_S(u)\rN(t)$ where  $\mu_S \in \Lip(\bT,[-U,U])$ denotes the deviation from the centerline of the tubular region $\rT$.	
\end{definition}

Note that $\cX_0 \subset \cX_1$, and that by construction any curve $\cC = \rC+\mu \rN \in \Gamma_0$ is simple, (i.e.\ without self intersections, this may not hold in $\Gamma_1$, as illustrated on \cref{Fig_offsetCurves}), hence there is some shape $S \in \cX_0$ whose parametrized boundary satisfies $\cC_S = \cC$ and $\mu_S = \mu$. 
In fact this shape $S$ can be obtained explicitly as the union of $\{\rC(t)+u\rN(t)\mid t\in \bT, \mu(t) \leq u \leq U\}$ with the bounded component of $\bR^2 \sm \rT$.

So far we have attached three objects to a suitable shape $S$: (i) the characteristic function $\chi_S$, (ii) the contour $\cC_S$, and if $S \in \cX_0$ (iii) the deviation $\mu_S$ from the centerline of the tubular domain. This leads to different ways to compare shapes, which are examined in the rest of this subsection.
\begin{proposition}
\labelx{prop:mu_chi}
For any shapes $S,\gS\in \cX_0$ one has
\begin{equation}
	\tfrac 2 3 \|\mu_S-\mu_\gS\|_1 \leq \|\chi_S - \chi_\gS\|_1 \leq \tfrac 4 3 \|\mu_S-\mu_\gS\|_1. 
\end{equation}
\end{proposition}

\begin{proof}
The area of the symmetric set difference reads 
\begin{align}
	\label{eq:band_area}
	\int_\Omega |\chi_S - \chi_\gS| \, \diff \fx 
	&= \int_\bT \int_{\mu_{\min}(t)}^{\mu_{\max}(t)} \Jac_\Phi(t,u) \,\diff u \diff t\\
	\nonumber
	&= \int_\bT |\mu_S-\mu_\gS| \left(1-\kappa \frac{\mu_S+\mu_\gS} 2 \right) \,\diff t,
\end{align}
where $\mu_{\min}(t):= \min\{\mu_S(t),\mu_\gS(t)\}$ and $\mu_{\max}(t):= \max\{\mu_S(t),\mu_\gS(t)\}$. We conclude noting that $|\kappa (\mu_S+\mu_\gS)/2|\leq \kappa_{\max} U \leq 1/3$.
\qed
\end{proof}

The following geometric quantity evaluates the proximity between the contours of two shapes $\gS,S\in \cX_1$:
\begin{equation}
	\label{eq:diver_SS}
	D(S \| \gS) := \int_\bT d_{\partial \gS}(\cC_S(t))^2 \, \cR(\cC_S(t),\cC_S'(t)) \, \diff t,
\end{equation}
where $\cR$ is a fixed Riemannian metric \eqref{eq_RiemannMetric}.
We denoted by $d_{\partial \gS}$ the Euclidean distance from the boundary $\partial \gS$, in other words from the parametrized contour $\cC_\gS$. Numerically, $d_{\partial \gS}$ can be computed using the standard isotropic fast marching method \citep{sethian1999fast}, or alternatively with techniques specialized for the Euclidean distance \citep{fabbri20082d}. Quantities such as $D(S\|\gS)$, which is non-negative and vanishes only if $S=\gS$, are often referred to as \emph{divergences} in statistics and machine learning.

\Cref{th:divergence_dominates} below, preceded with a technical lemma, compares $D(S \| \gS)$ with the area $\|\chi_S-\chi_\gS\|_1$ of the symmetric difference between the sets $S$ and $\gS$, and with the norm $\|\mu_S-\mu_\gS\|_2$ of the difference between their centerline deviations.

\begin{lemma}
\labelx{lem:distance_centerline}
	Let $\gS \in \cX_0$ be such that $\Lip(\mu_\gS)\leq 1$, and let $\fx = \rC(t) + u \rN(t)$ with $t \in \bT$ and $u\in [-U,U]$. Then
	\begin{equation}
	\label{eq:dist_centerline}
		|u-\mu_\gS(t)|/3 \leq d_{\partial \gS}(\fx) \leq |u-\mu_\gS(t)|.
	\end{equation}
\end{lemma}

\begin{proof}
Choosing $\fy_* := \rC(t)+\mu_\gS(t) \rN(t) \in \partial \gS$, and observing that $d_{\partial \gS}(\fx) \leq \|\fx-\fy_*\| = |u-\mu_\gS(l)|$, we establish the announced upper bound (\ref{eq:dist_centerline}, right).

Conversely, let $s \in \bT$ and $\fy := \rC(s)+\mu_\gS(s) \rN(s) \in \partial \gS$ be such that $d_{\partial \gS}(\fx) = \|\fx-\fy\|$.
Assume for contradiction that $\|\fx-\fy\| < |u-\mu_\gS(t)|/3$, otherwise there is nothing to prove. It follows that $\|\fx-\fy\| \leq 2U/3$, and therefore 
\begin{equation*}
	[\fx,\fy] \subset \overline B(\fx,\tfrac{1}{3}U)\cup \overline B(\fy,\tfrac{1}{3}U) \subset \Phi(\bT \times [-\tfrac 4 3 U,\tfrac 4 3 U])=:\rT'.
\end{equation*}
We have shown that the segment $[\fx,\fy]$ is entirely contained within a slightly extended tubular domain $\rT' \supset \rT$, where by \cref{prop:Phi} one has $\|\JacM_{\Phi^{-1}}\| \leq 1/(1-\frac{4}{3}\kappa_{\max} U) \leq 9/5$ since $\kappa_{\max} U\leq \frac 1 3$. Note that $\|\JacM_{\Phi^{-1}}\|$ is the operator norm of the matrix $A:=\JacM_{\Phi^{-1}}$, i.e.\ the square root of the largest eigenvalue of $A^\top A$, which is easily computed since $\JacM_{\Phi}$ is obtained as the product of an isometry and of a diagonal matrix.
Therefore, we obtain
\begin{equation*}
	\|(s,\mu_\gS(s)) - (t,u)\| = \|\Phi^{-1}(\fy) - \Phi^{-1}(\fx)\| \leq \tfrac{9}{5}\|\fy-\fx\|.
\end{equation*}
Finally, we conclude 
\begin{align*}
	|\mu_\gS(t)-u| &\leq |s-t| + |\mu_\gS(s)-u|\\
	&\leq \sqrt 2 \sqrt{(s-t)^2 + (\mu_\gS(s)-u)^2} \\
	& \leq \frac{9 \sqrt{2}}{5}\|\fy-\fx\| \leq 3 d_\gS(\fx),
\end{align*}
using successively (i) $\Lip(\mu_\gS) \leq 1$, (ii) the quadratic mean inequality, (iii) the previous inequality and $\frac{9 \sqrt 2} 5 < 3$. 
\qed
\end{proof}

\begin{theorem}
\label{th:divergence_dominates}
Let $\gS\in \cX_0$ with $\Lip(\mu_\gS) \leq 1$. Then
\begin{align}
\label{eq:div_dominates_area}
	&\|\mu_S-\mu_\gS\|_2^2 \leq K_0 D(S\|\gS), &
	& \text{for all } S \in \cX_0,\\
\label{eq:div_dominates_centerline}
	&\|\chi_S-\chi_\gS\|_1^2 \leq K_1 D(S\|\gS), & 
	& \text{for all } S \in \cX_1,
\end{align}
where $K_0 = 27/(2\rho_{\min}(\cR))$, and $K_1=48/ \rho_{\min}(\cR)$. 
\end{theorem}

The rest of this subsection is devoted to the proof. For the first point, since $S \in \cX_0$ we can assume that $\cC_S=\rC+\mu_S \rN$. Thus for a.e.\ $t \in \bT$
\begin{align}
\nonumber
	\|\cC'_S(t)\| 
	&= \| \rC'(t) + \mu_S(t) \rN'(t) + \mu_S'(t) \rN(t)\|\\
\nonumber
	&= \sqrt{(1-\kappa(t)\mu_S(t))^2 + \mu_S'(t)^2} \\
\label{eq:velocity_lower_bound}
	&\geq 1-U \kappa_{\max} \geq 2/3,
\end{align}
recalling that $(\rC'(t), \rN(t))$ is an orthonormal basis, and that $\rN'(t) = -\kappa(t) \rC'(t)$. 
It follows that $\cR(\cC_S(t),\cC'_S(t)) \geq  \tfrac 2 3 \rho_{\min}(\cR)$, hence by~\cref{eq:dist_centerline}
\begin{equation*}
	D(S\|\gS) \geq \int_\bT \frac{|\mu_S(t)-\mu_\gS(t)|^2}{3^2} \frac{2\rho_{\min}(\cR)} 3 \diff t,
\end{equation*}
which establishes \eqref{eq:div_dominates_centerline}. The proof of \eqref{eq:div_dominates_area} is similar in spirit, but is a bit more technical because one only assumes $S \in \cX_1$. 
Define $\mu^+_S, \mu^-_S : \bT\to [-U,U]$ as 
\begin{align*}		
	\mu_S^+(t) &:= \max \{u\in [-U,U]\mid \rC(t)+u\rN(t) \in S\},\\
	\mu_S^-(t) &:= \min \{u\in [-U,U]\mid \rC(t)+u\rN(t) \in \overline{\bR^2 \sm S}\}. 
\end{align*}
Let also $\bT_S^-,\bT_S^+ \subset \bT$ be defined as $\bT_S^- := \{\mu_S^-<\mu_\gS\}$ and $\bT_S^+ := \{\mu_S^+>\mu_\gS\}$. Observing that $\chi_S-\chi_\gS$ is supported within the band defined by $\mu_*^- := \min\{\mu_S^-,\mu_\gS\}$ and $\mu_*^+:=\max \{\mu_S^+,\mu_\gS\}$, we obtain with $\cC_S^\pm := \rC+\mu^\pm_S \rN$
\begin{align*}
&\|\chi_S-\chi_\gS\|_1 \\
&\leq\int_\bT |\mu_*^--\mu_*^+| (1-\kappa (\mu_*^-+\mu_*^+)/2 ) \,\diff t\\
&\leq \tfrac 4 3 \Big( \int_{\bT_S^-} |\mu_S^- - \mu_\gS| \,\diff t+ \int_{\bT_S^+} |\mu_S^+ - \mu_\gS| \,\diff t\Big)\\
&\leq 4 \Big( \int_{\bT_S^-} d_{\partial\gS}(\cC_S^-) \,\diff t+ \int_{\bT_S^+} d_{\partial\gS}(\cC_S^+) \,\diff t\Big),
\end{align*}
using (i) the same expression of the area of a band as in~\cref{eq:band_area}, (ii) the identity $\mu_*^--\mu_*^+ = \chi_{\bT_S^+} (\mu_S^+-\mu_\gS)+ \chi_{\bT_S^-} (\mu_\gS - \mu_S^+)$ along with the estimate $|\kappa \mu^\pm_*| \leq \kappa_{\max} U \leq 1/3$, and (iii) \cref{lem:distance_centerline}. 
On the other hand observing that the curves $\cC_S^+$ and $\cC_S^-$ are continuous except at isolated points (ignored in the integral below), and a.e.\ differentiable we obtain
\begin{align*}
&D(S\|\gS)/\rho_{\min}(\cR) \\
& \geq \int_{\bT_S^-} d_{\partial\gS}(\cC_S^-)^2 \|(\cC_S^-)'\| \,\diff t+ \int_{\bT_S^+} d_{\partial\gS}(\cC_S^+)^2 \|(\cC_S^+)'\| \,\diff t\\
& \geq \tfrac 2 3 \Big(\int_{\bT_S^-} d_{\partial\gS}(\cC_S^-)^2 \,\diff t+ \int_{\bT_S^+} d_{\partial\gS}(\cC_S^+)^2 \,\diff t\Big)\\
& \geq \tfrac 1 3\Big(\int_{\bT_S^-} d_{\partial\gS}(\cC_S^-) \,\diff t+ \int_{\bT_S^+} d_{\partial\gS}(\cC_S^+) \,\diff t\Big)^2,
\end{align*}
using (i) the fact that $\cC_S^+$ and $\cC_+^-$ parametrize disjoint sections of the boundary of the shape $S$, (ii) the estimate \eqref{eq:velocity_lower_bound} applied to $\cC_S^\pm := \rC+\mu^\pm_S \rN$, (iii) the Cauchy-Schwartz inequality, noting that $\int_{\bT_S^+} 1 \,\diff t + \int_{\bT_S^-} 1 \,\diff t \leq 2\int_\bT 1 \,\diff t = 2$. Combining the obtained upper bound on $\|\chi_S-\chi_\gS\|_1$ and the lower bound on $D(S \| \gS)$ yields \eqref{eq:div_dominates_centerline} and concludes the proof.

\subsection{Reformulating the Region-based Energy Functional Via a Randers Metric}
\label{subsec_RandersInterpretation}

We introduce an approximation of the active contour functional \eqref{eq_HybridEnergy} which is overestimating, and is accurate up to second order, see \cref{prop:relaxation_close}. Using Stokes formula, we show that this approximation takes the form of a Randers length of the shape contour, see \cref{th:stokes}. This allows us to reformulate a shape optimization problem in the form of a minimal geodesic problem, see \cref{corol:minS_wellposed}.

For that purpose, we need the region-based functional $\Psi$ to obey the following regularity property: for any shapes $\gS\in \cX_0$ and $S \in \cX_1$
\begin{equation}
\label{eq:taylor_phi}
	\Big| \Psi(\chi_\gS) + \int_\Omega (\chi_S-\chi_\gS) \xi_\gS \,\diff \fx - \Psi(\chi_S) \Big| 
	\leq K_\Psi \|\chi_S - \chi_\gS\|_1^2,
\end{equation}
where $K_\Psi>0$ is a constant, and $\xi_\gS \in C^0( \overline\Omega , \bR)$ is a scalar-valued function referred to as the region-based energy gradient.
The equation~\eqref{eq:taylor_phi} should be regarded as a Taylor expansion of the region-based energy functional $\Psi$. 
Some additional assumptions on $\xi_\gS$, including H\"{o}lder regularity and continuity w.r.t.\ the parameter $\gS$, are introduced later in~\cref{th:exists_omega}.

As discussed in~\cref{subsec:GradientExamples}, the expansion~\eqref{eq:taylor_phi} can be established for a large range of functionals in image analysis, and $\xi_\gS$ is obtained explicitly by differentiation. 
The terminology ``region-based energy gradient'', which admittedly is slightly abusive, is justified in~\cref{rem:hilbert_gradient}.
In the case of two specific region-based energy functionals, namely the balloon model~\citep{cohen1991active} and the piecewise constants-based active contours model~ \citep{chan2001active,cohen1997avoiding},  the region-based energy gradient $\xi_\gS$ is found to be independent of the shape $\gS$, and the regularity constant $K_\Psi = 0$ is formally admissible.

\begin{remark}
\labelx{rem:shape_gradient}
A common approach in shape optimization \citep{sokolowski1992introduction} is to consider variations of the full energy functional, here $\gS \mapsto \Psi(\chi_\gS)+ \Length_\cR(\cC_\gS)$ or and in other applications an energy defined e.g.\ from a PDE on the domain $\gS$, subject to perturbations of the region contour $\cC_\gS$ in the normal direction $(\cC'_\gS)^\perp$. 
The resulting first order term $\zeta_\gS : \partial \gS \to \bR$, often referred to as the ``shape gradient'', is only defined along the contour boundary. We do not follow this approach here, and we emphasize that $\xi_\gS : \overline\Omega \to \bR$ is defined on the whole image domain, but only takes into account the region-based energy functional $\Psi$ (not $\Length_\cR$).
\end{remark}

Following a common practice in optimization algorithms (e.g.\ sequential quadratic programming~\citep{boggs1995sequential}), we next introduce an approximation $E_\gS$ of the energy functional, where $\gS$ is set as some reference shape. It is obtained by simultaneously linearizing the non-linear term $\Psi$ at $\gS$, and penalizing the divergence $D(S\|\gS)$. 

Precisely, for all shapes $\gS\in \cX_0,$ $S \in \cX_1$ we define
\begin{equation*}
\label{eq_ApproximateEnergy}
	E_\gS(S) := e_\gS + \int_S \xi_\gS \, \diff \fx  + \Length_\cR(\cC_S) + 2\lambda D(S\| \gS),
\end{equation*}
where $e_\gS := E(\gS) - \int_\gS \xi_\gS\, \diff \fx$ and $\lambda := K_1 K_\Psi$. Note that the constant $K_1$ is defined in Eq.~\eqref{eq:div_dominates_centerline}, and $K_\Psi$ in Eq.~\eqref{eq:taylor_phi}.

\begin{proposition}
\labelx{prop:relaxation_close}
For all $\gS\in \cX_0$, $S \in \cX_1$, with $\Lip(\mu_\gS) \leq 1$, we have
\begin{equation*}
	E(S)+\lambda D(S\|\gS) \leq E_\gS(S)\leq E(S)+3 \lambda D(S \|\gS).
\end{equation*}
\end{proposition}

\begin{proof}
	The estimate \eqref{eq:taylor_phi} can be rewritten as 
	\begin{equation*}
		|E_\gS(S) - E(S) - 2\lambda D(S\|\gS)| \leq K_\Psi \|\chi_S-\chi_\gS\|_1^2.
	\end{equation*}
	The announced result then follows from \cref{th:divergence_dominates}.
\qed
\end{proof}

The presence of the divergence term $D(S \|\gS)$ in $E_\gS(S)$ is essential for the theoretical analysis of \cref{alg_ContourEvolution}, which will be presented in \cref{subsec_convergence_analysis}. From the numerical standpoint, this term is easily implemented (although the practical choice of $\lambda$ would deserve some discussion), but empirically it does not appear to be necessary for the stability and the convergence of the iterations, hence it is usually omitted.

The following step requires an additional ingredient, which is the solution to a curl PDE\footnote{$\curl \omega = - \diver(\omega^\perp) = \partial_x \omega_y - \partial_y \omega_x$ if $\omega = (\omega_x,\omega_y)$}.
Precisely, \cref{th:exists_omega} in \cref{subsec_ExistenceVectorField} establishes, under suitable assumptions and possibly by reducing the tubular domain width $U$, that 
\begin{equation}
\label{eq:curl_omega}
	\curl \omega_\gS = \xi_\gS \quad \text{on } \rT,
\end{equation}
for some vector field $\omega_\gS \in C^1(\rT, \bR^2)$.
In addition, one has $\|\omega_\gS(\fx)\|_{\cM(\fx)^{-1}} \leq 1/2$ for all $\fx \in \rT$, the Lipschitz constant $\Lip(\omega_\gS)$ is bounded independently of $\gS$, and $\omega_S \to \omega_\gS$ uniformly as $\mu_S \to \mu_\gS$ uniformly with $S \in \cX_0$.

The following two results make the connection between the approximate energy $E_\gS$ and the framework of Randers geometry. The main ingredient of \cref{th:stokes} is the Stokes formula, which is equivalent to Green's divergence theorem in two dimensions, and is applicable thanks to the PDE~\eqref{eq:curl_omega}. This identity allows reformulating the integral over the shape $S$ appearing in $E_\gS$, into an integral over the boundary $\cC_S$, which then defines the asymmetric term of a Randers metric. 
\begin{theorem}
\label{th:stokes}
One has $E_\gS(S) = e'_\gS + \Length_{\cF^\gS}(\cC_S)$, for any shapes $\gS\in \cX_0$, $S \in \cX_1$, where $e'_\gS$ is a constant and $\cF^\gS$ denotes the following Randers metric
\begin{equation}
\label{eq:rander_stokes}
	\cF^\gS_\fx(\dfx) := (1+2\lambda\, d_{\partial\gS}(\fx)^2) \| \dfx\|_{\cM(\fx)} + \<\omega_\gS(\fx),\dfx\>.
\end{equation}
Furthermore, $\rho_{\min}(\cF^\gS)$, $\rho_{\max}(\cF^\gS)$, and the Lipschitz constant of the coefficients of $\cF^\gS$, are bounded independently of the shape $\gS$. 
\end{theorem}

\begin{proof}
Defining $\cM_\gS := (1+2\lambda\,d_{\partial \gS}(\fx)^2)^2 \cM$, we check that  
$\cF^\gS_\fx(\dfx) = \| \dfx\|_{\cM_\gS(\fx)} + \<\omega_\gS(\fx),\dfx\>$ for all $\fx \in \rT$, $\dfx \in \bR^2$, thus $\cF^\gS$ does have the structure of a Randers metric \eqref{eq_RandersForm}.

Denote by $\tilde{S} \in \cX_0$ the innermost shape, whose boundary reads $\cC_{\tilde{S}} = \rC+U \rN$, equivalently $\mu_{\tilde S} \equiv U$. 
By construction one has $\tilde{S}\subset S$ and $S \sm \tilde{S} \subset \rT$.
The Stokes formula, applied to the region $S \sm \tilde S$, on which \eqref{eq:curl_omega} holds, yields
\begin{align*}
\int_{S\sm \tilde S} \xi_\gS \,\diff \fx &=\int_S \xi_\gS \,\diff \fx-\int_{\tilde{S}} \xi_\gS \,\diff \fx\\
&=\int_\bT \<\omega_\gS\circ\cC_S, \cC_S'\> \,\diff t 
- \int_\bT \<\omega_\gS\circ \cC_{\tilde S}, \cC_{\tilde S}'\> \,\diff t.
\end{align*}
Therefore
\begin{equation*}
\int_S \xi_\gS \,\diff \fx = c_\gS + \int_\bT \<\omega_\gS\circ\cC_S, \cC_S'\> \,\diff t,
\end{equation*}
where $c_\gS$ is a scalar valued that is defined as
\begin{equation*}
c_\gS := \int_{\tilde S} \xi_\gS \diff \fx -\int_\bT \<\omega_\gS\circ \cC_{\tilde S}, \cC_{\tilde S}'\> \,\diff t.	
\end{equation*}
The identity $E_\gS(S) = e_\gS+c_\gS + \Length_{\cF^\gS}(\cC_S)$ follows, in view of the expression~\eqref{eq:diver_SS} of $D(S\| \gS)$.

One has $\| \omega_\gS(\fx)\|_{\cM(\fx)^{-1}} \leq 1/2$ by assumption, hence the metric $\cF^\gS$ is definite and $\rho_{\min}(\cF^\gS) \geq \rho_{\min}(\cR)/2$ by Eq.~\eqref{eq:Randers_bounds}. Noting that $\|d_{\partial\gS}\|_\infty\leq 2 U$ by construction~\eqref{eq:diver_SS} on the set $\rT$, we obtain that $\rho_{\max}(\cF^\gS) \leq 2 (1+8 \lambda U^2) \rho_{\max}(\cR)$ by Eq.~\eqref{eq:mumax_bounds}. 
One has $\Lip(d_{\partial\gS}) = 1$, since $d_{\partial \gS}$ is the Euclidean distance function from a set, hence the tensor field  $(1+2\lambda\,d_{\partial \gS}(\fx)^2)^2 \cM$ has a Lipschitz constant bounded independently of the shape $\gS$. On the other hand $\Lip(\omega_\gS)$ is bounded independently of $\gS$ by assumption. 
\qed
\end{proof}

For readability, we denote $\Length_\gS := \Length_{\cF^\gS}$ and $\Dist_\gS := \Dist_{\cF^\gS}$ in the following, for any $\gS\in \cX_1$.

\begin{corollary}
\labelx{corol:minS_wellposed}
	Assume that $U$ is sufficiently small. 
	Then for any $\gS \in \cX_0$ with $\Lip(\mu_\gS) \leq 1$, and any $\fx \in \rT$, the optimization problem
	\begin{equation}
	\label{eq:argmin_region}
		\min\{ E_\gS(S)\mid S \in \cX_1,\, \fx \in \partial S\}.
	\end{equation}
	admits a minimizer. 
	More precisely, $S \mapsto \cC_S$ defines a bijection between the minimizers of~\eqref{eq:argmin_path} with $\cF := \cF^\gS$, and the minimizers of~\eqref{eq:argmin_region}. As a result, any such minimizer satisfies $S \in \cX_0$ and $\Lip(\mu_S) \leq 1$.
\end{corollary}

\begin{proof}
First observe that~\cref{th:geodesic_tubular} applies to the metric $\cF^\gS$, provided the tubular region width $U$ is sufficiently small, in view of the properties established in \cref{th:stokes}.

Define $E_\gS^{\min}$ as the minimum value \eqref{eq:argmin_region}, and $\Length_\gS^{\min}$ as the minimum value \eqref{eq:argmin_path} using $\cF := \cF^\gS$. For any shape $S \in \cX_1$, the boundary curve $\cC_S\in \Gamma_1$, hence $c_\gS+\Length_\gS^{\min} \leq E_\gS^{\min}$ by \cref{th:stokes}. On the other hand, the optimization problem \eqref{eq:argmin_path} admits a minimizer $\cC_*$, by \cref{th:geodesic_tubular}, obeying $\Length_\gS(\cC_*) = \Length_\gS^{\min}$ and $\cC_*\in \Gamma_0$. 
As observed below \cref{def:cX01}, there exists a shape $S_* \in \cX_0$ such that $\cC_{S_*} = \cC_*$. Hence $E_\gS^{\min} \leq E_{\gS}(S_*) \leq  c_\gS + \Length_\gS(\cC_*) = c_\gS +\Length_\gS^{\min}$. Therefore $E_\gS^{\min} = c_\gS+\Length_\gS^{\min}$, and from this point the result easily follows.
\qed
\end{proof}

\subsection{Minimization of the Active Contour Energy}
\label{subsec_convergence_analysis}

This subsection is devoted to the convergence analysis of \cref{alg_ContourEvolution}, which is based on a recursive sequence of minimizations of the approximate energy $E_\gS$.
In a strict understanding, our analysis falls short of proving that the shapes $(S_n)_{n \geq 0}$ produced by the algorithm converges, or that  a minimizer of the active contour energy $E$ is obtained in the process. Instead, we establish that the distance between successive shapes $S_n$ and $S_{n+1}$ tends to zero, as stated in~\cref{th:successive_min}, and that they cluster around a shape $\gS$ which is a critical point of $E$, see \cref{th:critical}. A number of other properties are also proved, such as the decrease of the energies $(E(S_n))_{n \geq 0}$ and a geodesic property of $\cC_\gS$, see the precise statements of \cref{th:successive_min,th:critical}, which provide additional guarantees for our algorithm. We assume below that the tube width $U$ is sufficiently small so that \cref{corol:minS_wellposed} holds.

\begin{theorem}
\label{th:successive_min}
Let $S_0 \in \cX_0$ be such that $\Lip(\mu_{S_0}) \leq 1$, and let for all $n \geq 0$
\begin{equation}
	\label{eq:successive_min}
	S_{n+1} \in \arg\min \{E_{S_n}(S)\mid S\in \cX_1, \fx_n \in \partial S\},
\end{equation} 
where $\fx_n \in \partial S_n$ is arbitrary. 
Then the sequence $(S_n)_{n \geq 0}$ is well defined, $S_n \in \cX_0$, and $\Lip(\mu_{S_n}) \leq 1$. 
Furthermore, the energies $(E(S_n))_{n\geq 0}$ are decreasing and converging, and 
\begin{align}
\label{eq:sum_sqnorms}
	\sum_{n \geq 0} \| \mu_{n+1} -\mu_n \|_2^2 &< \infty, &
	\text{with } \mu_n := \mu_{S_n}.
\end{align} 
\end{theorem}

\begin{proof}
By \cref{corol:minS_wellposed} one has $S_1 \in \cX_0$ and $\Lip(\mu_1) \leq 1$. By induction, these properties are satisfied for all $n\geq 0$, and the sequence $(S_n)_{n\geq 0}$ is well defined.

By \cref{prop:relaxation_close} and Eq.~\eqref{eq:successive_min}, one has 
\begin{align}
\label{eq:energy_decreases_diff}
&E(S_{n+1})+\lambda D(S_{n+1}\|S_n) \\
\nonumber
&\leq E_{S_n}(S_{n+1}) \leq E_{S_n}(S_n) = E(S_n),
\end{align}
hence $(E(S_n))_{n\geq 0}$ decreases, since $D(S_{n+1}\|S_n)\geq 0$. 
Observing that the region-based term $\Psi(S)$ is bounded over all $S\in \cX_0$ by Eq.~\eqref{eq:taylor_phi}, and that the boundary term $\Length_\cR(\cC_S)$ is non-negative, we obtain that the total energy $E$ is bounded below on $\cX_0$. Therefore $E(S_n)$ converges as $n \to \infty$ to a limit denoted $E_\infty$. 
In addition, by Eq.~\eqref{eq:energy_decreases_diff}
\begin{equation*}
	\lambda \sum_{n \geq 0} D(S_{n+1}\|S_n) \leq \sum_{n \geq 0} \big(E(S_n)-E(S_{n+1})\big),
\end{equation*}	
hence $\sum_{n \geq 0} \| \mu_{n+1} -\mu_n \|_2^2 \leq K_0 (E(S_0)-E_\infty)/\lambda$ by Eq.~\eqref{eq:div_dominates_centerline}, which establishes the estimate~\eqref{eq:sum_sqnorms} and concludes the proof.
\qed
\end{proof}
The property~\eqref{eq:sum_sqnorms} is not sufficient to ensure the convergence of the sequence $(\mu_n)_{n\geq 0}$, of deviations from the centerline of the tubular domain $\rT$. 
Nevertheless, the following result establishes that a cluster point, denoted by $\mu_\gS$, exists, defining a shape $\gS \in \cX_0$ whose contour $\cC_\gS$ is a circular geodesic path~\eqref{eq:locally_geodesic} (and locally a minimizing geodesic) and which is a critical point of the energy functional~\eqref{eq:critical_point}. 

\begin{figure}[t]
\includegraphics[height=4.5cm]{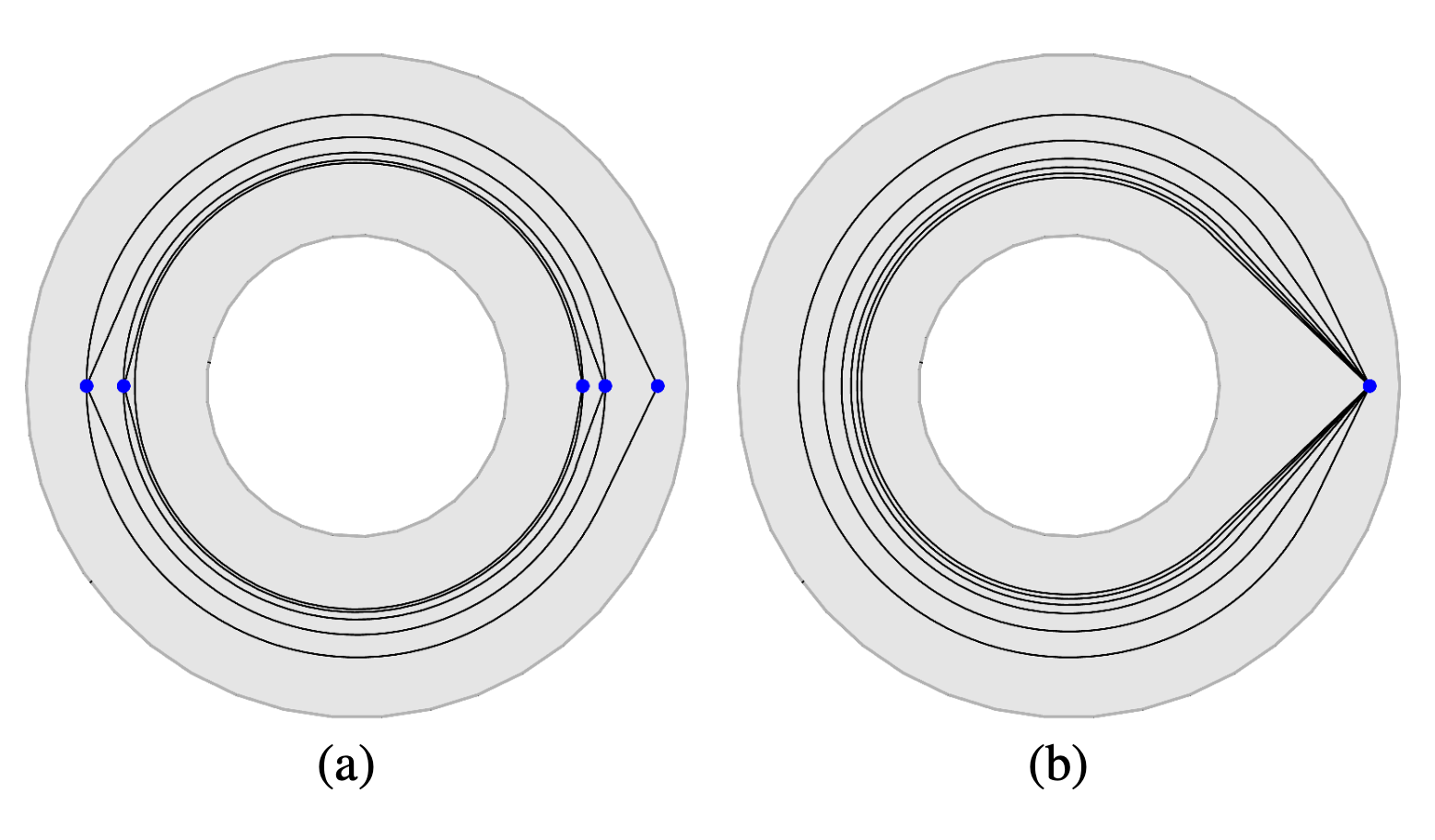}
\caption{Illustration of the importance of choosing $\fx_n$ and $\fx_{n+1}$ (indicated by blue dots) sufficiently far from each other. \textbf{a} If $\fx_{n+1}$ is opposite to $\fx_n$, then the curves $\cC_n$ converge to a smooth closed geodesic. \textbf{b} If $\fx_n=\fx_0$ for all $n\geq 0$, then a non-vanishing angle remains at this point.}
\label{fig_ChoosingSRCPoint}
\end{figure}

A peculiar aspect of \cref{th:critical}, and \cref{alg_ContourEvolution}, is the need to choose a sequence of points $(\fx_n)_{n\geq 0}$ whose successive elements are far enough, as expressed by the constraint $\|\fx_{n+1}-\fx_n \| \geq \delta > 0$. First, let us recall that the constraint $\fx_n \in \partial S$ in Eq.~\eqref{eq:successive_min} is what makes this problem tractable numerically, as discussed below \cref{th:geodesic_tubular}. Now, as illustrated on \cref{fig_ChoosingSRCPoint}, using well separated successive points $\fx_n$ and $\fx_{n+1}$ allows obtaining a smooth closed geodesic in the limit, whereas using unchanged or excessively close successive points may lead to a non-smooth limit curve.

Some elements $t_0, \cdots, t_k \in \bT$ of the periodic unit interval are said circularly ordered, which is denoted $t_0 \leq \cdots \leq t_k$, if they admit 
representatives modulo $1$, denoted $\hat t_0, \cdots, \hat t_k\in \bR$ satisfying $\hat t_0 \leq \cdots \leq \hat t_k \leq \hat t_0+1$. 
In particular, we let $[t_0,t_1] := \{t \in \bT\mid t_0 \leq t \leq t_1\}$. Finally $|t_0-t_1| := \min_{z\in \bZ} |\hat t_0-\hat t_1-z|$. 

\begin{theorem}
\label{th:critical}
Same assumptions as \cref{th:successive_min}, with additionally $\|\fx_{n+1}-\fx_n\|\geq \delta > 0$  for all $n\geq 0$. Then there exists a subsequence converging uniformly: $\mu_{\vp(n)} \to \mu_\gS$ as $n \to \infty$, where $\gS \in \cX_0$ is such that $\Lip(\mu_\gS) \leq 1$.
Furthermore one has $\mu_{\vp(n)+1} \to \mu_\gS$ uniformly as $n \to \infty$, and 
\begin{equation}
\label{eq:sci_region}
E(\gS) \leq \lim_{n \to \infty} E(S_n).	
\end{equation}
Let $t_0<t_1<t_2 \in \bT$ be three points in circular order, with $|t_0-t_1| \leq \delta/3$. Then
in the domain $\rT \sm \rW(t_2)$
\begin{equation}
\label{eq:locally_geodesic}
	\Length_\gS(\cC_{\gS|[t_0,t_1]}) 
	= \Dist_\gS (\cC_\gS(t_0), \cC_\gS(t_1)),
\end{equation}
where $\cC_\gS:=\rC+\mu_\gS \rN$ and $\cC_{\gS|[t_0,t_1]}$ denotes the restriction of this path to the interval $[t_0,t_1]$. 
Finally consider $S_\ve \in \cX_1$ such that $\cC_{S_\ve} = \cC_\gS +\ve \eta$, where $\eta \in \Lip( \bT , \bR^2)$ is fixed for all $\ve \in ]0,1[$. 
Then as $\ve \to 0$
\begin{equation}
\label{eq:critical_point}
	E(S_\ve) \geq E(\gS) + o(\ve).
\end{equation}
\end{theorem}

The rest of this subsection is devoted to the proof of \cref{th:critical}, 
which establishes the main guarantees of our segmentation method \cref{alg_ContourEvolution}, and is thus the central mathematical result of this paper.
We would like to acknowledge the constructive input of an anonymous referee, who encouraged the authors to fully develop this analysis.
Recall that $\Length_\gS := \Length_{\cF^\gS}$ and $\Dist_\gS := \Dist_{\cF^\gS}$, where the Randers metric $\cF^\gS$ is defined in \cref{th:stokes}.
\vspace{-0.5\baselineskip}\\

\noindent\emph{Properties of the sequence $\mu_n$}.
Observe that the functions $\mu_n : \bT \to [-U,U]$ are defined between compact spaces, and that $\Lip(\mu_n) \leq 1$ for all $n \geq 0$. Therefore, by the Arzelà-Ascoli compactness theorem \eqref{eq:ArzelaAscoli}, there exists a converging subsequence $\mu_{\vp(n)} \to \mu_*$ uniformly as $n \to \infty$, where $\mu_* : \bT \to [-U,U]$ satisfies likewise $\Lip(\mu_*) \leq 1$.
This defines a curve $\cC_* := \rC+ \mu_* \rN \in \Gamma_0$, which as observed below \cref{def:cX01} is the boundary of some region $\gS \in \cX_0$, as announced. Note that $\mu_\gS = \mu_*$ and $\cC_\gS = \cC_*$.

One has $\|\mu_{n+1}-\mu_n\|_2 \to 0$ as $n \to \infty$, in view of Eq.~\eqref{eq:sum_sqnorms}. Since $\Lip(\mu_n)\leq 1$ and $\Lip(\mu_{n+1}) \leq 1$, it follows that $\mu_{n+1}-\mu_n\to 0$ uniformly. Therefore $\mu_{\vp(n)+1} = (\mu_{\vp(n)+1}-\mu_{\vp(n)}) + \mu_{\vp(n)} \to \mu_\gS$ as $n \to \infty$, as announced.

The region-based energy $\Psi$ satisfies, for any shape $S\in \cX_0$
\begin{align*}
&|\Psi(\chi_\gS) -\Psi(\chi_S)|\\
&\leq \|\xi_\gS\|_\infty \|\chi_S-\chi_\gS\|_1+	K_\Psi\|\chi_S-\chi_\gS\|_1^2\\
&\leq K \|\chi_S-\chi_\gS\|_1
\leq \tfrac 4 3 K \|\mu_S-\mu_\gS\|_1.
\end{align*}
We used successively (i) the expansion \eqref{eq:taylor_phi}, (ii) the constant $K := \|\xi_\gS\|_\infty+ \Leb(\rT)K_\Psi$ where $\Leb$ stands for the Lebesgue measure, and (iii) \cref{prop:mu_chi}. It follows that $\Psi(\chi_{S_{\vp(n)}}) \to \Psi(\chi)$ as $n \to \infty$, and on the other hand $\Length_\cR(\cC_\gS) \leq \liminf \Length_\cR(\cC_{S_{\vp(n)}})$ as $n \to \infty$ by lower semi-continuity of the length functional~\eqref{eq:length_lsc}. Taking the sum of these two estimates we obtain $E(\gS) \leq \liminf E(S_{\vp(n)})$ as $n\to \infty$, which implies~\eqref{eq:sci_region} since $E(S_n)$ is converging .
\vspace{-0.5\baselineskip}\\

\noindent\emph{Exploiting the assumption $\|\fx_n-\fx_{n+1}\|\geq \delta$}.
Let $s_n \in \bT$ be such that $\fx_n = \rC(s_n) + \mu_n(s_n) \rN(s_n)$, for all $n \geq 0$. Then we have 
\begin{align*}
&\|\fx_{n+1}-\fx_n\| \\
&\hspace{-1.5mm} \leq \|\rC(s_{n+1})-\rC(s_n)\| + |\mu_n(s_{n+1})| \|\rN(s_{n+1})-\rN(s_n)\|\\
&\hspace{-1.5mm} + \big(|\mu_{n+1}(s_{n+1})-\mu_{n+1}(s_n)|\\
&\hspace{-1.5mm}+|\mu_{n+1}(s_n)-\mu_n(s_n)|\big) \| \rN(s_n)\|\\
&\hspace{-1.5mm}  \leq (2+U\kappa_{\max})|s_{n+1} - s_n| + |\mu_{n+1}(s_n)-\mu_n(s_n)|.
\end{align*}
We used, in the last estimate, the fact that $\rC$ and $\mu_{n+1}$ are $1$-Lipschitz, and that $\rN$ is $\kappa_{\max}$-Lipschitz.
Recalling that $U\kappa_{\max} \leq 1/3$ and that $\mu_{n+1}-\mu_n \to 0$ uniformly, we obtain that $|s_n - s_{n+1}|\geq \delta /(2+1/3)+o(1)$ as $n \to \infty$.
Following the theorem statement, consider $t_0<t_1<t_2$ cyclically ordered in $\bT$ and such that $|t_1-t_0| \leq \delta/3$. Then for sufficiently large $n$, one has $|s_{n+1}-s_n| > |t_0-t_1|$, hence one cannot have both $s_n$ and $s_{n+1}$ in $[t_0,t_1]$. 
Without loss of generality, we thus assume that $s_{\vp(n)} \notin [t_0,t_1]$ for all $n \geq 0$, up to changing the extraction to $\psi(n) = \vp(n_0+n)+\sigma(n)$ for a suitable $n_0\geq 0$ and $\sigma(n)\in \{0,1\}$.
\vspace{-0.5\baselineskip}\\

\noindent\emph{The circular geodesic property}.
For readability, we let $\cC_n := \cC_{S_n} = \rC+\mu_n \rN$.
For any $n \geq 0$, one has in the walled domain $\rT\sm \rW(t_2)$
\begin{equation}
\label{eq:locally_geodesic_n}
\Length_{S_n}(\cC_{n|[t_0,t_1]}) 
	\leq \Dist_{S_n} (\cC_n(t_0), \cC_n(t_1)).
\end{equation}
Indeed if this fails, by contradiction, then there exists a path $\cC \in \Lip([t_0,t_1],\rT\sm \rW(t_2))$ with endpoints $\cC(t_0) = \cC_n(t_0)$ and $\cC(t_1) = \cC_n(t_1)$, and such that $\Length_{S_n}(\cC) < \Length_{S_n}(\cC_{n|[t_0,t_1]})$. Then define, as illustrated on \cref{fig_shortcut}a
\begin{equation}
\label{eq:shortcut}
\tilde \cC(t) := 
\begin{cases}
\cC(t), &\text{if } t \in [t_0,t_1],\\
\cC_n(t), &\text{otherwise}.	
\end{cases}
\end{equation}
One has $\tilde \cC(s_{\vp(n)}) = \cC_n(s_{\vp(n)}) = \fx_n$ since $s_{\vp(n)} \notin [t_0,t_1]$, and $\Length_{S_n}(\tilde \cC) < \Length_{S_n}(\cC_n)$ by concatenation.
Note in addition that $\cC$ is homotopic to $\cC_{n|[t_0,t_1]}$, since the walled domain $\rT\sm \rW(t_2)$ is diffeomorphic to $(\bT \sm\{t_2\}) \times [-U,U]$ hence is simply connected, and therefore $\tilde \cC$ is homotopic to $\cC_n$, thus $\tilde \cC \in \Gamma_1$. 
This shows that $\cC_n$ is not a minimizer of \eqref{eq:argmin_path}, with $\cF = \cF^{S_n}$ and $\fx = \fx_n$, which contradicts Eq.~\eqref{eq:successive_min} and \cref{corol:minS_wellposed}. Thus  the inequality \eqref{eq:locally_geodesic_n} holds.

Taking the limit of \eqref{eq:locally_geodesic_n} as $n \to \infty$, and recalling that the length is lower semi-continuous \eqref{eq:length_lsc} whereas the other involved quantities are continuous, we obtain 
\begin{equation*}
\Length_\gS(\cC_{\gS|[t_0,t_1]})\leq \Dist_\gS (\cC_\gS(t_0),\cC_\gS(t_1))
\end{equation*}
in $\rT \sm \rW(t_2)$.
However the converse inequality holds as well by definition of the geodesic distance, which establishes Eq.~\eqref{eq:locally_geodesic}. By the same argument, Eq.~\eqref{eq:locally_geodesic_n} is in fact an equality. 
\vspace{-0.5\baselineskip}\\

\begin{figure*}[t]
\centering
\includegraphics[height=6cm]{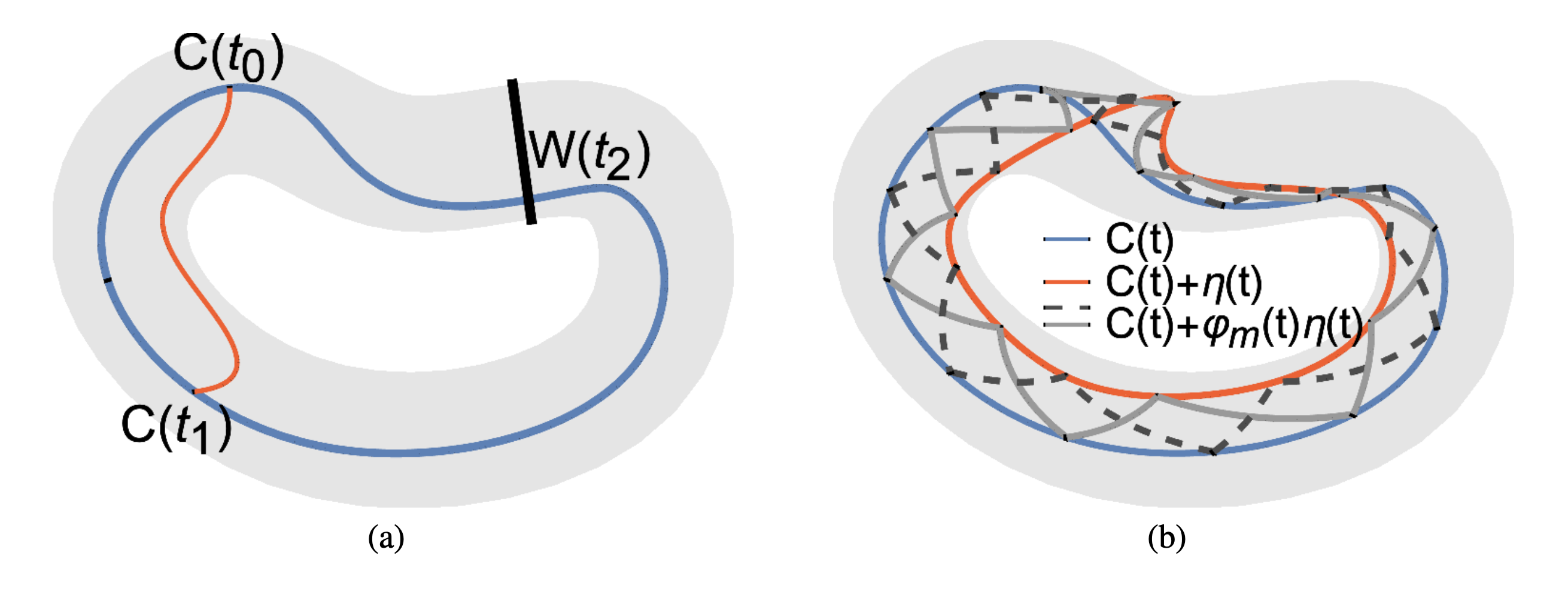}
\caption{
\textbf{a} Illustration of the curve concatenation \eqref{eq:shortcut}. \textbf{b} A global perturbation (orange) of given reference curve (blue) can be expressed as a sum of local perturbations (gray and dashed), as used in Eq.~\eqref{eq:linearity_differential}.
}
\label{fig_shortcut}
\end{figure*}

\noindent\emph{The critical point property, for a local perturbation}.
For readability, and following the notations of the theorem, denote $\cC_\ve := \cC_{S_\ve} = \cC_\gS+\ve \eta$, for all $0< \ve < 1$. 
We assume in this paragraph that $\eta(t) = 0$ for all $t \notin ]t_0,t_1[$, where $t_0 < t_1 < t_2$ and $|t_0-t_1| \leq \delta/3$.
For sufficiently small $\ve>0$, the path section $\cC_{\ve|[t_0,t_1]}$ takes its values in the walled domain $\rT \sm \rW(t_2)$. Thus, by definition of the distance function within this domain,
\begin{equation*}
	\Length_\gS(\cC_{\ve|[t_0,t_1]}) \geq \Dist_\gS(\cC_\ve(t_0),\cC_\ve(t_1)).
\end{equation*}
In view of \eqref{eq:locally_geodesic}, and noting that $\cC_\ve(t) = \cC_\gS(t)$ for all $t \notin ]t_0,t_1[$, we obtain $\Length_\gS(\cC_\ve) \geq \Length_\gS(\cC_\gS)$. This establishes \eqref{eq:critical_point} in this special case. 
\vspace{-0.5\baselineskip}\\

\noindent\emph{The critical point property, general case}.
Let $t'_0 < \cdots < t'_M = t'_0$ be circularly ordered elements of $\bT$ such that $|t'_m - t'_{m+1}| \leq \delta/6$ for all $1 \leq m \leq M$, with periodic indexing. Let also $\vp_m : \bT \to \bR$ be the piecewise affine function such that $\vp_m(t'_m)=1$ and $\vp_m(t'_{\tilde m})=0$ for all $\tilde m \neq m$, also known as the $m$-th ``hat function'', where $1 \leq m \leq M$. Note that $\vp_1+\cdots+\vp_M=1$,  that $0 \leq \vp_m\leq 1$, and that $\supp(\vp_m) = [t'_{m-1}, t'_{m+1}]$ is an interval of length at most $\delta/3$, for any $1 \leq m \leq M$.

We consider a perturbation $\eta$, as in the theorem statement, and denote $\eta_m := \vp_m \eta$ for all $1 \leq m \leq M$. Then 
\begin{equation*}
		\cC_\gS(t) + \ve \eta_m(t) = \cC_\gS(t) + (\ve \vp_m(t)) \, \eta(t) \in \rT
\end{equation*}
for all $\ve \in [0,1]$ and all $t \in \bT$. Therefore $\Length_\gS(\cC_\gS+\ve \eta_m) \geq \Length_\gS(\cC_\gS)$, as shown in the previous paragraph and since $\supp(\eta_m) \subset [t'_{m-1},t'_{m+1}]$ is sufficiently short ($|t'_{m-1} - t'_{m+1}| \leq \delta/3$). Thus 
\begin{align}
\nonumber
&\Length_\gS(\cC_\gS+\ve \eta) - \Length_\gS(\cC_\gS)\\
\nonumber
&=\int_\bT \Big[\cF^\gS\big(\cC_\gS(t)+\ve \eta(t),\, \cC'_\gS(t)+\ve \eta'(t)\big) \\
\nonumber
& \qquad - \cF^\gS\big(\cC_\gS(t),\, \cC'_\gS(t)\big) \Big]\diff t \\ 
\nonumber
&=\ve \int_\bT \big(\<a_\gS(t), \eta(t)\> + \<b_\gS(t),\eta'(t)\>\big) \diff t + o(\ve),\\
\nonumber
&=\ve \sum^{M}_{m=1} \int_\bT \big(\<a_\gS(t), \eta_m(t)\> + \<b_\gS(t),\eta_m'(t)\>\big) \diff t + o(\ve),\\
\nonumber
&=\sum^M_{m=1} \big(\Length_\gS(\cC_\gS+\ve \eta_m) - \Length_\gS(\cC_\gS)\big) + o(\ve), \\
\label{eq:linearity_differential}
&\geq o(\ve),
\end{align}
as $\ve\to 0^+$. We have used the fact that the global perturbation $\eta$ is expressed as the sum of the local perturbations $\eta_m$, as illustrated on \cref{fig_shortcut}b, and the linearity of the first order differential of $\cF^\gS$ at $\cC_\gS$. The coefficients $a_\gS,b_\gS \in C^0( \bT, \bR^2)$ can be expressed in terms of $\cC_\gS$, $\cM$, $\omega_\gS$ and of their first order derivatives.
(Because $\cC_\gS$ is a geodesic, see Eq.~\eqref{eq:locally_geodesic}, it has bounded curvature, hence is continuously differentiable - the curvature bound follows from \cref{prop:bounded_curvature} and the fact that the smooth change of variables $\Phi$ maps the convex band $\bR\times ]-U,U[$ onto the tubular domain $\rT$ where $\cC_\gS$ lies, as in the proof of \cref{th:geodesic_tubular}.
The fields $\cM$ and $\omega_\gS$ are continuously differentiable by assumption. The term $d_{\partial \gS}(\cC_\gS+\ve \eta)^2 = \cO(\ve^2)$ from Eq.~\eqref{eq:rander_stokes} is absent from the Taylor expansion since it is of second order.)

In addition, 
\begin{equation*}
	D(\cC_\gS+\ve \eta\|\cC_\gS) \leq 
	\ve^2 \|\eta\|^2_\infty \Length_\cR(\cC_\gS+\ve \eta) = \cO(\ve^2),
\end{equation*}
in view of the definition \eqref{eq:diver_SS} of $D$, and of the upper bound on the Euclidean distance $d_{\partial\gS}(\cC_\gS(t)+\ve \eta(t)) \leq |\ve \eta(t)| \leq \ve \|\eta\|_\infty$, for any $t\in \bT$ and any $\ve \in [0,1]$. 
Finally
\begin{align*}
	&E(S_\ve) -E(\gS) \geq E_\gS(S_\ve)-E_\gS(\gS)-3 \lambda D(\cC_\gS+\ve \eta\|\cC_\gS)\\
	&=\Length_\gS(\cC_\gS+\ve \eta) - \Length_\gS(\cC_\gS) + \cO(\ve^2) \geq o(\ve),
\end{align*}
using \cref{prop:relaxation_close}. This establishes \eqref{eq:critical_point} and concludes the proof of \cref{th:critical}.

\subsection{Construction of the Randers Vector Field}
\label{subsec_ExistenceVectorField}
In this section, we present the construction of the vector field $\omega_\gS$, which defines the asymmetric part of the Randers metric $\cF^\gS$ used in our geodesic segmentation method.
This is a core ingredient of the introduced region-based Randers geodesic model, which allows  reformulating the region term of the segmentation functional \eqref{eq_HybridEnergy} into a boundary term using Stokes formula, see \cref{th:stokes}. 
Precisely, this subsection is devoted to a constructive proof of the following result.
A discussion of its numerical implementation is presented in \cref{subsec:curl_implem}. 

Recall that the $\alpha$-H\"{o}lder semi-norm of $\xi : \overline \Omega \to \bR$, denoted $|\xi|_{C^\alpha}$, is defined as the smallest constant such that 
$|\xi(\fx)-\xi(\fy)| \leq |\xi|_{C^\alpha} \, \|\fx-\fy\|^\alpha$ for all $\fx,\fy \in \overline \Omega$, where $0<\alpha<1$. The set of shapes $\cX_0$ could be replaced with an arbitrary metric space in the following result, and no assumption is made on the expression of the region-based energy gradient $\xi_\gS$.
The tubular domain $\rT\subset \Omega$ of width $U$ is defined at Eq.~\eqref{eqdef:tubular_domain}.

\begin{theorem}
\label{th:exists_omega}
Assume that $\gS \in \cX_0 \mapsto \xi_\gS \in C^0(\overline \Omega,\bR)$ is continuous, and that $\|\xi_\gS\|_\infty$ and $|\xi_\gS|_{C^\alpha}$ are bounded independently of $\gS \in \cX_0$, for some $0<\alpha<1$.
Then one can define $\omega_\gS \in C^1(\rT,\bR^2)$ such that 
\begin{equation}
\label{eq:curl_omega_th}
\curl \omega_\gS = \xi_\gS \quad \text{on~~} \rT.
\end{equation}
In addition, one has $\|\omega_\gS\|_\infty \leq K U (1+|\ln U|)$ and $\|\diff \omega_\gS\|_\infty \leq K$, where the constant $K$ is independent of $U$, and $\gS \in \cX_0 \mapsto \omega_\gS \in C^0$ is continuous. 
\end{theorem}

The proof of \cref{th:exists_omega} is based on elementary potential theory, specialized to the case of a thin tubular domain. The vector field $\omega_\gS$ is eventually obtained as a convolution \eqref{eq:omega_expl}, which can be implemented either directly or through a fast Fourier transform.
As a starter, we recall the expression of the Green kernel of the two-dimensional Poisson PDE, denoted $G : \bR^2 \to \bR$, and its fundamental property:
\begin{align}
\label{eq:green_kernel}
	G(\fx) &:= \ln |\fx|, &
	\Delta G &= 2 \pi \delta_0,
\end{align}
where $\delta_0$ denotes the Dirac mass at the origin, and Eq.~(\ref{eq:green_kernel}, right) is understood in the sense of distributions~\citep{friedlander1998introduction}. 
As recalled in the next result, the convolution of a given function $\xi$ with $G$ yields a solution of the Poisson PDE, in the sense of distributions. A solution to the curl equation~\eqref{eq:curl_omega_th}, which is of interest in this paper, can be deduced. Note that alternative approaches to the curl PDE, with fewer guarantees, are discussed in \cref{subsec:curl_implem}.

\begin{proposition}	
\labelx{prop:laplace}
Let $\xi : \bR^2 \to \bR$ be bounded and compactly supported. Define
\begin{align*}
	\vp(\fx) &:= \int_{\bR^2} \xi(\fy) G(\fx-\fy)\, \diff \fy, &
	\omega(\fx) &:= \frac 1 {2 \pi} (\nabla \vp(\fx))^\perp.
\end{align*}
for all $\fx \in \bR^2$, so that $\vp : \bR^2 \to \bR$ and $\omega : \bR^2 \to \bR^2$.
Then $\Delta\vp=2\pi \xi$ and $\curl \omega = \xi$, in the sense of distributions. 
\end{proposition}

\begin{proof}
Differentiating under the integral sign, in the sense of distributions, yields $\Delta \vp(\fx) = \int_{\bR^2} \xi(\fy) \Delta G(\fx-\fy)\diff y = 2 \pi \xi(\fx)$ in view of Eq.~(\ref{eq:green_kernel}, right). On the other hand
\begin{align*}
	\curl \omega &= -\diver(\omega^\perp) = -\tfrac 1 {2 \pi} \diver((\nabla \vp)^{\perp\perp}) = \tfrac 1 {2 \pi}\diver(\nabla \vp), 
\end{align*}
hence $\curl \omega =\tfrac 1 {2 \pi}\Delta \vp = \xi$, as announced.
\qed
\end{proof}

It is tempting to apply~\cref{prop:laplace} to $\xi_\gS$, the gradient of the region-based energy functional $\Psi$, to obtain a solution $\omega_\gS$ to the curl PDE \eqref{eq:curl_omega_th} over the whole domain $\Omega$. 
However this construction fails to obey the second requirement $\|\omega_\gS\|_\infty \leq K U (1+|\ln U|)$ in general. Indeed, it is independent of the width $U$ of the thin tubular domain $\rT \subset \Omega$, which is the only place where~\cref{eq:curl_omega_th} needs to hold.

For the purposes of our construction, we introduce a twice larger tubular domain, defined similarly to~\cref{eqdef:tubular_domain} as 
\begin{equation}
\label{eq_DoubleTube}
\rT_2 := \Phi(\bT \times [-2U,2U]).
\end{equation}
In other words, $\rT_2$ is a neighborhood of width $2U$ of the reference curve $\rC$ fixed in~\cref{subsec_SearchSpace}, and in particular one has $B(\fx,U) \subset \rT_2$ for any $\fx\in \rT$.
By~\cref{prop:Phi} and since $2U \leq \frac 2 3 \lfs(\rC)$ in view of~\cref{eqdef:tubular_domain}, the mapping $\Phi\in C^2(\bT \times [-2U,2U], \rT_2)$ defines a diffeomorphism with Jacobian determinant $1/3 \leq \Jac_\Phi \leq 5/3$.
The diffeomorphism property implies in addition that
\begin{equation}
\label{eq:Phi_inverse_Lipschitz}
	\|\Phi(t,u)-\Phi(\tilde t,\tilde u)\| \geq c_\Phi(|t-\tilde t|+|u-\tilde u|)
\end{equation}
for all $(t,u),(\tilde t,\tilde u) \in  \bT \times [-2 U, 2 U]$, where $c_\Phi>0$ is a constant (independent of $U$).

In \cref{prop:grad_phi_small,prop:hessian_phi_bounded} below, we fix a function $\xi : \rT_2 \to \bR$ and constants $C_\xi$, $\tilde C_\xi$ and $0<\alpha<1$ such that
\begin{align}
\label{eq:Holder_regularity}
	| \xi(\fx) | &\leq C_\xi, &
	|\xi(\fx)-\xi(\fy)| & \leq \tilde C_\xi \|\fx-\fy\|^\alpha,
\end{align}
for all $\fx \in \rT_2$ (resp.\ $\fx \in \rT$ and $\fy \in B(\fx,U)$).
One may choose $C_\xi := \|\xi\|_\infty$ and $\tilde C_\xi := |\xi|_{C^\alpha}$.
We define 
\begin{equation}
\label{eq:potential_band}
	\vp(\fx) := \int_{\rT_2} \xi(\fy) G(\fx-\fy) \diff \fy,
\end{equation}
and note that $\Delta \vp = 2\pi \, \xi$ on $\rT$, by \cref{prop:laplace} applied to $\xi \chi_{\rT_2}$. 
The following two results estimate the magnitude of the gradient $\nabla \vp$ and of the hessian $\nabla^2\vp$ on $\rT$. 
Let us stress that the constants $K$ and $\tilde K$ are independent of the tube width $U$, subject to $0 < U \leq \lfs(\rC)/3$ see~\cref{eqdef:tubular_domain}.

\begin{proposition}
\labelx{prop:grad_phi_small}
Under the assumptions in Eqs.~(\ref{eq:Phi_inverse_Lipschitz}, left), \eqref{eq:Holder_regularity}, and \eqref{eq:potential_band}, one has 
$\| \nabla \vp (\fx)\| \leq K U(1+ |\ln U|)$ for all $\fx \in \rT$, with constant $K = K(C_\xi,c_\Phi)$.
\end{proposition}

\begin{proof}
We have by differentiation under the integral sign
\begin{align*}
	\nabla \vp(\fx) &= \int_{\rT_2} \xi(\fy) \nabla G(\fx-\fy) \diff y, &
	\text{with }\nabla G(\fz) &= \frac {\dfe_\fz}{\|\fz\|},
\end{align*}
and where $\dfe_\fz := \fz/\|\fz\|$. Therefore, introducing $t_\fx\in \bT$ and $u_\fx \in [-U,U]$ such that $\fx = \Phi(t_\fx,u_\fx)$, we obtain
\begin{align*}
	\|\nabla \vp(\fx)\|
	&\leq \int_{\rT_2} |\xi(\fy)| \, \|\nabla G(\fx-\fy)\| \,\diff \fy\\
	&\leq C_\xi \int_{\rT_2} \frac {\diff \fy} {\|\fx-\fy\|}\\
	& = C_\xi \int_\bT \int_{-2U}^{2U} 
	\frac{\Jac_\Phi(t,u) \, \diff u\, \diff t}{\|\Phi(t,u) - \Phi(t_\fx,u_\fx)\|}\\
	& \leq C_1 \int_\bT \int_{-2U}^{2U} \frac{\diff u\, \diff t}{|t-t_\fx|+|u-u_\fx|}\\
	& \leq 4 C_1 \int_0^{\frac 1 2} \int_0^{3 U} \frac{\diff u\, \diff t}{t+u} \\
	& = 2 C_1 [-6 U \ln (6U) + (6U+1) \ln (6U+1) ].
\end{align*}
We used successively (i) the triangular inequality, and (ii) the upper bound (\ref{eq:Holder_regularity}, left) and the identity $\|\nabla G(\fz)\| = 1/\|\fz\|$. Note that $\fz \in \bR^2 \mapsto 1/\|\fz\|$ is integrable over bounded regions of $\bR^2$, and thus $\nabla \vp(\fx)$ is obtained as a classical summable integral. The subsequent estimates use successively (iii) a change of variables by the diffeomorphism $\Phi$, (iv) the upper bound $\Jac_\Phi \leq 5/3$ on the Jacobian determinant and the inverse Lipschitz property \eqref{eq:Phi_inverse_Lipschitz}, with $C_1 = \frac 5 3 C_\xi/c_\Phi$, (v) the fact that $\bT = \bR/\bZ$ is the unit segment with periodic boundary conditions, and that $u_\fx \in [-U,U]$, and (vi) an exact integral computation. The result follows.
\qed
\end{proof}

\begin{proposition}
\labelx{prop:hessian_phi_bounded}
Under the assumptions in~\cref{eq:Phi_inverse_Lipschitz,eq:Holder_regularity,eq:potential_band}, one has 
$\| \nabla^2 \vp(\fx)\| \leq \tilde K$ for all $\fx \in \rT$, with constant $\tilde K = \tilde K(C_\xi,\tilde C_\xi, c_\Phi,\alpha)$.
\end{proposition}

\begin{proof}
A classical argument from potential theory \citep{friedlander1998introduction} yields
\begin{equation*}
\nabla^2 \vp(\fx) = \int_{\rT_2} \big(\xi(\fy) - \chi_{B(\fx,U)}(\fy)\xi(\fx) \big) \nabla^2 G(\fx-\fy) \diff \fy,
\end{equation*}
where denoting $\dfe_\fz := \fz/\|\fz\|$ one has 
\begin{align*}
\nabla^2 G(\fz) &= \frac{\Id - 2 \dfe_\fz \dfe_\fz^\top }{\|\fz\|^2}, &
\| \nabla^2 G(\fz)\| &= \frac 1 {\|\fz\|^2}.
\end{align*}
By definition of the characteristic function, $\chi_{B(\fx,U)}(\fy) = 1$ if $\|\fx-\fy\| < U$, and $\chi_{B(\fx,U)}(\fy) = 0$ otherwise.
As a side note, the ball $B(\fx,U)$ appearing in the integral expression of $\nabla^2\vp$ could be replaced with any ball $B(\fx,r) \subset \rT_2$ of positive radius, since for any radii $0<r<R$ 
\begin{align*}
	&\int_{r<|\fz|<R} \nabla^2 G(\fz)\, \diff \fz \\
	&= 
	\int_r^R \frac {\diff \rho} \rho
	\int_0^{2 \pi}
	\begin{pmatrix}
		1-2\cos^2\theta & -2\cos \theta \sin \theta\\
		-2\cos \theta \sin \theta & 1-2\sin^2 \theta
	\end{pmatrix} \diff \theta
	=0.
\end{align*}
The boundedness and H\"{o}lder semi-norm regularity \eqref{eq:Holder_regularity} of $\xi$ yields
\begin{align*}
	&\|\nabla^2 \vp(\fx)\| \\
	&\leq \int_{\rT_2} |\xi(\fy) - \chi_{B(\fx,U)}(\fy) \xi(\fx) | 
	\|\nabla^2 G(\fx-\fy)\| \diff \fy\\	
	& \leq  
	\tilde C_\xi\int_{B(\fx,U)} \frac{ \diff \fy}{\|\fx-\fy\|^{2-\alpha}}
	+ C_\xi \int_{\rT_2\sm B(\fx,U)} \frac { \diff \fy}{\|\fy-\fx\|^2}.
\end{align*}
The first integral term equals $2 \pi U^\alpha/\alpha$ which is bounded by $2\pi (\lfs(\rC)/3)^\alpha/\alpha$. The second integral is likewise bounded independently of $U$. Indeed, introducing $t_\fx\in \bT$ and $u_\fx \in [-U,U]$ such that $\fx = \Phi(t_\fx,u_\fx)$, we obtain
\begin{align*}
	&\int_{\rT_2} \frac { \diff \fy}{\max\{U,\|\fy-\fx\|\}^2}\\
	&=\int_\bT \int_{-2U}^{2U} 
	\frac{\Jac_\Phi(t,u) \, \diff u\, \diff t}{\max \{U,\|\Phi(t,u) - \Phi(t_\fx,u_\fx)\|\}^2}\\
	&\leq \tfrac 5 3 \int_\bT \int_{-2U}^{2U} \frac{\diff u\, \diff t}{\max \{U,c_\Phi |t-t_\fx|\}^2}\\
	& \leq \tfrac{40} 3 U \int_0^{\frac 1 2} \frac{\diff t}{\max \{U, c_\Phi t\}^2}\\
	&=\tfrac{40} 3 (2/c_\Phi-2U/c_\Phi^2) \leq 80/(3 c_\Phi).
\end{align*}
We used successively (i) a change of variables by $\Phi$, (ii) the inverse Lipschitz estimate Eq.~\eqref{eq:Phi_inverse_Lipschitz} and the bound $\Jac_\Phi(l,u) \leq 5/3$ on the Jacobian determinant, (iii) the fact that the integrand is independent of $U$, and that $\bT$ is the periodic unit segment, (iv) exact integration assuming w.l.o.g.\ $U<c_\Phi/2$. The announced result follows, and note that $\nabla^2\vp$ is obtained as a classical summable integral. 
\qed
\end{proof}

\noindent\emph{Conclusion of the proof of \cref{th:exists_omega}.}
Define $\vp_\gS : \rT \to \bR$ by Eq.~\eqref{eq:potential_band}, and $\omega_\gS := \frac 1 {2\pi} (\nabla \vp_\gS)^\perp$. 
Then $\curl \omega_\gS = \xi_\gS$ by \cref{prop:laplace}, $2 \pi \|\omega_\gS\|_\infty = \|\nabla \vp_\gS\|_\infty \leq K U (1+|\ln U|)$ by \cref{prop:grad_phi_small}, and $2 \pi\| \diff \omega_\gS \|_\infty = \|\nabla ^2 \vp_\gS\| \leq \tilde K$ by \cref{prop:hessian_phi_bounded}, as announced and where the constants $K$ and $\tilde K$ are independent of $U$.

In addition, the proofs of \cref{prop:grad_phi_small} and \cref{prop:hessian_phi_bounded} reveal that $\nabla \vp$ and $\nabla^2 \vp$ are obtained as summable integrals. From this point, a standard application of Lebesgue's dominated convergence theorem yields that $\nabla \vp$ and $\nabla^2 \vp$ depends continuously on $\xi$ and $\fx$. As a result, $\omega_\gS$ depends continuously on $\gS$ through $\xi_\gS$, and $\diff \omega_\gS$ depends continuously on $\fx$, which concludes the proof of \cref{th:exists_omega}.

\begin{figure*}[t]
\centering
\includegraphics[height=6.6cm]{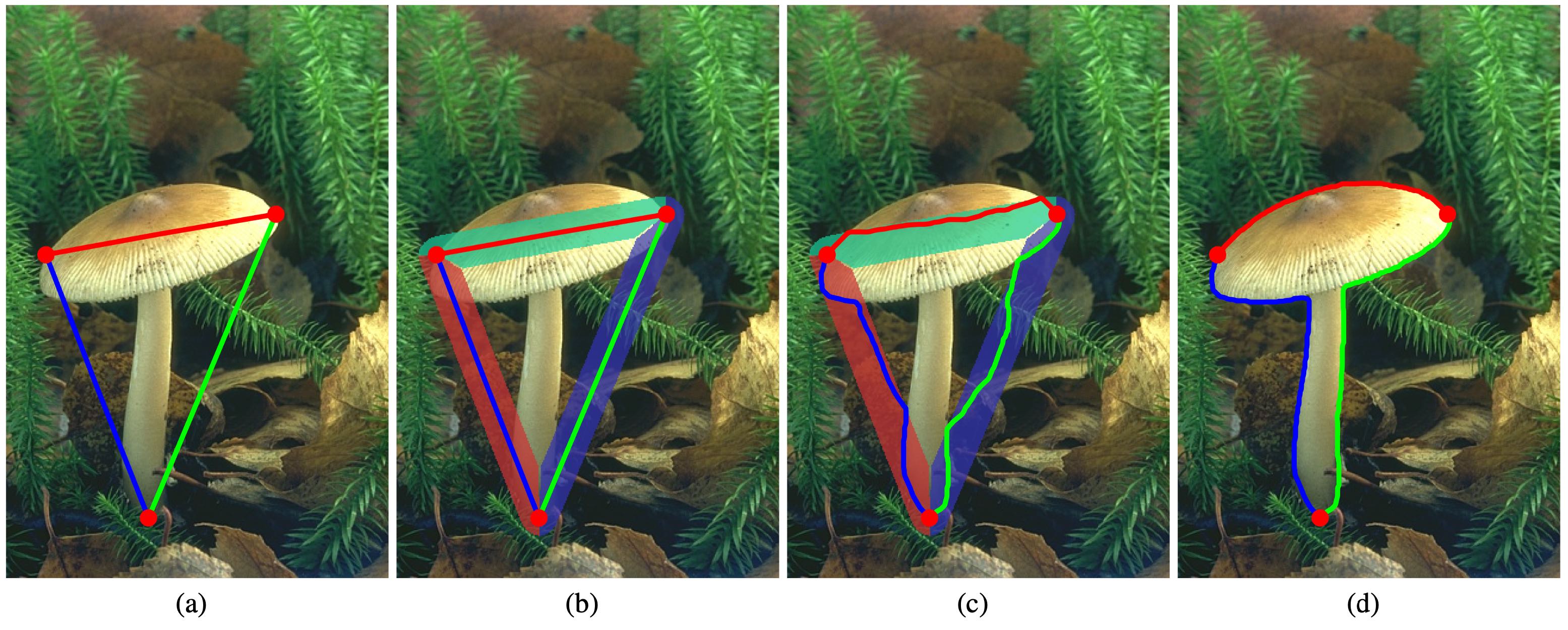}
\caption{An example for extracting piecewise geodesic paths using three landmark points. \textbf{a} The red dots represent the landmark points and the solid lines form the initial contour. \textbf{b} The decomposition of the tubular neighborhood, where the decomposed subregions are denoted by different colors. \textbf{c} Extracted geodesic paths within the decomposed subregions, where the initial contour is shown as the solid lines in figure (b). \textbf{d} Final segmentation result obtained by running a few contour evolution iterations.}
\label{fig:FixedPoints}	
\end{figure*}

\subsection{Summary} 
\label{subsec_Summary}

In this section, we have established a new solution to the active contours problem based on the Randers geodesic model. The proposed region-based Randers geodesic segmentation model, as stated in Line~\ref{line:closed_geodesic} of Algorithm~\ref{alg_ContourEvolution}, relies on a series of geodesics going from a point $\fx$ to itself and which travels along the tubular domain $\rT$.
In the practical implementation, a wall passing through the point $\fx$ is introduced and serves as an obstacle in the domain for the front propagation of the fast marching method, see Eq.~\eqref{eqdef:wall}. The computed geodesic path joins two slight perturbations of the point $\fx$ on each side of the wall, similarly to~\citep{appleton2005globally,chen2021elastica}.

While the tubular domain is fixed for the theoretical analysis of \cref{alg_ContourEvolution}, it is in contrast usually updated at each iteration of the practical implementation.
During the contour evolution, one often chooses the parametrization of the boundary $\partial{S_{n}}$ as the centerline $\rC$ to build the tubular domain $\rT:=\rT_{n}$
\begin{equation}
\label{eq_TubeDomainImp}
\rT_{n}=\left\{\fx\in\Omega \mid d_{\partial{S}_n}(\fx)<U\right\},	
\end{equation}
where $d_{\partial{S}_n}(\fx)$ denotes the Euclidean distance between $\fx$ and the shape boundary $\partial{S}_n$. 

\Cref{alg_ContourEvolution} also requires choosing a point $\fx_n \in \partial S_n$ far enough from the previous point $\fx_{n-1} \in S_n \cap S_{n-1}$, for all $n \geq 1$. In the practical implementation, we choose $\fx_n$ in such way that the two connected components of $\partial S_n \sm \{\fx_n, \fx_{n-1}\}$ have the same Euclidean length. This choice, as illustrated in~\cref{fig_ChoosingSRCPoint}a, fits the assumptions of \cref{th:successive_min}. The point $\fx_n$ is used as the source point for the computation of a distance map, with respect to the Randers metric based on the gradient of the region-based energy functional and the solution to a curl PDE. The next shape contour $\partial S_{n+1}$ is then obtained by geodesic backtracking. By iterating the main loop of \cref{alg_ContourEvolution}, a subsequence $S_{\vp(n)}$ of the constructed shapes converges to a critical point of the energy functional~\eqref{eq_HybridEnergy}, as established in~\cref{th:critical}. In practice, the shapes appear to stabilize quickly and the obtained contour matches the target boundary.

There are several additional practical modifications that can enhance the performance of the proposed algorithm in complicated scenarios.
A first avenue, as described in \cref{sec_SegmentationAlgs}, is to take into account additional user intervention, which turns \cref{alg_ContourEvolution} into an interactive image segmentation method. This variant involves a set of ordered landmark points lying on the boundary of interest, so that the evolving contour is a piecewise geodesic path, rather than a single geodesic path traveling from a point to itself. A second discussion, presented in~\cref{sec_PracticalRanders}, is devoted to the construction of the Randers vector field $\omega_\gS$ and of the tensor field $\cM$ defining the Randers metric, and to several techniques meant to increase the width $U$ of the tubular search space $\rT$.

\section{Interactive Image Segmentation}
\label{sec_SegmentationAlgs}

In this section, we discuss practical enhancements to the construction of the evolving contours $\cC_n:=\cC_{S_n}$, $n\geq 0$, used in our segmentation method.
A first fundamental ingredient, which adds some flexibility to our approach, is to regard each contour $\cC_n$ as an end-to-end concatenation of finite set of open geodesic paths, rather than as a single geodesic path from a point to itself, see \cref{subsec_FixedPointsScheme}.
Recall that path concatenation is defined as 
\begin{equation}
\label{eq_ConcaOperator}
(\gamma_1\doublecup\gamma_2)(t)=
\begin{cases}
\gamma_1(2t),    &\text{if}~0\leq t\leq \frac{1}{2},\\
\gamma_2(2t-1),	 &\text{if}~\frac{1}{2}\leq t\leq 1,
\end{cases}	
\end{equation}
where $\gamma_1,\,\gamma_2\in\Lip([0,1],\overline\Omega)$ are open curves subject to $\gamma_1(1)=\gamma_2(0)$. 
A second ingredient, described in \cref{subsec_CombPaths,subsec_Polygon}, is a careful choice of the segmentation initialization, in other words of the curve $\cC_0$. 

\subsection{Extracting Piecewise Geodesic Paths}
\label{subsec_FixedPointsScheme}
We introduce a scheme for the extraction of piecewise geodesic paths based on a set of landmark points $\fp_1,\cdots,\fp_m \in \Omega$, which are distributed along the target boundary in a counterclockwise order. In our model, these landmark points are taken as the endpoints of the geodesic paths and are not allowed to move in the course of the contour evolution. 

The proposed segmentation model involves two main steps: (i) the decomposition of the tubular neighborhood, and (ii) the computation of Randers geodesic paths using a set of landmark points, as depicted in Fig.~\ref{fig:FixedPoints}. In Fig.~\ref{fig:FixedPoints}a, we illustrate an example of $3$ landmark points $\{\fp_k\}_{1\leq k \leq 3}$  which are visualized as red dots. The solid lines indicated by different colors form a polygon, which serves as the initial contour $\partial{S}_0$. We leave the description of the construction methods for the initial contour to Section~\ref{subsec_Polygon}.

\subsubsection{Tubular Neighborhood Decomposition}
In~\Cref{alg_ContourEvolution}, the input contour for each iteration is $\cC_{n}$, which is the parametrization of the shape boundary $\partial{S}_n$. As above, the initial contour is denoted by $\cC_0$. Then one can obtain a tubular domain $\rT_n$ whose centerline is $\cC_{n}$, as defined in Eq.~\eqref{eq_TubeDomainImp}. We assume that $\cC_{n}$ is the concatenation of $m$ paths $\cG_{k,n}\in\Lip([0,1],\rT_n)$, $1\leq k \leq m$, i.e.\ 
\begin{equation}
\label{eq_ContourDecomp}
\cC_{n}=\cG_{1,n}\doublecup\cG_{2,n}\doublecup\cdots\doublecup\cG_{m,n},
\end{equation}
where the operator $\doublecup$ is defined in Eq.~\eqref{eq_ConcaOperator}. 
Each path $\cG_{k,n}$ travels from the landmark point $\fp_k$ to the successive point\footnote{In the rest of this paper, we regard $\fp_1$ as the successive point of the landmark point $\fp_m$, i.e.\ $\fp_{m+1} := \fp_1$ by convention.} $\fp_{k+1}$. More precisely, one has
\begin{equation}
\label{eq_EndtoEnd}
\cG_{k,n}(0)=\fp_{k},\text{~~and~~}\cG_{k,n}(1)=\fp_{k+1}.
\end{equation}

With these definitions, we decompose the tubular domain $\rT_{n}$ into a family of subregions $\{Z_{k,n}\}_{1\leq k\leq m}$ such that
\begin{align*}
&Z_{k,n}=\tilde{Z}_{k,n}\cup\{\fp_{k+1},\fp_k\},\quad\forall k<m,\\
&Z_{m,n}=\tilde{Z}_{m,n}\cup\{\fp_{1},\fp_m\}.
\end{align*}
The subregion $\tilde{Z}_{k,n}$ is defined for any index $1\leq k\leq m$ by
\begin{equation}
\label{eq_SubRegions}
\tilde{Z}_{k,n}:=\big\{\fx\in \rT_{n}\mid \forall l\neq k,~\tilde{d}(\fx,\cG_{k,n})<\tilde{d}(\fx,\cG_{l,n})\big\},
\end{equation}
where $d(\fx,\gamma)$ stands for the Euclidean distance between  $\fx$ and a given path $\gamma\in \Lip([0,1],\overline\Omega)$
\begin{equation*}
\tilde{d}(\fx,\gamma):=\min_{0<t<1}\|\fx-\gamma(t)\|.	
\end{equation*}
Each subdomain $\tilde{Z}_{k,n}$ can be treated as the tubular neighborhood of the open path $\cG_{k,n}$. In  Figs.~\ref{fig:FixedPoints}c and~\ref{fig:FixedPoints}d, the transparent regions of different colors illustrate the sets $\tilde{Z}_{k,n}$, $1 \leq k \leq m$.  In other words, each path $\cG_{k,n}$ serves as the \emph{centerline} of the tubular neighborhood $\tilde{Z}_{k,n}$.

\subsubsection{Extraction of Piecewise Randers Geodesic Paths}
In each subregion $Z_{k,n}\subset\rT_{n}$, a geodesic path $\cG_{k,j}$ from $\fp_k$ to $\fp_{k+1}$ can be extracted by solving the minimization problem  
\begin{align}
\label{eq_LengthAgain}
\min\big\{\Length_{S_n}(\gamma) & \mid \gamma\in\Lip([0,1],Z_{k,n}),\nonumber\\
&\gamma(0)=\fp_{k},\,\gamma(1)=\fp_{k+1}\big\},
\end{align}
where $\Length_{S_n}(\gamma)$ is the energy of the curve $\gamma$ associated to the Randers geodesic metric $\cF^{S_n}$, see Eq.~\eqref{eq:rander_stokes}. As discussed in Section~\ref{sec_Background}, this is achieved by computing a geodesic distance map $\rD_{k,n}:Z_{k,n}\to[0,\infty[$, which is the viscosity solution to the Randers eikonal PDE~\eqref{eq_FinslerEikonal}.

Then, the extraction of the geodesic paths $\cG_{k,n+1}$ for $1\leq k\leq m$ is implemented by performing the gradient descent on the corresponding geodesic distance maps $\rD_{k,n}$, as described in Eq.~\eqref{eq:backtracking_ODE}. We illustrate the extracted geodesic paths $\cG_{k,n}$ for $1\leq k \leq 3$ and $n=1$ in Fig.~\ref{fig:FixedPoints}c. Eventually, the desired contour $\cC_{n+1}$ is generated as the concatenation of the geodesic paths $\cG_{k,n+1}$ as follows:
\begin{equation}
\label{eq_GeoConcatenation}
\cC_{n+1}=\cG_{1,n+1}\doublecup\cG_{2,n+1}\doublecup\cdots\doublecup\cG_{m,n+1}.
\end{equation}
Each geodesic path $\cG_{k,n+1}$ lies within the subregion $Z_{k,n}$, and one has $\Length_{S_{n}}(\cC_{n+1})\leq \Length_{S_{n}}(\cC_{n})$.
Note that the intersection between two adjacent subregions $Z_{k,n}$ and $Z_{k+1,n}$ is the point $\fp_k$, which guarantees that the contour $\cC_{n+1}$ is simple.

In Section~\ref{subsec_Polygon}, we introduce two methods, named \emph{polygon construction method} and \emph{simple closed contour construction method}, for building the initial contour used for extracting piecewise geodesic paths. These methods are taken as the variants of the combination of piecewise geodesic paths model~\citep{mille2015combination}, see Section~\ref{subsec_CombPaths}. 

\subsection{Overview of the Combination of Piecewise Geodesic Paths Model}
\label{subsec_CombPaths}
\citet{mille2015combination} introduced a combination of piecewise geodesic paths model for interactive image segmentation, which also involves a collection of landmark points $\{\fp_k\}_{1\leq k\leq m}$. This technique can be used in the construction of the initial guess $\cC_0$ of our method. The goal of the combination of piecewise geodesic paths model is to search for a simple and closed contour by minimizing the following energy:
\begin{equation}
\label{eq_CombEnergy}
E_{\rm comb}(\cC):=E_{\rm simp}(\cC)+z_{\rm edge}\,E_{\rm edge}(\cC)+z_{\rm r}\,\Psi(\chi_{S}),
\end{equation}
where $\cC$ belongs to a set of candidate curves, and $S$ denotes the region enclosed by $\cC$. The parameters $z_{\rm edge},\,z_{\rm r}\in[0,\infty[$ determine the relative importance between different energy terms. The simplicity term $E_{\rm simp}(\cC)$ measures the amount of the self-tangency and self-intersection of $\cC$, in such way that a simple curve $\cC$ yields a low value of $E_{\rm simp}(\cC)$. Specifically, the self-tangency in  $E_{\rm simp}(\cC)$ is quantified through an implicit measurement of the length of overlapped curve segments of $\cC$, implemented by considering  the zero-level set of a scalar-valued function $\phi(u,v):=\|\cC(u)-\cC(v)\|^2,\,\forall u,v\in[0,1]$. Moreover, the quantification of curve self-intersection is carried out by detecting inverted loops and calculating signed area values of these loops.

%

\begin{figure*}[t]
\centering
\includegraphics[height=6cm]{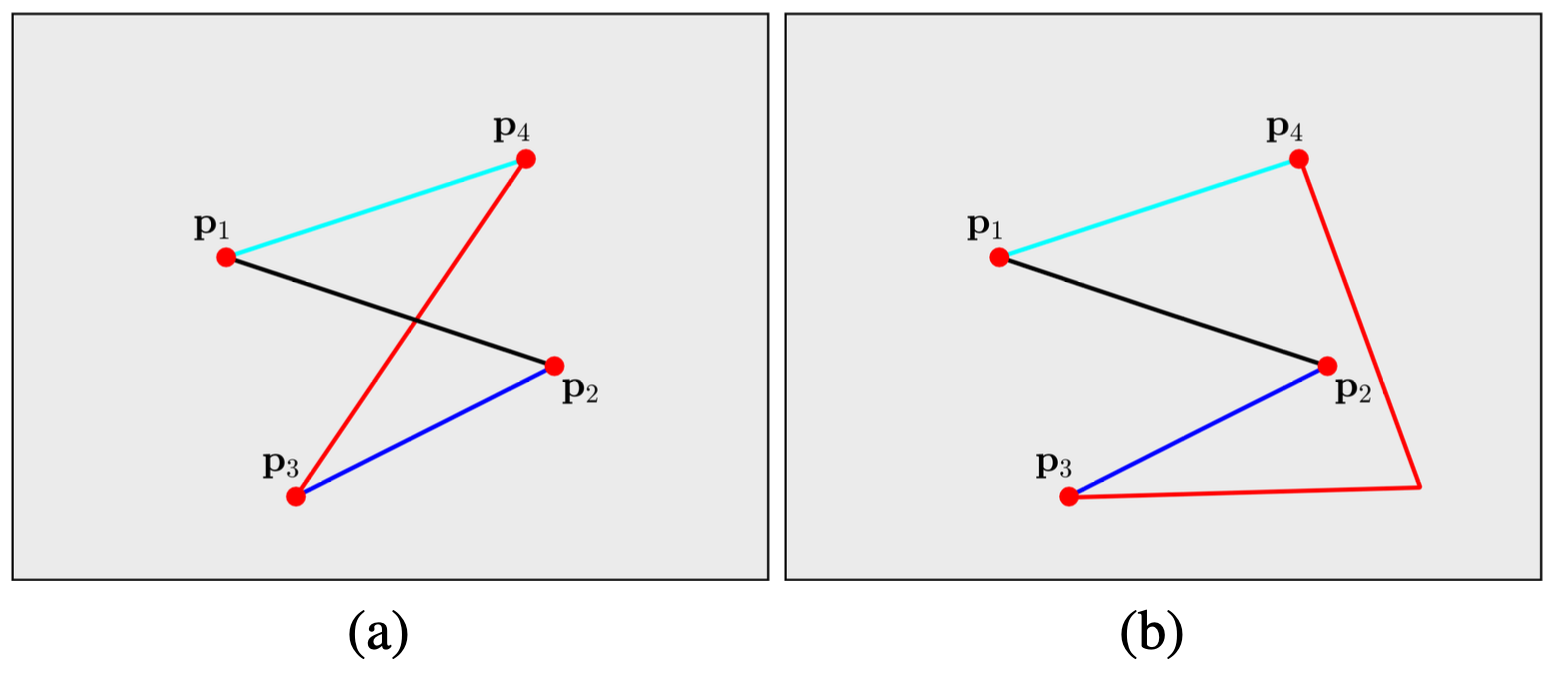}
\caption{An example for building the initial contour as a polygon. The landmark points are indicated by red dots. \textbf{a} Initial contour encountered with a self-intersection issue. \textbf{b} The initial contour that is composed of four polylines (indicated by lines with different colors) obtained through the introduced polygon construction method.}
\label{fig:Selfcrossing}
\end{figure*}

The energy term $E_{\rm edge}$ is the normalized edge-based energy along the curve  $\cC$, defined as
\begin{equation}
E_{\rm edge}(\cC):=\frac{1}{\kL(\cC)}\int_0^1\cP_{\rm comb}(\cC(t))\|\cC^\prime(t)\|dt
\end{equation}
where $\kL(\cC)$ is the Euclidean length of $\cC$, and where $\cP_{\rm comb}:\overline\Omega\to\bR^+$ is a potential that takes low values around the image edges. Let $g:\Omega\to[0,\infty[$ be a scalar-valued function carrying the image edge-based features, see Appendix~\ref{Appendix_EdgeFeatures}. As considered in~\citep{mille2015combination}, the potential $\cP_{\rm comb}$ can be computed as follows
\begin{equation}
\label{eq_PotentialComb}
\cP_{\rm comb}(\fx)=\epsilon_{\rm comb}+\max(0,1-z_{\rm comb}\,g(\fx)),
\end{equation}
where $\epsilon_{\rm comb}\in\bR^+$ is a scalar parameter dominating the regularization of geodesic paths, and the weight $z_{\rm comb}\in\bR^+$ controls the importance of the  edge-based features.

The region-based appearance term $\Psi$ is set as the Bhattacharyya coefficient~\citep{michailovich2007image} of the histograms of the image colors inside and outside $\cC$, respectively denoted by $\mathfrak{H}_{\rm in}$ and $\mathfrak{H}_{\rm out}$. The Bhattacharyya coefficient can be expressed as 
\begin{equation}
\label{eq_BhaCoeff}
\Psi(\chi_\aS)	=\int_{\mathbf{Q}}\sqrt{\mathfrak{H}_{\rm in}(\mathbf{q},\chi_\aS)\,\mathfrak{H}_{\rm out}(\mathbf{q},\chi_\aS)}\,d\mathbf{q},
\end{equation}
where $\mathbf{Q}$ stands for the image feature space. The histograms $\mathfrak{H}_{\rm in},\,\mathfrak{H}_{\rm out}$ are estimated using a Gaussian kernel $\cK_\sigma$ of standard deviation $\sigma$ as follows
\begin{align}	
\label{eq_PDFIn}
&\mathfrak{H}_{\rm in}(\fq,\chi_\aS)=\frac{1}{\Leb(\aS)}\int_{\aS} \cK_{\sigma}(\fq-\fI(\fx))d\fx,\\
\label{eq_PDFOut}
&\mathfrak{H}_{\rm out}(\fq,\chi_\aS)=\frac{1}{\Leb(\Omega\backslash\aS)}\int_{\Omega\backslash\aS}\cK_{\sigma}(\fq-\fI(\fx))d\fx,
\end{align}
where $\fI:\overline\Omega\to \bR^3$ is a vector-valued image.

\begin{figure*}[!t]
\centering
\includegraphics[height=4.8cm]{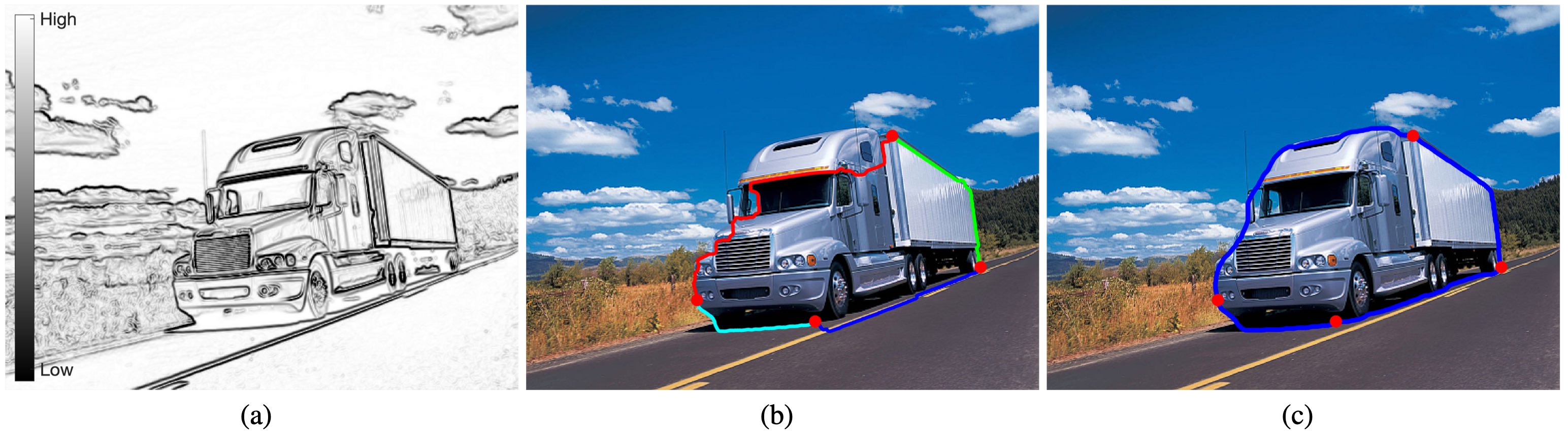}
\caption{An example for building the initial contour as a simple closed contour. \textbf{a} Visualization for the potential $\cP_{\rm comb}$. \textbf{b} The initial contour comprised of four paths (indicated by lines with different colors) is generated through  the simple closed curve construction method. \text{c} The segmentation contour derived from the fixed landmark points-based scheme.}
\label{fig:EdgeSimp}
\end{figure*}

A key step for minimizing the energy~\eqref{eq_CombEnergy} lies in the choice and the construction of a set of candidate contours. Each candidate contour is obtained as the concatenation of $m$ open paths, joining a landmark point to the successive one. Specifically, by advancing the fronts propagation simultaneously from a pair of landmark point $\fp_{k}$ and $\fp_{k+1}$, one can estimate a combined geodesic distance map using the partial fast marching  scheme~\citep{deschamps2001fast}, associated to the isotropic metric
\begin{equation}
\label{eq_CombMetric}
\cR_{\rm comb}(\fx,\dfx)=\cP_{\rm comb}(\fx)\|\dfx\|.	
\end{equation}

The meeting interface of the two propagating fronts associated to the source points $\fp_{k}$ and $\fp_{k+1}$ is  made up of  all equidistant points to $\fp_{k}$ and $\fp_{k+1}$. From this interface, several saddle points $\mathbf{m}_{k,i}$ for $1\leq i\leq q_k$ can be detected, where $q_k$ is a positive integer representing the maximal number of saddle points between $\fp_{k}$ and $\fp_{k+1}$. Each saddle point $\mathbf{m}_{k,i}$ leads to two paths $\rG^+_{k,i}$ and $\rG^-_{k,i}$ defined over the range $[0,1]$, which respectively joins from a saddle point $\mathbf{m}_{k,i}$ to the landmark points $\fp_{k+1}$ and $\fp_{k}$. The generation of paths $\rG^+_{k,i}$ and $\rG^-_{k,i}$ is implemented by performing gradient descents on the combined geodesic distance map. Then, one can obtain a new path linking from  $\fp_{k}$ to $\fp_{k+1}$
\begin{equation}
\label{eq_CombGeos}
\rG_{k,i}=\tilde\rG^-_{k,i}\doublecup\rG^+_{k,i}
\end{equation}
where $\tilde\rG^-_{k,i}$ is the reverse path of $\rG^-_{k,i}$.
With these definitions, \citet{mille2015combination} proposed to establish an admissible path $\tilde\cC$ as 
\begin{equation}
\label{eq_CombCandidate}
\tilde\cC=\rG_{1,i_1}\doublecup \rG_{2,i_2}\cdots	\doublecup\rG_{m,i_m}
\end{equation}
where $1\leq i_k\leq q_k$ is an index. Each $\tilde\cC$ is taken as a candidate minimizer for the energy $E_{\rm comb}$. We refer to~\citep{mille2015combination} for more details on the combination of piecewise geodesic paths model.

\subsection{Initial Contour Construction via Variants of the Combination of Piecewise Geodesic Paths Model}
\label{subsec_Polygon}
Simply connecting each pair of successive landmark points via a straight segment may potentially  yield a curve with self-intersections or self-tangencies in some cases, as depicted in Fig.~\ref{fig:Selfcrossing}a. This issue can be solved  by manually replacing some straight segments by polylines to remove the unexpected curve self-crossings. However, this will lead to demanding requirement to user. For the purpose of minimally interactive segmentation, we exploit two variants of the combination of piecewise geodesic paths model to build the initial contour using a set of landmark points, see~\cref{subsec_FixedPointsScheme}. In the combination of piecewise geodesic paths model, the optimal contour~\eqref{eq_CombCandidate} that minimizes the energy $E_{\rm comb}$ is expected to be simple, agreeing with the requirement on the initial contour used in our model. 
 
\subsubsection{Polygon Construction Method}
The first method is to regard the initial contour as a polygon, providing that its vertices are the given landmark points $\fp_k$. In this case, we impose that the paths $\rG_{k,i}$ (see Eq.~\eqref{eq_CombGeos}) are polylines with vertices of $\fp_k$ and $\fp_{k+1}$, done by setting that the potential $\cP_{\rm comb}\equiv1$. In this case, we replace the data-driven terms $E_{\rm edge}$ and $\Psi$ by the Euclidean curve length $\kL(\cC)$, leading to a new functional
\begin{equation}
\label{eq_PolyEnergy}
E_{\rm polygon}(\cC):=E_{\rm simp}(\cC)+z_{\rm euclid}\kL(\cC),
\end{equation}
where $z_{\rm euclid}\in\bR^+$ is a constant. A small value of $\alpha_{\rm euclid}$ is able to enhance the importance of the curve simplicity term $E_{\rm simp}$ to  reduce the risk of curve self-intersection issue. The use of the Euclidean curve length $\kL(\cC)$ as a penalty encourages a polygon of small perimeter.
Note that we apply the identical strategy as the combination of piecewise geodesic paths model~\citep{mille2015combination} to search for the optimal polygon that minimizes the energy~\eqref{eq_PolyEnergy}. As a result,  each pair of  successive landmark points is allowed to be connected by either a polyline or a straight segment. This can be observed from Fig.~\ref{fig:Selfcrossing}b, in which the curve connecting the points $\fp_2$ and $\fp_3$ is a polyline, instead of the straight segment shown in Fig.~\ref{fig:Selfcrossing}a.

\begin{figure*}[t]
\centering
\includegraphics[height=6cm]{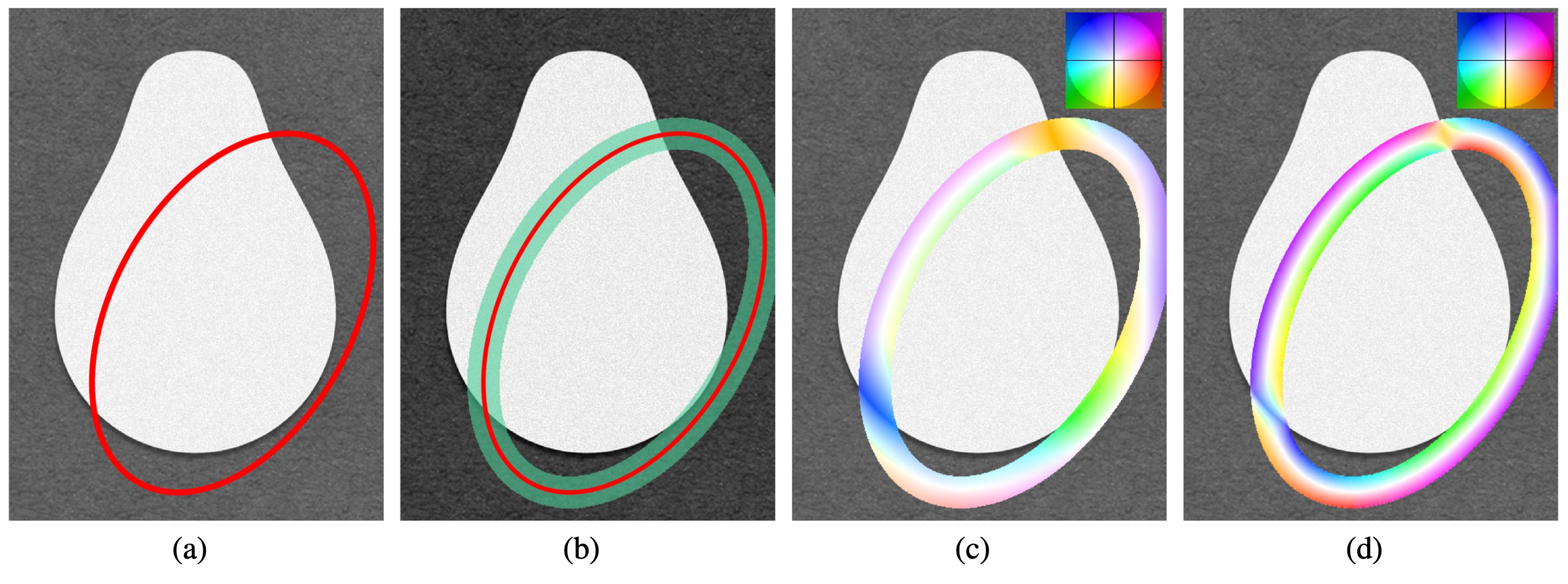}
\caption{An example for the visualization of the vector field $\omega_{\gS}$.  \textbf{a} The original image. The red line indicates the boundary of the shape $\gS$. \textbf{b} The tubular domain centered at $\partial\gS$. \textbf{c} and \textbf{d} Color coding for the visualization of the vector field $\rV_{\gS}$, respectively obtained by the convolution method and the variation method.}
\label{fig_VFDiv}
\end{figure*}

\subsubsection{Simple Closed Contour Construction Method}
In the original combination of piecewise geodesic paths model, the detection of optimal contours partially relies on the region-based appearance model $\Psi$. However, in our model, the information from the region-based appearance term $\Psi$ is embedded in the Randers geodesic metrics. Thus, it is not necessary to take into account the region-based features to build the initial contours for our model. In contrast, as the second method, the initial contour is chosen by minimizing the energy that is composed of only the simplicity measurement $E_{\rm simp}$ and the edge-based term $E_{\rm edge}$, i.e.
\begin{equation*}
E_{\rm SC}(\cC):=	E_{\rm simp}(\cC)+z_{\rm edge}E_{\rm edge}(\cC).
\end{equation*}
The removal of the region-based appearance term $\Psi$ may also reduce the computation complexity in the initialization stage of our model. We believe that using the terms $E_{\rm simp}$ and $E_{\rm edge}$ are sufficient to generate a suitable initial contour, at least when the target boundaries can be roughly defined by the image gradients. As an example, in Fig.~\ref{fig:EdgeSimp}b we illustrate a simple closed curve obtained by minimizing the energy $E_{\rm SC}(\cC)$. The potential $\cP_{\rm comb}$ which is estimated using the image gradients is shown in Fig.~\ref{fig:EdgeSimp}a. 

\section{Implementation Consideration}
\label{sec_PracticalRanders}

\subsection{Computation of the Region-based Energy Gradients}
\label{subsec:GradientExamples}

We present in this subsection the computation of the region-based energy gradient of two widely considered region-based appearance models and of the balloon model. For that purpose, we begin with a more generic case of a region-based energy that is defined as a smooth functional of integrals.

\begin{proposition}
\labelx{prop:region_gradient}
Let $\Omega \subset \bR^d$ be a bounded domain, 
let $\xi_1,\cdots, \xi_I \in L^\infty(\Omega)$ and let $\boldsymbol \Psi : \bR^I \to \bR$ be differentiable and such that $\diff \bPsi$ is Lipschitz.
Define for any shape $S \subset \Omega$
\begin{align}
\label{eq:Psi_generic}
	\Psi(\chi_S) &:= \bPsi\big(\int_S \xi_1 \diff \fx,\cdots,\int_S \xi_I \diff \fx\big),\\
	\xi_S &:= \sum_{1 \leq i \leq I} \partial_i \bPsi \big(\int_S \xi_1 \diff \fx,\cdots,\int_S \xi_I \diff \fx\big)\, \xi_i.
\end{align}
Then $\| \xi_S -\xi_\gS\|_\infty \leq K'_\Psi \|\chi_S - \chi_\gS\|_1$, and  
$$
	\Big| \Psi(\chi_\gS) + \int_\Omega (\chi_S-\chi_\gS) \xi_\gS \,\diff \fx - \Psi(\chi_S) \Big| 
	\leq K_\Psi \|\chi_S - \chi_\gS\|_1^2,
$$
for all shapes $\chi_S, \chi_\gS \subset \Omega$, where $K_\Psi, K'_\Psi$ are constants.
\end{proposition}

\begin{proof}
By assumption, there exists a constant $K_{\bPsi}$ such that 
\begin{equation}
\label{eq:taylor_dbPsi}
\|\diff \bPsi(\alpha)-\diff\bPsi(\beta)\|_{l^1} \leq K_{\bPsi} \|\alpha-\beta\|_{l^\infty},
\end{equation}
for all $\alpha, \beta\in \bR^I$, where $\|\cdot\|_{l^p}$ stands for the $l^p$ norm on $\bR^I$. By the Taylor formula with integral remainder, we obtain
\begin{equation}
\label{eq:taylor_bPsi2}
	| \bPsi(\alpha) + \diff \bPsi(\alpha) \cdot (\beta - \alpha) - \diff \bPsi(\beta)| \leq \tfrac 1 2 K_{\bPsi} \|\alpha-\beta\|^2_{l^\infty}.
\end{equation}
Now, let 
\begin{align*}
\alpha &:= \big(\int_\gS \xi_1\diff \fx,\cdots, \int_\gS \xi_I\diff \fx\big), \\
\beta &:= \big(\int_S \xi_1 \diff \fx, \cdots, \int_S \xi_I\diff \fx\big),
\end{align*}
in such way that $\Psi(\chi_\gS) = \bPsi(\alpha)$, $\Psi(\chi_S) = \bPsi(\beta)$ and 
$\int_\Omega (\chi_S-\chi_\gS) \xi_\gS \,\diff \fx = \diff \bPsi(\alpha) \cdot (\beta-\alpha)$. Letting $K_\xi := \max_{1 \leq i \leq I} \|\xi_i \|_\infty$ we obtain in addition  $\|\alpha-\beta\|_{l^\infty} \leq K_\xi \|\chi_S-\chi_\gS\|_1$ and 
$
	\|\xi_S-\xi_\gS\|_\infty \leq K_\xi \|\diff \bPsi(\alpha)-\diff\bPsi(\beta)\|_{l^1}.
$
The announced estimates follow by substitution in \eqref{eq:taylor_dbPsi} and \eqref{eq:taylor_bPsi2}, with $K_\Psi := \frac 1 2 K_{\bPsi} K_\xi^2$ and $K'_\Psi := K_{\bPsi} K_\xi$.
\qed
\end{proof}

Note that the Lebesgue measure $\Leb(S) := \int_S 1 \diff \fx$, and integrals on the complementary region $\int_{\Omega\sm S} \xi \diff\fx = \int_\Omega \xi \diff\fx -  \int_S \xi \diff\fx$, may also be used as arguments of \eqref{eq:Psi_generic}.

\begin{remark}[Extension to a Hilbert space]
\labelx{rem:hilbert_gradient}
Under the assumptions of \cref{prop:region_gradient}, we may define $\Psi(u) := \bPsi(\int_\Omega \xi_1 u \diff \fx, \cdots, \int_\Omega \xi_I u \diff \fx)$ for any $u \in H := L^2 (\Omega)$. Then $\nabla_{H} \Psi(\chi_\Omega) = \xi_\Omega$, where $\nabla_H$ denotes the gradient operator in the Hilbert space $H$ (as opposed to the usual gradient $\nabla$ in $\bR^2$ which is used in the rest of this paper).
\end{remark}

\begin{remark}[Application to \cref{th:summary}]
\labelx{rem:lebesgue_inv_gradient}
The functional \eqref{eq:Psi_generic} fits the assumptions of \cref{th:summary}, provided one has in addition $\xi_1,\cdots,\xi_I \in C^\alpha(\overline \Omega)$ for some $0<\alpha\leq 1$.
\Cref{th:summary} in addition features the assumption that the shape boundary $\partial S$ is homotopic to a curve $\rC$ within its tubular neighborhood $\rT$. 
It follows that the areas of $S$ and of $\Omega \sm S$ are positively bounded below, since each contains an open connected component of $\Omega \sm \rT$, hence the expression of $\Psi$ may also depend on $1/\Leb(S)$ and $1/\Leb(\Omega\sm S)$. 
\end{remark}

\subsubsection{The Piecewise Constants-based Model}
Let $\cI:\overline\Omega\to\bR$ be a gray level image. \citet{chan2001active} introduced a piecewise constants-based active contour model which is regarded as a practical reduction of the full Mumford-Shah model. This model assumes that the image gray levels can be well approximated via a piecewise constant function. The corresponding  energy functional reads
\begin{align*}
\Psi(\chi_S,\kc_{\rm in},\kc_{\rm out}) := 
\int_S (\cI-\kc_{\rm in})^2 d\fx
+ \int_{\Omega\sm S}(\cI-\kc_{\rm out})^2\diff \fx.
\end{align*}
where $\kc_{\rm in}$ and $\kc_{\rm out}$ are two scalar values. 
The gradient of this energy at $\chi_\gS$ reads by \cref{prop:region_gradient}
\begin{equation}
\label{eq_PConstGradient}
\xi_{\gS}(\fx)=(\cI(\fx)-\kc_{\rm in})^2-(\cI(\fx)-\kc_{\rm out})^2,
\end{equation}
which again is independent of the current shape $\gS$. This fits the assumptions of \cref{th:summary}, provided the images $\cI \in C^\alpha(\overline \Omega)$ for some $\alpha>0$, which we now assume.
Some image segmentation scenarios also consider the mean intensities 
\begin{align*}
	\kc_{\rm in}(S) &:= \frac{\int_S \cI(\fx) \diff \fx}{\Leb(S)}, &
	\kc_{\rm out}(S) &:= \frac{\int_{\Omega\sm S} \cI(\fx) \diff \fx}{\Leb(\Omega\sm S)},
\end{align*}
and define the slightly more complex functional 
\begin{align*}
	\tilde \Psi(\chi_S) := &\,\Psi(\chi_S,\kc_{\rm in}(S),\kc_{\rm out}(S))\\
	=&\int_S \cI(\fx)^2 \diff \fx - \frac{(\int_S \cI(\fx)\diff \fx)^2}{\Leb(S)} \\
	&+ \int_{\Omega\sm S} \cI(\fx)^2 \diff \fx - \frac{(\int_{\Omega\sm S} \cI(\fx)\diff \fx)^2}{\Leb(\Omega\sm S)},
\end{align*}
which is the sum of the variances of $\cI$ within $S$ and $\Omega\sm S$, weighted by their areas. By \cref{prop:region_gradient,rem:shape_gradient}, this again fits the assumptions of \cref{th:summary}.

Finally, we refer to~\citep{chan2000active} for an extension of this piecewise constants-based model which is designed to deal with vector-valued images.

\subsubsection{Bhattacharyya Coefficient Model}
The Bhattacharyya coefficient~\citep{michailovich2007image} between the  histograms of image colors inside and outside a shape $\aS$ is defined as 
\begin{align*}
\Psi(\chi_S)&:= \int_{\mathbf{Q}} \Psi_{\mathbf{q}}(\chi_S) \diff \fq,\qquad \text{where}\\
\Psi_{\mathbf{q}}(\chi_S) &:= \sqrt{\mathfrak{H}_{\rm in}(\mathbf{q},\chi_S)\,\mathfrak{H}_{\rm out}(\fq,\chi_S)},
\end{align*}
and where the histograms $\mathfrak{H}_{\rm in}$ and $\mathfrak{H}_{\rm out}$ are defined in Eqs.~\eqref{eq_PDFIn} and~\eqref{eq_PDFOut}.
By \cref{prop:region_gradient} the functional $\Psi_{\mathbf{q}}$ fits the assumptions of \cref{th:summary}, for any fixed $\mathbf{q} \in \mathbf{Q}$, provided the color image $\fI=(\cI_1,\cI_2,\cI_3):\overline\Omega\to\bR^3$ obeys $\cI_i\in C^\alpha(\overline \Omega)$ for some $\alpha>0$ and for $i=1,2,3$, hence so does $\Psi$ by integration over $\mathbf{Q}$. 
Eventually the gradient $\xi_{\gS}$ of the region-based energy functional $\Psi$ at $\chi_{\gS}$ is obtained as 
\begin{align*}
\xi_\gS(\fx)=&\frac{1}{2}\Psi(\chi_\gS)\left(\Leb(\gS)^{-1}-\Leb(\Omega\backslash\gS)^{-1}\right)\\
&+\frac{1}{2}\int_{\mathbf{Q}}{G}_{\sigma}(\mathbf{q}-\mathbf{I}(\fx))\wp(\mathbf{q},\chi_\gS)\diff\mathbf{q},
\end{align*}
where $\wp(\fq,\chi_{\gS})$ is defined as
\begin{align*}
&\wp(\fq,\chi_{\gS}):=\\
&\frac{1}{\Leb(\Omega\backslash\gS)}\sqrt{\frac{\mathfrak{H}_{\rm in}(\fq,\chi_\gS)}{\mathfrak{H}_{\rm out}(\fq,\chi_\gS)}}-\frac{1}{\Leb(\gS)}\sqrt{\frac{\mathfrak{H}_{\rm out}(\fq,\chi_\gS)}{\mathfrak{H}_{\rm in}(\fq,\chi_{\gS})}}.
\end{align*}
For numerical consideration, $\mathfrak{H}_{\rm in}$ and $\mathfrak{H}_{\rm out}$ are computed  as the Gaussian-smoothed  histograms of the image data.

\subsubsection{Balloon Model}
The region-based energy of this balloon model reads
\begin{equation}
	\label{eq_Balloon}
	\Psi_{\rm balloon}(\chi_\aS):=\int_S f_{\rm balloon}\, d\fx,
\end{equation}
where $f_{\rm balloon}\in\{-1,1\}$ is a constant value \citep{cohen1991active,cohen1993finite,chen2017anisotropic}. Specifically, the choice $f_{\rm balloon}=-1$ yields an expanding force, increasing the area of the region minimizing the active contour energy \eqref{eq_BalloonEnergy} below, whereas the choice $f_{\rm balloon}=-1$ yields a shrinking force, with the opposite effect.

The gradient of $\Psi_{\rm balloon}$ at $\chi_\gS$ reads by \cref{prop:region_gradient} 
\begin{equation}
\label{eq_balloonGradient}
\xi_\gS\equiv f_{\rm balloon},
\end{equation}
which is independent of the current shape $\gS$.
The full energy of the simplified edge-based balloon model is defined as 
\begin{equation}
\label{eq_BalloonEnergy}
E_{\rm balloon}(S)=\alphareg\,\Psi_{\rm balloon}(\chi_S)+	\Length_{\cR}(\cC_S),
\end{equation}
Its optimization can thus be directly addressed using the proposed method based on Randers geodesics.

Edge-based anisotropy features derived from the image gradients can be naturally incorporated in the Riemannian metric-based regularization term $\Length_{\cR}$, and thus drive the course of the contour evolution.

\subsection{Numerical Construction of the Randers Vector Field}
\label{subsec:curl_implem}
An important ingredient of the numerical implementation of the proposed geodesic model is the computation of the vector field $\omega_{\gS}$ defining the Randers metric used in our segmentation method, see \cref{th:stokes}. This vector field solves a curl PDE, over a tubular domain $\rT$, whose r.h.s.\ is the gradient $\xi_\gS$ of the region-based energy functional $\Psi$ at a given shape $\gS$, see \eqref{eq:taylor_phi} and \eqref{eq:curl_omega}. 
During \Cref{alg_ContourEvolution} which governs the contour evolution, the vector field $\omega_\gS$ is updated in each iteration, whereas the domain $\rT$ is either fixed or optionally updated.

\subsubsection{Convolution Method}
The Randers vector field $\omega_\gS$ constructed in the proof of \cref{th:exists_omega} can be expressed as
\begin{align}
\label{eq:omega_expl}
	\omega_\gS &= (\chi_{\rT_2}\, \xi_\gS) \ast \fH, &
	\text{where } \fH(\fz) := \frac {\dfe_\fz^\perp} {2 \pi \|\fz\|},
\end{align}
which is the convolution of the region-based energy gradient $\xi_\gS$ associated to $\gS$, cropped to the extended tubular domain $\rT_2$ as defined in Eq.~\eqref{eq_DoubleTube}, with the integrable kernel denoted by $\fH$. 
The expression of $\fH(\fz)$ involves $\dfe_\fz^\perp$ the counterclockwise perpendicular to the unit vector $\dfe_\fz := \fz/\|\fz\|$, for all $\fz\neq \mathbf{0}$.
For practical purposes we set $\fH(\mathbf{0}) = \mathbf{0}$.
From the numerical standpoint, the convolution \eqref{eq:omega_expl} can be implemented efficiently and accurately using the fast Fourier transform, or alternatively by direct computation at a reasonable cost if one crops the vector-valued kernel $\fH$ to a small window. In the latter case the window size is often chosen as $[-4U,4U]^2$ which is twice as wide as the tubular domain \eqref{eqdef:tubular_domain}. In a similar spirit~\citep{li2007active} also introduced a convolution with a vector-valued kernel, different from the kernel $\fH$ proposed here, for computing an extended image gradient vector field. 

\subsubsection{Alternative Variational Method}
Another natural approach to select a solution to the PDE $\curl \omega_{\gS} = \xi_{\gS}$ on $\rT$, is to choose the one of minimal $L^2(\rT)$ norm. In other words, we set $\omega_{\gS} = w^\perp$, where $w$ is a vector field solving the PDE-constrained problem~\eqref{eq_UMinimization} below. 
This solution is obtained by solving a Poisson equation on the tubular domain $\rT$, see~\cref{prop:poisson_minimization}, and leads to convincing numerical experimental results. 
From the theoretical standpoint, and in contrast with the first construction~\eqref{eq:omega_expl} whose properties are established in \cref{th:exists_omega}, it is however not clear how to ensure the smallness property of $\|\omega_{\gS}\|_\infty$ and the boundedness of $\|\diff\omega_{\gS}\|_\infty$,  when $\rT$ is a tubular domain of small width $U$, as required for the convergence analysis of our segmentation method.
Finally, let us mention that the solution \eqref{eq_UMinimization} to the divergence equation, far from original, is discussed in the introduction of \citep{bourgain2003equation} and followed by various other nonlinear and/or non-constructive approaches which certainly do not appear to fit our application. We recall that $H^1_0(\rT)$ denotes the Sobolev space of functions vanishing on $\rT$ and whose squared gradient is integrable on $\partial\rT$, whereas $H^{-1}(\rT)$ is the dual space w.r.t.\ the $L^2(\rT)$ inner product, see \citep{adams2003sobolev}. 

\begin{proposition}
\labelx{prop:poisson_minimization}
Consider the variational problem:
\begin{equation}
	\label{eq_UMinimization}
\minimize\int_\rT\|w\|^2d\fx,\quad s.t.~\diver w=\xi\text{~~on~~}\rT,
\end{equation}	
where $\xi \in H^{-1}(\rT)$. The unique solution is $w = \nabla \vp$, where $\vp\in H^1_0(\rT)$ obeys $\Delta\vp = \xi$ on $\rT$ and $\vp=0$ on $\partial \rT$.
\end{proposition}

\begin{proof}
	The existence of a solution $\vp \in H^1_0(\rT)$ to the Poisson equation $\Delta \vp=\xi$, with a r.h.s.\ $\xi$ lying in the dual Sobolev space $H^{-1}(\rT)$, is a classical result \citep{adams2003sobolev}. Note that the boundary conditions $\vp=0$ are actually redundant with the definition of $H^1_0(\rT)$.
	Defining $w = \nabla \vp$ one obtains $\diver w = \diver( \nabla \vp) = \Delta \vp=\xi$, hence the PDE constraint is satisfied. 
	On the other hand, assume that $\diver(w+\eta) = \xi$. Then $\diver \eta=0$ and therefore 
	\begin{equation*}
		\int_\rT \<\nabla \vp,\eta\> \diff \fx 
		=\int_{\partial\rT} \vp \<\eta, \mathrm{n}\> \diff \fs - \int_\rT \vp \diver \eta \, \diff \fx
		= 0,
	\end{equation*}
	where $\mathrm{n}$ denotes the unit normal to $\partial \rT$, and $\diff \fs$ is the surface element on $\partial \rT$. We have shown that $\int_\rT\<w,\eta\> \diff \fx= 0$, and therefore $\int_\rT\|w+\eta\|^2 \diff \fx = \int_\rT (\|w\|^2+\|\eta\|^2)\diff \fx \geq \int_\rT \|w\|^2 \diff \fx$. This establishes, as announced \eqref{eq_UMinimization}, that $w$ is the solution of smallest $L^2$ norm to 
	$\diver w = \xi$.
\qed
\end{proof}

Fig.~\ref{fig_VFDiv} depicts an example for the Randers vector field $\omega_{\gS}$, obtained by the convolution method and the variation method, respectively. In Figs.~\ref{fig_VFDiv}a and~\ref{fig_VFDiv}b, the shape boundary $\partial\gS$ and the tubular domain centered at $\partial\gS$ is illustrated. We apply the color coding for visualizing $\omega_{\gS}$, as depicted in Figs.~\ref{fig_VFDiv}c and~\ref{fig_VFDiv}d, respectively. 

In connection with the computation of the Randers vector field $\omega_\gS$, a crucial issue encountered in the numerical implementation of our region-based Randers geodesic model lies in the choice of the width of the tubular domain $\rT$. That is, by reducing the width $U$, we can limit the norm $\|\omega_{\gS}\|_\infty$, so that the compatibility condition $\|\omega_\gS(\fx)\|_{\cM(\fx)^{-1}}$ holds for all $\fx \in \rT$, and thus the Randers metric $\cF^\gS$ is definite, see \cref{th:stokes}. However, reducing $U$ also limits the search region for geodesic paths, which increases the number of iterations of \cref{alg_ContourEvolution}, and may ultimately lead the model to fail if the neighborhood width is less than the grid scale. In next subsection, we propose a heuristic strategy for alleviating this issue. 

\subsection{Heuristic Construction of a Modified Randers Metric}
\label{subsec_PracticalMetric}
In the proposed piecewise geodesic paths extraction scheme, a crucial step is to compute a geodesic path within each subregion to connect two successive landmark points, which minimizes  the weighted curve length $\Length_{\gS}$ associated to the Randers metric $\cF^{\gS}$, see~\cref{eq:rander_stokes}. 
We propose in this subsection a modified metric  $\kF^\gS$ which accelerates the shape evolution by (i) removing the restoring force associated to the divergence term $D(S\|\gS)$, and (ii) introducing a non-linear transformation $\bpsi$ which automatically enforces the compatibility condition \eqref{eq_compatibilityVariant}, and thus eliminates the need to reduce the tubular width $U$ for that purpose. 
There is no guarantee that the algorithm, modified as such, necessarily leads to the minimization of the energy functional, since this practical construction falls somewhat outside of the strict assumptions of the theoretical analysis as discussed in~\cref{sec_AC}.
Nevertheless, the evolution of the shapes appears to remain stable and to converge to the correct limit. We also discuss the construction of the Riemannian metric tensor field $\cM$. 

The modified Randers metric $\kF^{\gS}$ constructed in this section takes the form
\begin{equation}
\label{eq_VariantRanders}
\kF^{\gS}_\fx(\dfx):=\|\dfx\|_{\cM(\fx)}+\langle\dfx,\rV_{\gS}(\fx)\rangle,
\end{equation}
where $\cM: \overline\Omega \to \bS_2^{++}$ is a fixed tensor field, and $\rV_{\gS} : \overline\Omega \to \bR^2$ is a vector field depending on the shape $\gS$. 
We do not focus here on the regularity of $\cM$ and $\rV_{\gS}$, but nevertheless ensure Randers compatibility condition: for all $\fx \in \overline\Omega$
\begin{equation}
\label{eq_compatibilityVariant}
\|\rV_{\gS}(\fx)\|_{\cM(\fx)^{-1}}<1.
\end{equation}

The tensor field $\cM$ is constructed at the initialization of the algorithm, from the image edge-based features $g : \overline\Omega \to [0,\infty[$ and $\kg : \overline\Omega \to \bS^1$ as described in Appendix~\ref{Appendix_EdgeFeatures}, where $\bS^1 := \{\dfx\in \bR^2 \mid \|\dfx\|=1\}$ denotes the unit circle. 
At a given point $\fx$, the non-negative scalar $g(\fx)$ describes the \emph{strength} of the edges, whereas the unit vector $\kg(\fx)$ is roughly aligned with the image gradient, and thus orthogonal to the image edges. 
The matrix $\cM(\fx)$ can be expressed in terms of its  eigenvalues $\lambda_i(\fx)$ and eigenvectors $\vartheta_i(\fx)$ for $i \in \{1,2\}$ as follows
\begin{equation*}
\cM(\fx)=\sum_{i=1}^2 \lambda_i(\fx)\vartheta_i(\fx)\vartheta_i(\fx)^\top.
\end{equation*}
Our construction, detailed below, satisfies $\lambda_1(\fx)\leq \lambda_2(\fx)$ and $\vartheta_2(\fx) = \vartheta_1(\fx)^\perp$. Specifically, we define
\begin{equation}
\label{eq_Eigen1}
\lambda_1(\fx)=\exp(\ell_{\rm mag}(\|g\|_{\infty}-g(\fx))),
\end{equation}
and 
\begin{equation}
\label{eq_Eigen2}
\lambda_2(\fx)=\lambda_1(\fx)\exp(\ell_{\rm aniso}\,g(\fx)),
\end{equation}
where $\ell_{\rm mag},\,\ell_{\rm aniso}$ are two positive constants. 
For the eigenvectors of $\cM$, we set $\vartheta_2(\fx):=\kg(\fx)$, in such way that $\vartheta_1(\fx) := -\vartheta_2(\fx)^\perp$ is tangent to the image edge (if any) at $\fx$, whereas $\vartheta_2(\fx)$ is orthogonal to it. Since $\lambda_1(\fx)\leq \lambda_2(\fx)$, this construction of the metric $\cM$ locally assigns a smaller cost to paths that are tangent to the image edges than to the orthogonal ones.
From Eqs.~\eqref{eq_Eigen1} and \eqref{eq_Eigen2} we obtain
\begin{align}
\label{eq_EigenLimitation}
1 &\leq \lambda_1(\fx) \leq \lambda_2(\fx), &
\inf \{\lambda_1(\fx)\mid \fx \in \overline\Omega\}&=1.
\end{align}
Notice that in our previous work~\citep{chen2016finsler}, we only considered an isotropic tensor field $\cM(\fx)=\lambda_1(\fx)\Id$. This earlier metric construction thus only took into consideration of the image edge saliency features $g$, and not the edge anisotropy information which is encoded in $\kg$. 

\begin{remark}[Regularity of the tensor field $\cM$]
\labelx{rem:reg_M}
The image gradient vector $\kg(\fx)$ is obtained as the eigenvector associated to the largest eigenvalue of a symmetric matrix $\cQ(\fx)$ often referred to as the structure tensor, see Eq.~\eqref{eqdef:cQ} in Appendix~\ref{Appendix_EdgeFeatures}, which has $C^\infty(\Omega)$ regularity since it is defined in terms of convolutions of the image.
Denoting by 
\begin{equation*}	
X := \{\fx \in \overline \Omega\mid \exists \lambda, \cQ(\fx) = \lambda \Id\},
\end{equation*}
the set of points where the eigen-decomposition of $\cQ$ is not uniquely defined, we obtain that $\kg \in C^\infty(\Omega\sm X)$ (up to changes in direction of $\kg$, which have no effect on symmetric expressions such as $\kg \kg^\top$ appearing in the expression of $\cM$). 
By a dimensionality argument, $X$ generically consists of finitely many isolated points, which in addition are far from the edge features of the processed image, along which $\cQ$ is by construction strongly anisotropic hence has distinct eigenvalues. 
On the other hand $g$ has $C^\infty$ regularity except at points where the smoothed image gradient vanishes, see \eqref{eq_imageGradsColor} and \eqref{eq_imageGradsGray}, which are a subset of $X$. It follows that $\cM \in C^\infty(\Omega\sm X)$.

Thus $\cM$ fulfills the $C^1$ regularity assumption of \cref{th:summary} provided that $\rT \subset \Omega \sm X$, which is the case if the tubular search region remains close to the edge features of interest, thus covering all practical purposes. If this condition fails, then one may regularize $\cM$ by convolution as in Appendix~\ref{Appendix_EdgeFeatures}.
\end{remark}

Now let us show how to estimate the vector field $\rV_{\gS}$, from the new Randers vector field $\omega_{\gS}$ obtained in~\cref{subsec:curl_implem} by solving a curl PDE. The vector field $\rV_{\gS}$ needs to obey the compatibility condition~\eqref{eq_compatibilityVariant}, but in view of Eq.~\eqref{eq_EigenLimitation} it suffices that $\|\rV_{\gS}\|_\infty<1$. 
For that purpose, we define  
\begin{align*}
	\bpsi(\fz) &:= \psi(\|\fz\|) \fz, &
	\text{where } \psi(a) := \frac{1-\exp(-a)}{a},
\end{align*}
for all $\fz \in \bR^2$ and all $a \geq 0$, with the convention $\bpsi(\boldsymbol{0}) = \boldsymbol{0}$ and $\psi(0)=1$. One easily checks that $\bpsi \in C^1(\bR^2,\bR^2)$ and 
\begin{enumerate}[(a).]
\item $\|\bpsi(\fz)\| < 1$ for all $\fz\in\bR^2$.
\label{enu_pr1}
\item $\bpsi(\fz)=\fz+\cO(\|\fz\|^2)$ for small $\fz \in\bR^2$.
\label{enu_pr2}
\end{enumerate}
We define, for all $\fx \in \overline\Omega$, the modified vector field
\begin{equation}
\label{eq_NonlinearVectorField}
	\rV_\gS(\fx)= \bpsi(\alphareg\,\omega_\gS(\fx)).
\end{equation}
This choice of $\rV_\gS$ satisfies the compatibility condition~\eqref{eq_compatibilityVariant}, in view of property~(\ref{enu_pr1}), and $\rV_\gS$ retains the $C^1$ regularity of the original Randers vector field $\omega_\gS$.

In the rest of this subsection, we discuss the differences between the Randers metric $\cF^\gS$ \eqref{eq:rander_stokes} that is considered in the theoretical analysis of~\cref{sec_AC}, and the variant $\kF^\gS$~\eqref{eq_VariantRanders} that is usually preferred in the numerical experiments. 
\vspace{-0.5\baselineskip}\\

\noindent\emph{Introduction of the active contour energy parameter $\alphareg$.}
The hybrid active contour energy \eqref{eq_HybridEnergy} features a coefficient $\alphareg \in]0,\infty[$ which controls the relative importance of the region-based term $\alphareg\Psi$ and of the edge-based term associated to the tensor field  $\cM$. Until now, we assumed that $\alphareg=1$, for simplicity and without loss of generality. Choosing an arbitrary $\alphareg$ amounts to replacing $\Psi$ with $\alphareg\Psi$, thus $\omega_\gS$ with $\alphareg\omega_\gS$ (since $\omega_\gS$ depends linearly on $\Psi$), which appears in the definition \eqref{eq_NonlinearVectorField} of the non-linearly modified vector field $\rV_\gS$.
\vspace{-0.5\baselineskip}\\

\noindent\emph{Removal of the divergence term $D(S\|\gS)$.} 
The original Randers metric $\cF^\gS$ defined in \cref{th:stokes} features a weighting term $(1+2\lambda\, d_{\partial\gS}(\fx)^2)$ multiplying the tensor field $\cM$. This term comes from the divergence $D(S\|\gS)$ defined in Eq.~\eqref{eq:diver_SS}, introduced in the approximate energy $E_\gS(S)$ in \cref{subsec_RandersInterpretation}. In essence, this term acts as a restoring force, stabilizing the contour evolution in \cref{alg_ContourEvolution}, and  used in \cref{th:successive_min} to establish the decrease of the energy $E(S_n)$ of the successive shapes produced along the iterations.
Practical experience shows however that the proposed algorithm usually does not suffer from instabilities, hence this term is removed from the implementation so as to allow for a faster contour evolution and to speed up the computation time.
\vspace{-0.5\baselineskip}\\

\noindent\emph{Use of the non-linear function $\bpsi$.} The presence of the non-linearity $\bpsi$ in~\cref{eq_NonlinearVectorField} means that $\curl \rV_\gS \neq \xi_\gS$ in general, which appears to break the usage of Stokes formula upon which our approach is based, see \cref{th:stokes}. A very basic justification for~\cref{eq_NonlinearVectorField} is that $\omega_\gS(\fx)$ and $\rV_\gS(\fx)$ have the same direction at each point $\fx$, hence define in their respective Randers metrics $\cF^\gS$ and $\kF^\gS$ an asymmetric term which promotes geodesic paths aligned with the same direction (opposite to theirs), hence a similar qualitative behavior is expected. 

In addition we usually tune the region-based energy importance parameter $\alphareg$ in such way that $\|\alphareg \omega_\gS\|_\infty \leq 7$ on $\rT$, see~\cref{eq_nonlinearPara}, and in practice $\|\alphareg \omega_\gS\| < 1$ on a large portion of $\rT$, especially close to the centerline. As a result, the nonlinear transformation~\eqref{eq_NonlinearVectorField} often falls in the regime where it is close to the identity, see property (\ref{enu_pr2}).

Eventually, a small tubular domain width $U$ is still preferable in order to reduce the risk of shortcuts problem encountered in the extraction of geodesic paths. 
In the next section, we present a method for constructing an asymmetric variant of the tubular domain $\rT = \rT_{\gS}$, by using the region-based energy gradient $\xi_\gS$ to identify and suppress some redundant components of $\rT_{\gS}$.

\begin{figure*}[t]
\centering
\includegraphics[height=8.5cm]{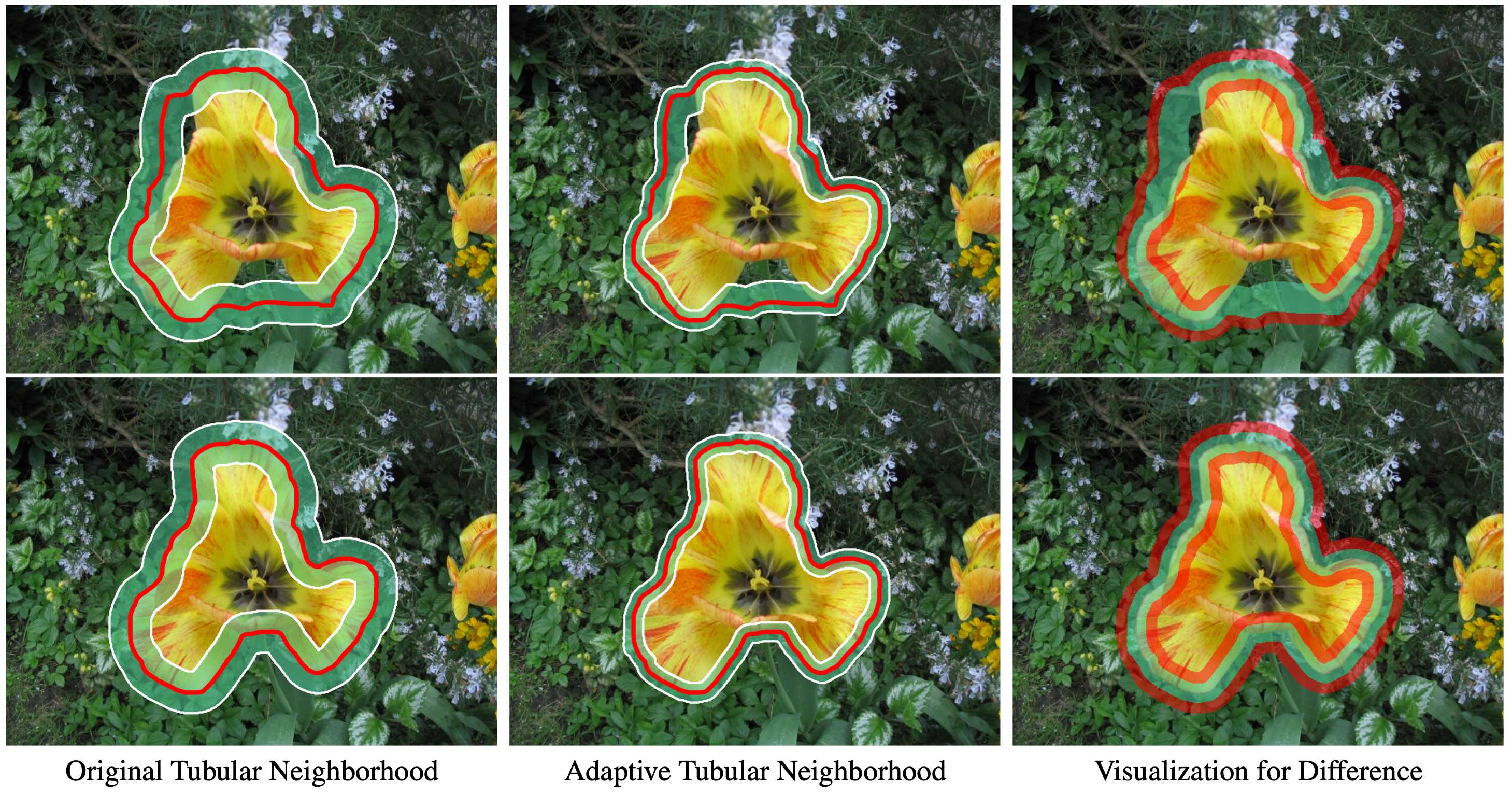}
\vspace{1mm}
\caption{Visualization for the original tubular neighborhood and its adaptive variant in two different contour evolution iterations.  The results in each row are derived from the same iteration.  Columns $1$ and $2$ respectively illustrate the original and  adaptive tubular neighborhood regions in the corresponding iterations. The transparent regions indicate the tubular neighborhood regions whose centerlines are indicated by red lines. The white lines are the boundaries of the tubular neighbourhood regions. In column 3, the red regions visualize the difference sets between the original and adaptive tubular neighbourhoods.}
\label{fig:AsyTube}
\end{figure*}

\subsection{Construction of an Adaptive Tubular Neighborhood}
\label{subsec_AsyTube}
We assumed so far that the tubular search space $\rT$ is centered on a given centerline $\rC$, which can either be fixed as in the theoretical analysis in~\cref{sec_AC}, or regularly updated and defined as the boundary $\partial \gS$ of the shape obtained in the previous iteration of the method. 
However, the efficiency of the proposed piecewise geodesic paths extraction scheme can be improved by using an adaptive tubular domain, based on a guess on the likely evolution directions of the shape boundary, in terms of the region-based energy gradient $\xi_{\gS}$. In other words, a subdomain within  the tubular domain  $\rT_\gS$ can be predicted, such that it likely covers the output contour consisting of piecewise geodesic paths. Thus one way is to suppress the area of the complementary set to the predicted set, yielding an adaptive domain $\cT_{\gS}\subset{\rT_{\gS}}$, allowing to be asymmetric with respect to the boundary $\partial\gS$.  This can be implemented by propagating distance front from $\partial\gS$, such that the front travels slowly in the regions that we attempt to remove from the tubular domain $\rT_{\gS}$. For this purpose, we solve an isotropic eikonal PDE
\begin{equation}
\label{eq_AsyTubeEikonal}
\begin{cases}
\|d\rD_{\partial\gS}(\fx)\|=\cP_{\gS}(\fx),\quad&\forall\fx\in\rT_\gS\backslash\partial\gS,\\
\rD_{\partial\gS}(\fx)=0,\quad&\forall\fx\in\partial{\gS},
\end{cases}
\end{equation}
where $\cP_{\gS}:\overline\Omega\to\bR^+$ is a potential. Then the adaptive tubular domain $\cT_{\gS}$ is generated by
\begin{equation}
\label{eq_AdaptiveTube}
\cT_{\gS}=\{\fx\in{\rT_{\gS}}\mid \rD_{\partial\gS}(\fx)<U\}.
\end{equation}
The construction of the potential $\cP_{\gS}$, involves two subsets $\bD_\gS$ and $\tilde{\bD}_\gS$ of the domain $\overline\Omega$ defined as follows:
\begin{align*}
\bD_\gS &:= \{\fx\in\gS\mid\xi_{\gS}(\fx)\geq \varrho\}
\cup\{\fx\in\overline\Omega\backslash\gS\mid\xi_\gS(\fx)\leq -\varrho\},\\
\tilde{\bD}_\gS &:= \{\fx\in\overline\Omega\mid|\xi_\gS(\fx)|< \varrho\},	
\end{align*}
where $\varrho \in [0,\infty[$ is a sufficiently small parameter. Then 
\begin{equation}
\label{eq_TubePotential}	
\cP_{\gS}(\fx):=
\begin{cases}
\upsilon, &\forall\fx\in\bD_\gS,\\
1,&\forall\fx\in\tilde{\bD}_{\gS},\\
1/\upsilon,&\text{otherwise}.
\end{cases}
\end{equation}
where $\upsilon\in\,]0,1]$ is a given parameter.
The construction of the adaptive tubular neighborhood $\cT_{\gS}$ using Eq.~\eqref{eq_AdaptiveTube} is regarded as a refinement process to the original tubular neighborhood $\rT_{\gS}$.
In Fig.~\ref{fig:AsyTube}, we show an example to illustrate the difference between the original and adaptive tubular neighborhood regions.  The asymmetry is pronounced in the early steps of the curve evolution, see \cref{fig:AsyTube} (top), whereas an approximately symmetric neighborhood is recovered in the final steps, see \cref{fig:AsyTube} (bottom), consistently with the above discussion regarding the non-linear transformation $\bpsi$. 
In this experiment,  the Bhattacharyya coefficient of two histograms is applied for the computation of the region-based energy gradient $\xi_{\gS}$ in our Randers geodesic model.
In this experiment, we set the width parameter $U=20$ (measured in grid points) for the construction of the both tubular neighborhoods. In addition, we set $\upsilon=0.2$ and $\varrho=0.1\|\xi_\gS\|_\infty$ for the estimation of the potential~\eqref{eq_TubePotential}. Finally, a Gaussian mixture model is used to define the region-based appearance term $\Psi$, and thus to compute its gradient $\xi_{\gS}$ involved in our Randers geodesic model.

In the course of the contour evolution, once the adaptive neighborhood $\cT_n$ is constructed associated to the boundary $\partial{S}_n$, we can decompose $\cT_n$ as a family of subregions $\cZ_{k,n}$ for  $1\leq k\leq m$ using the decomposition $Z_{k,n}$ of $\rT_{n}$, such that
\begin{equation}
\cZ_{k,n}:=\cT_{n}\cap Z_{k,n}.	
\end{equation}
In this way, the geodesic paths $\cG_{k,n}$ are extracted in each subregion $\cZ_{k,n}$, see Sections~\ref{subsec_FixedPointsScheme}.

\begin{figure*}[t]
\centering
\includegraphics[height=3.7cm]{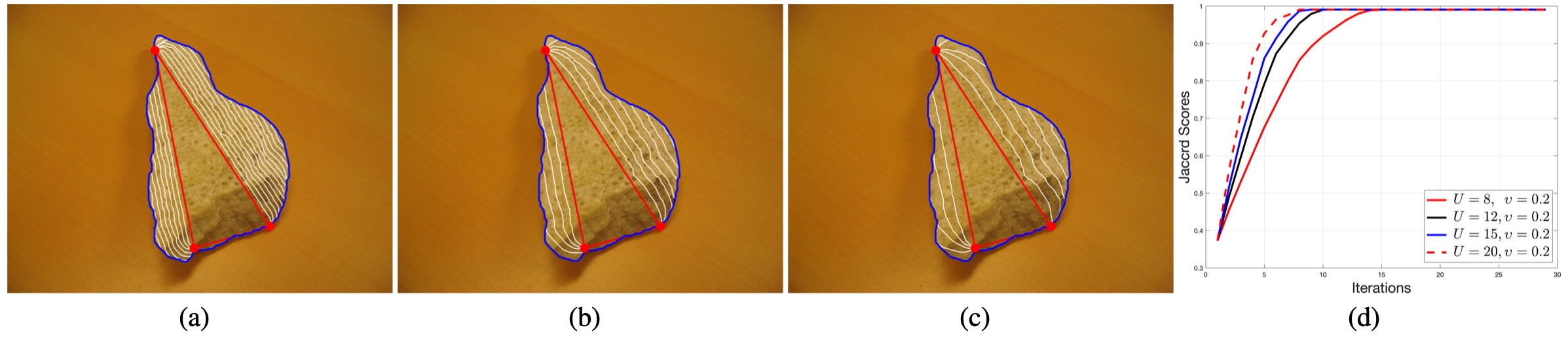}
\caption{Examples for the course of the contour evolution with respect to different values of the parameter $U$ for the construction of the adaptive tubular neighborhood. \textbf{a}-\textbf{c} The values of $U$ are respectively set as $8$, $15$ and $20$ (grid points).  The red dots are the landmark points.  \textbf{d} The plots of the Jaccard scores with respect to different values of the parameter $U$.}
\label{fig:tubeParas}
\end{figure*}

\begin{figure*}[t]
\centering
\includegraphics[height=3.7cm]{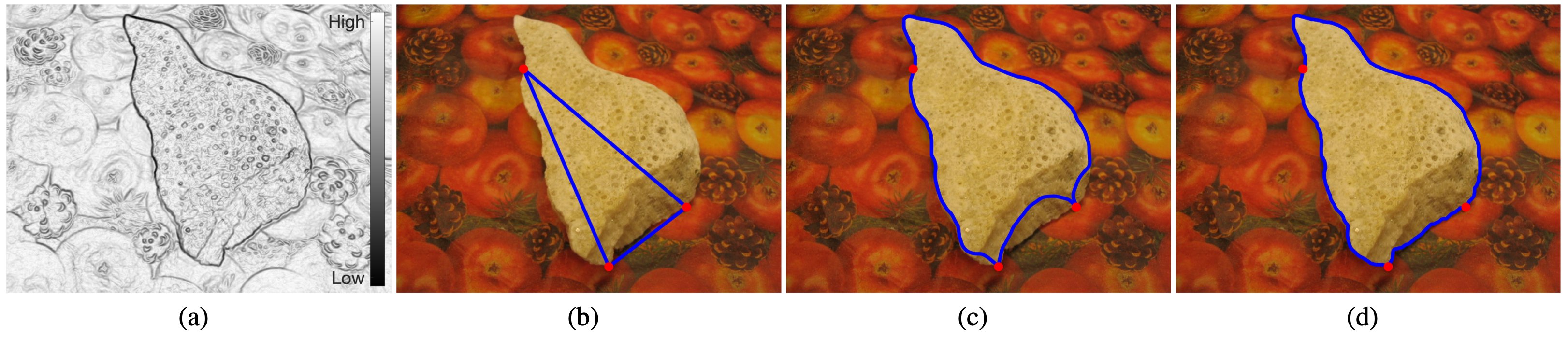}
\caption{Image Segmentation results with respect to different tensor fields $\cM$.  \textbf{a} The visualization for the eigenvalues $\tilde\lambda_1$ of $\cM$. \textbf{b} Initialization contour generated by the polygon construction method. The landmark points are indicated by red dots. \textbf{c} Image segmentation using $\cM$ associated with weights $\ell_{\rm mag}=\ell_{\rm aniso}=0$. \textbf{d} Image segmentation using $\cM$ associated with weights $\ell_{\rm mag}=2$ and $\ell_{\rm aniso}=1$.}
\label{fig_EffectsTensorField}
\end{figure*}

\section{Experimental Results}
\label{sec_Experiments} 
In this section, we verify the performance of the proposed region-based Randers geodesic model in image segmentation, evaluated on both synthetic and natural images.  
In the following experiments, we first discuss the influence of the tubular neighborhood and of the components of the Randers metrics. We then perform qualitative and quantitative comparisons between the proposed model with different initialization ways and state-of-the-art optimal paths-based image segmentation models.

\subsection{Parameter Setting}
\label{subsec_Parameters}
We discuss here the influence of the relevant parameters for (i) the computation of the modified Randers metrics $\kF$, and (ii) the  construction of the adaptive tubular neighborhood, which dominate the extraction of Randers geodesic paths in the proposed model.

\subsubsection{Parameters for the Modified Randers Metrics} 
Given a shape $\gS$, the modified Randers metric $\kF^{\gS}$ takes the form of~\cref{eq_VariantRanders}, which involves a tensor field $\cM$ and a vector field $\rV_{\gS}$ as described in Section~\ref{subsec_PracticalMetric}.  
For the construction of the tensor field $\cM$ in terms of image gradients, the parameters $\ell_{\rm mag}$ and $\ell_{\rm aniso}$ should be assembled properly. Typically, we choose $\ell_{\rm mag}\in\{0.5,2,3\}$, and $\ell_{\rm aniso}=1$ in the following experiments.  Note that the parameter $\ell_{\rm aniso}$ will be automatically set to $0$ if $\ell_{\rm mag}=0$. 

The parameter $\alphareg\in[0,\infty[$ used in Eq.~\eqref{eq_NonlinearVectorField} modulates the importance of the region-based appearance model in the image segmentation. In practice, we choose 
\begin{equation}
\label{eq_nonlinearPara}
\alphareg:=\tilde{\alphareg}/\|\omega_{\gS}\|_\infty,
\end{equation}  
where $\tilde{\alphareg}$ is a positive scalar value. This process simplifies the configuration of the modified Randers metric $\kF^{\gS}$, and also balances the importance of the symmetric and asymmetric terms in the Randers metric. (Admittedly, the choice~\eqref{eq_nonlinearPara} does however not comply with the theoretical analysis, which assumes a constant value of $\alphareg$ along the iterations.) By Eq.~\eqref{eq_nonlinearPara}, the vector field $\rV_{\gS}$ is in fact defined as $\rV_{\gS}=\bpsi(\tilde{\alphareg}\omega_{\gS}/\|\omega_{\gS}\|_\infty$). In the following experiments, the parameter $\tilde{\alphareg}$ is typically set as $\tilde{\alphareg}\in\{5,6,7\}$.

\begin{figure*}[t]
\centering
\includegraphics[height=4.6cm]{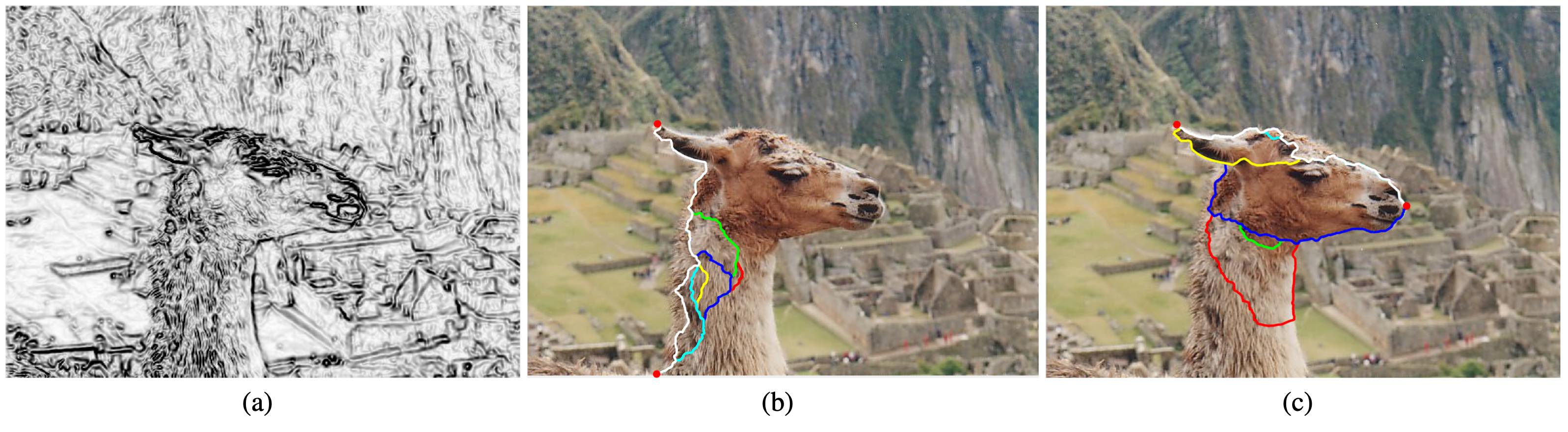}
\caption{An example for admissible paths generated from the combination of piecewise geodesic paths model, which are denoted by lines of different colors.  \textbf{a} Visualization for the potential $\cP_{\rm comb}$. \textbf{b}-\textbf{c} The admissible paths corresponding to two pairs of successive landmark points.}
\label{fig:AdmissiblePaths}
\end{figure*}

\begin{figure*}[!pht]
\centering
\includegraphics[width=17cm]{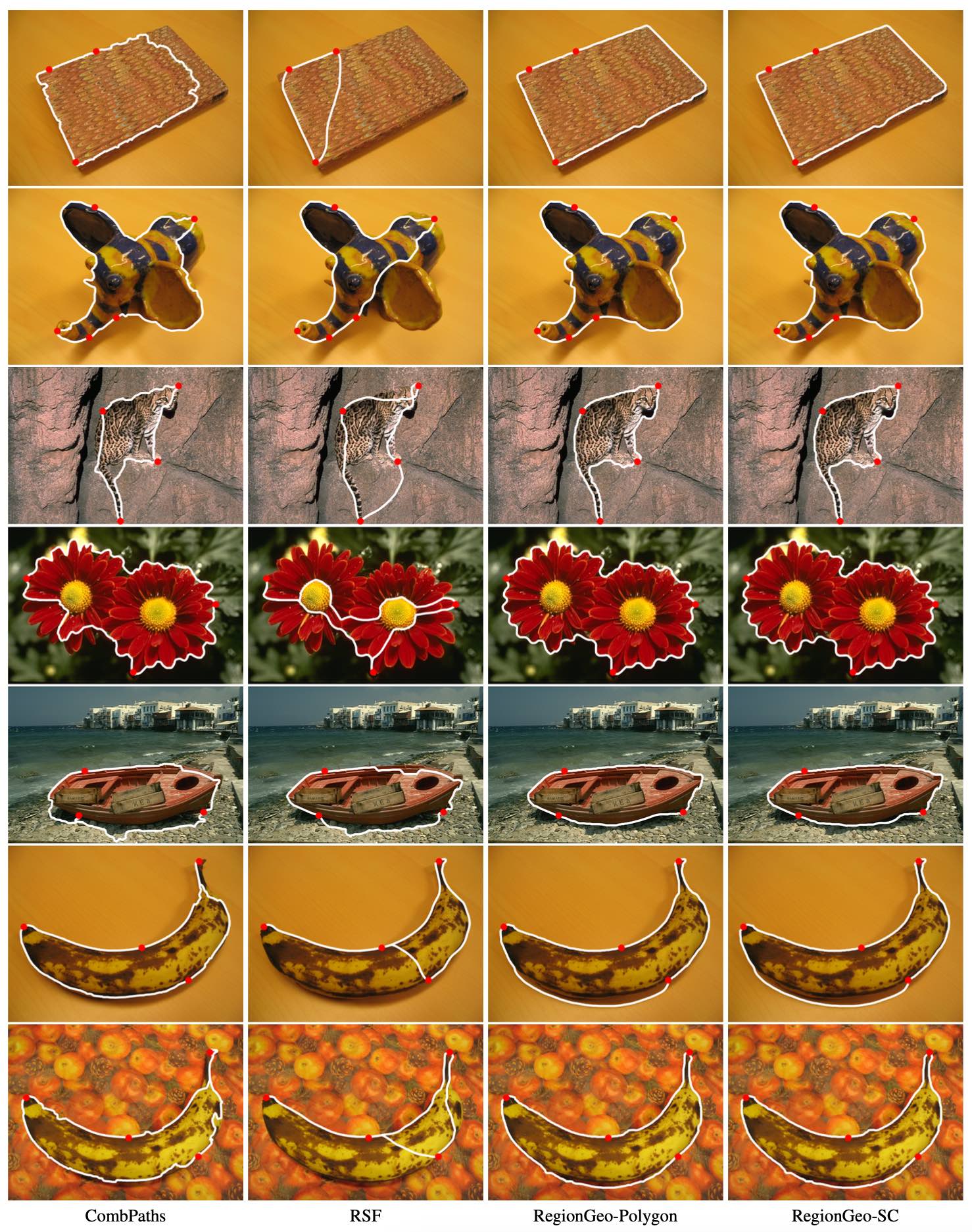}
\caption{Qualitative comparison with the combPaths model and the RSF segmentation model on  real images. \textbf{Columns} $1$ to $4$: The segmentation contours indicated by white lines are obtained from the combPaths model, the RSF segmentation model, and the proposed RegionGeo-Polygon and RegionGeo-SC models, respectively.}
\label{fig_NatureImagesComparison}
\end{figure*}

\begin{figure*}[t]
\centering
\includegraphics[height=7cm]{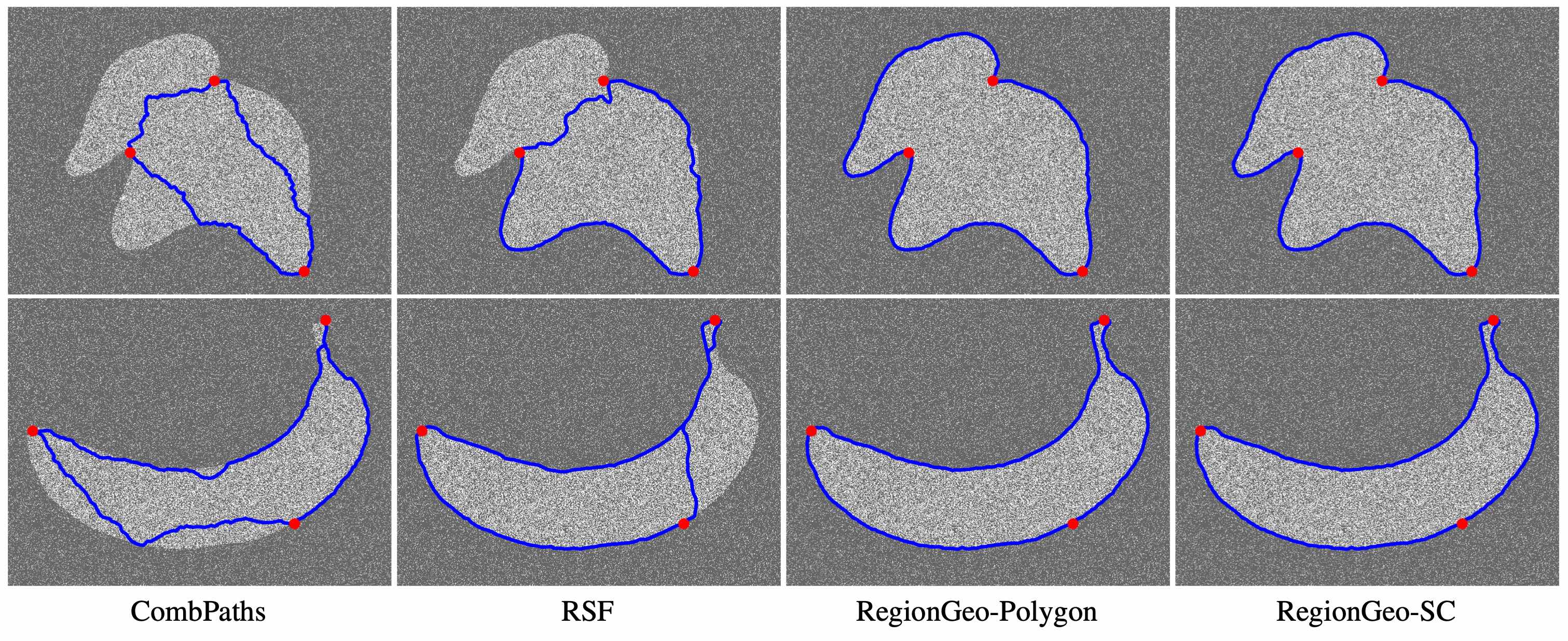}
\caption{Qualitative comparison with the CombPaths model and the RSF segmentation model on two synthetic images. \textbf{Columns} 1 to 4:  The  segmentation contours derived from the CombPaths, RSF, the proposed RegionGeo-Polygon and RegionGeo-SC models, respectively.}
\label{fig:unreliableEdges}
\end{figure*}

\subsubsection{Parameters for the Adaptive Tubular Neighborhood Construction}
In order to construct the adaptive tubular neighborhood~\eqref{eq_AdaptiveTube}, one should specify the relevant parameters $U$ and $\upsilon$, as introduced in Section~\ref{subsec_AsyTube}. An example of the effects on the convergence rate of contour evolution with respect to different values of the parameter $U$ and a fixed value of $\upsilon=0.2$ are described in Fig.~\ref{fig:tubeParas}. The initial contour and the fixed landmark points are shown in Fig.~\ref{fig:tubeParas}a,  represented by the red lines and the red dots, respectively.  In this experiment, the convergence rate is  assessed via the Jaccard accuracy score which can be defined as 
\begin{equation}
\label{eq_Jaccard}
J(\mathbb{A},\mathbb{A}_*)=\frac{|\mathbb{A}\cap \mathbb{A}_*|}{|\mathbb{A}\cup \mathbb{A}_*|},
\end{equation}
where $\mathbb{A}$ and  $\mathbb{A}_*$ respectively denote the segmented region derived from the considered segmentation models and the ground truth region, respectively.  During the segmentation process, we apply the proposed Randers geodesic model initialized with the polygon construction method for extracting piecewise geodesic paths.  Figs.~\ref{fig:tubeParas}a to~\ref{fig:tubeParas}c illustrate the evolving contours with respect different values of $U$ and a fixed value of $\upsilon=0.2$. One can see that using a small value of $U$ may increase the total number of contour iterations required by the proposed interactive segmentation model. Such a phenomenon can also be observed in Fig.~\ref{fig:tubeParas}d. However, for the consideration of practical implementation,  a high value of $U$ corresponds to a tubular neighborhood of large size, which may increase the risk of the shortcuts problem. As a tradeoff between the efficiency and accuracy of image segmentation,  we set $U\in\{12,15\}$ and fix $\upsilon=0.2$ in the remaining numerical experiments.

\subsection{Advantages of Using the Hybrid Image Features}
As discussed in Section~\ref{subsec_PracticalMetric}, the considered tensor field $\cM$ is usually constructed in terms of the image gradients. The use of such local edge-based features is able to reduce the potential influence of the image intensity inhomogeneities. These features can encourage the introduced segmentation models to find favorable segmentation results, even when the nonlocal region-based appearance models fail to accurately depict the image features. In Fig.~\ref{fig_EffectsTensorField}, we illustrate the effect derived from the tensor field $\cM$ in image segmentation. In this experiment, the fixed landmark points-based geodesic paths extraction method is applied, in conjunction with the piecewise constants-based appearance model. In Fig.~\ref{fig_EffectsTensorField}a, the edge saliency features, carried out by the scalar function~\eqref{eq_imageGradsColor}, are visualized. Fig.~\ref{fig_EffectsTensorField}b depicts the landmark points indicated by red dots, and the initial contour indicated by a blue line. In Fig.~\ref{fig_EffectsTensorField}c, we set the tensor field to be independent to the image gradient, i.e.\ $\ell_{\rm mag}=\ell_{\rm aniso}=0$. From this figure, one can observe that a portion of the  segmentation contour leaks into the region of interest. By invoking the parameters $\ell_{\rm mag}=2$ and $\ell_{\rm aniso}=1$ for the tensor field $\cM$, the resulting segmentation contour can accurately delineate the desired boundary, see Fig.~\ref{fig_EffectsTensorField}d. In this experiment, we make use of the piecewise constants-based appearance model~\citep{chan2000active,chan2001active} to construct the Randers geodesic metrics considered.

\begin{figure*}[t]
\centering
\includegraphics[width=17cm]{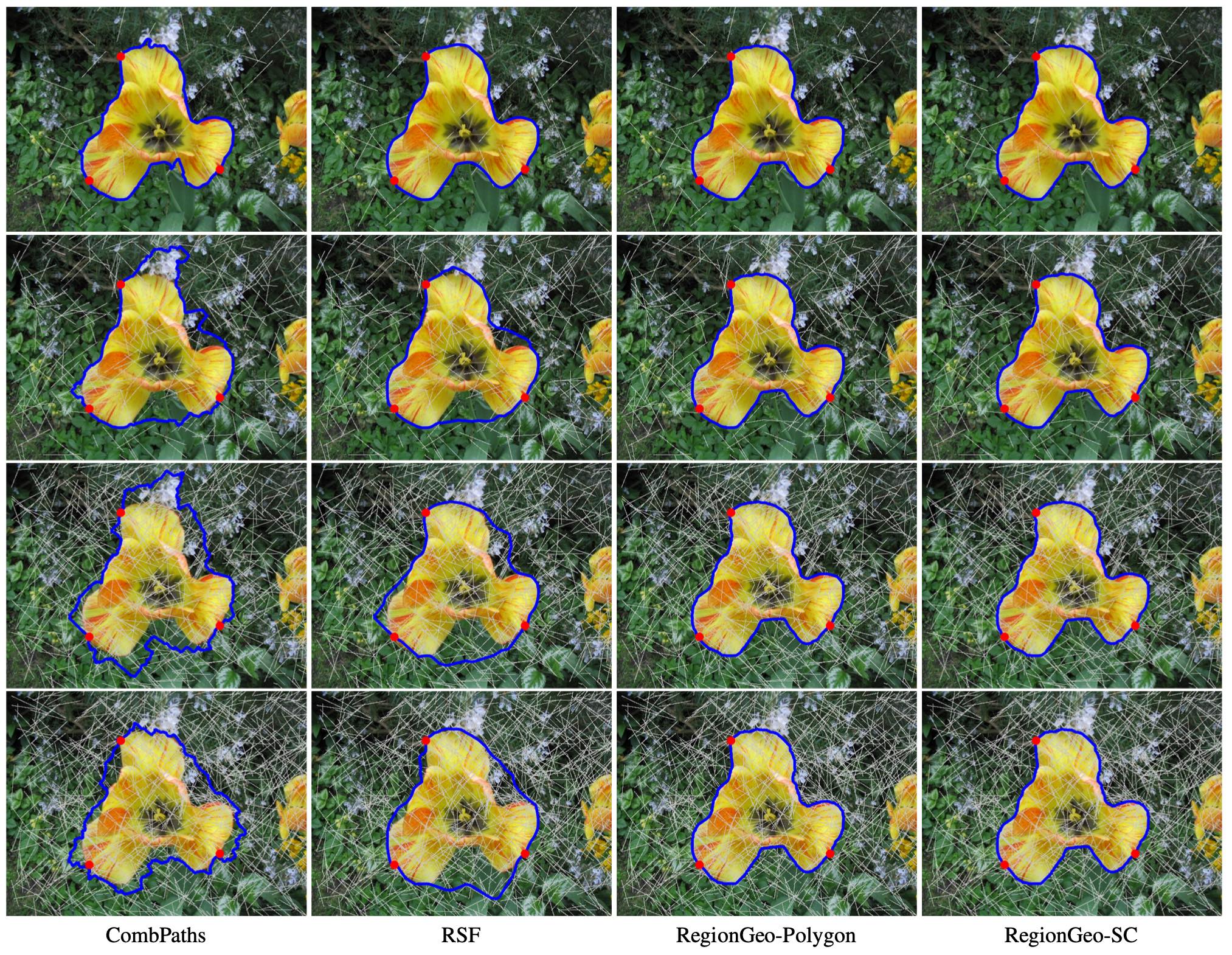}	
\caption{Comparison with the CombPaths model and the RSF segmentation model on images interrupted by arbitrary straight segments. \textbf{Columns} $1$ to $4$: The segmentation results from the CombPaths model, the RSF segmentation model, the proposed RegionGeo-Polygon and RegionGeo-SC models, respectively.}
\label{fig_ClutterDemos}
\end{figure*}

On the other hand, geodesic paths associated to the metrics which rely only on the edge-based image features favor passing through the regions with strong edge saliency features. As a consequence, these metrics often lead to unexpected paths for delineating the boundaries of interest, in case their edge-based features are insufficiently salient. In Fig.~\ref{fig:AdmissiblePaths}, we take the admissible paths computed by the combination of piecewise geodesic paths model as an example to illustrate that problem. The red dots in Fig.~\ref{fig:AdmissiblePaths}b or~\ref{fig:AdmissiblePaths}c denote a pair of successive landmark points. Also, the admissible paths in both figures are indicated by lines of different colors. The potential $\mathcal{P}_{\rm comb}$  defining the isotropic geodesic metrics for the combination of piecewise geodesic paths model is visualized in Fig.~\ref{fig:DemoComb}a. 
 From Figs.~\ref{fig:AdmissiblePaths}b and~\ref{fig:AdmissiblePaths}c, one can see that none of these admissible paths (see Eq.~\eqref{eq_CombCandidate}) can accurately delineate the desired boundary segment, since the edge-based features in this experiment suffer from the lack the reliability for the description of the target boundaries. This motivates the proposed  geodesic model, whose Randers metric definition integrates both the edge-based features and the region-based appearance model. 
 
\begin{figure*}[t]
\centering
\includegraphics[width=0.8\linewidth]{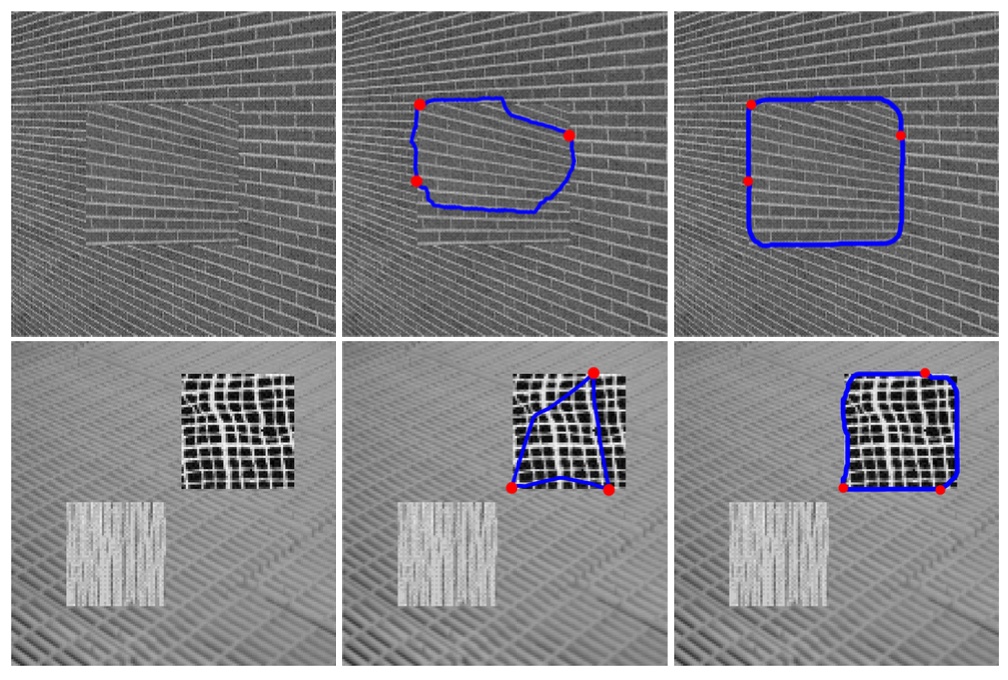}
\caption{Qualitative comparison with the CombPaths model on texture images. \textbf{Column} $1$: the original texture images. \textbf{Columns} $2$ and $3$: the segmentation results from the CombPaths model and the proposed RegionGeo-Polygon model, respectively.}
\label{fig_texture}
\end{figure*}

\subsection{Comparison with Paths-based Segmentation Models}
\label{subsec_EdgeRanders}
This subsection exhibits the qualitative and quantitative comparison results with state-of-the-art optimal paths-based interactive segmentation models. For the sake of convenience, we refer to the proposed region-based Randers geodesic model initialized with the polygon construction method (resp.\ the simple closed contour construction method) as RegionGeo-Polygon (resp.\ RegionGeo-SC).  
 
\subsubsection{Qualitative Comparison}
In this section, we qualitatively compare the proposed RegionGeo-Polygon and RegionGeo-SC models to the combination of piecewise geodesic paths (CombPaths) model~\citep{mille2015combination} and the Reeds–Shepp Forward (RSF) segmentation model\footnote{We integrate the RSF geodesic paths and a slight variant of the closed contour detection method proposed in~\citep{chen2017global} for finding image segmentation contours, as described in Appendix~\ref{appendix_RSF}.}~\citep{duits2018optimal}. 

Fig.~\ref{fig_NatureImagesComparison} illustrates comparison results with the CombPaths and RSF models on seven natural images sampled from the Grabcut dataset~\citep{rother2004grabcut} and the BSD500 datasets~\citep{arbelaez2011contour}. The first column depicts the segmentation contours (white lines) derived from the CombPaths model. The results in rows $1$ to $5$ of column $1$ suffer from the shortcuts problem, i.e.\ some portions of the obtained segmentation contours pass through the interiors of the target regions. In rows $6$ and $7$ of column $1$, we observe that the contours encounter the self-tangency problem, due to the limitation of the edge-based geodesic metrics. The segmentation results shown in column $2$ are derived from the RSF segmentation model, where the shortcuts problem can be still observed in rows $1$ and $2$. Furthermore, the RSF segmentation model suffers from more serious self-tangency problem when comparing to the CombPaths model, as illustrated in rows $4$, $6$ and $7$. One reason might be that the contour detection scheme applied with the RSF geodesic paths lacks the constraint on the contour simplicity. The segmentation contours for the proposed RegionGeo-Polygon and RegionGeo-SC models  are illustrated in column $3$ and $4$. One can point out that the segmentation contours indeed accurately capture the boundaries of the target regions, thanks to the use of the region-based appearance features and also to the geodesic path extraction scheme which guarantees the simplicity of the evolving contour, see Section~\ref{subsec_FixedPointsScheme}. In this experiment, we take the Bhattacharyya coefficient of two histograms~\citep{michailovich2007image} as the region-based appearance model.

Fig.~\ref{fig:unreliableEdges} depicts the comparison results on two synthetic images. The foreground region of each synthetic image is composed of dense small-blocks with strong intensities, while the background region involves sparse small blocks. Hence the image edges of interest are weakly visible and the edges features derived from the image gradients are not reliable. As a result, both of the CombPaths model and the RSF segmentation model cannot find favorable segmentations, as depicted in columns $1$ and $2$ of Fig.~\ref{fig:unreliableEdges}.  In columns $3$ and $4$, we demonstrate the segmentation results derived from the proposed RegionGeo-Polygon and RegionGeo-SC models, respectively. From these segmentation results, we can see that the contours accurately follow the target boundaries. In this experiment, we apply the piecewise constants-based appearance model to construct the Randers geodesic metrics.

\begin{table*}[t]
\centering
\caption{Quantitative comparisons between the CombPaths model, the RSF segmentation model, and the proposed RegionGeo-Polygon and RegionGeo-SC models in terms of the statistical values of the Jaccard scores (in percentage) evaluated over $30$ runs per image}
\label{table_NatureImages}
\setlength{\tabcolsep}{6pt}
\renewcommand{\arraystretch}{1.6}
\begin{tabular}[t]{c c c c c c c c c c c c}
\shline
Images  &\multicolumn{2}{c}{CombPaths}&\multicolumn{2}{c}{Riverbed}&\multicolumn{2}{c}{RSF }&\multicolumn{2}{c}{RegionGeo-Polygon}&\multicolumn{2}{c}{RegionGeo-SC} \\
\cmidrule(lr){2-3} \cmidrule(lr){4-5}\cmidrule(lr){6-7}\cmidrule(lr){8-9}\cmidrule(lr){10-11}
&{Mean}&{Std}&{Mean}&{Std}&{Mean}&{Std}&{Mean}&{Std}&{Mean}&{Std}\\
\cmidrule(lr){2-3} \cmidrule(lr){4-5}\cmidrule(lr){6-7}\cmidrule(lr){8-9}\cmidrule(lr){10-11}
LLAMA     &{$67.1$}&{$0.07$}  
          &{$40.2$}&{$0.20$}
          &{$54.3$}&{$0.10$} 
          &{$85.2$}&{$0.04$}
          &{$90.4$}&{$0.04$}\\
FLOWER    &{$93.9$}&{$0.06$}
          &{$82.2$}&{$0.36$}
          &{$92.4$}&{$0.12$}
          &{$95.6$}&{$0.09$}
          &{$97.9$}&{$0.06$}\\
STONE1    &{$76.2$}&{$0.26$}
          &{$56.6$}&{$0.34$}     
          &{$72.0$}&{$0.21$}
          &{$90.5$}&{$0.11$}
          &{$82.9$}&{$0.11$}\\
STONE2    &{$89.0$}&{$0.12$}
          &{$95.8$}&{$0.18$}
          &{$89.1$}&{$0.25$}
          &{$94.5$}&{$0.08$}
          &{$98.4$}&{$0.04$}\\
BOOK      &{$86.5$}&{$0.19$}
          &{$76.6$}&{$0.30$}
          &{$51.7$}&{$0.17$}
          &{$93.0$}&{$0.01$}
          &{$92.9$}&{$0.01$}\\
CERAMIC   &{$84.9$}&{$0.03$}
          &{$45.4$}&{$0.30$}
          &{$76.9$}&{$0.1$}
          &{$83.7$}&{$0.08$}
          &{$85.3$}&{$0.06$}\\
OCELOT    &{$44.7$}&{$0.18$}
          &{$27.8$}&{$0.09$}
          &{$34.8$}&{$0.18$}
          &{$79.9$}&{$0.01$}
          &{$80.5$}&{$0.05$}\\
FLOWER2     &{$71.3$}&{$0.11$}
            &{$24.5$}&{$0.10$}
            &{$30.5$}&{$0.10$}
            &{$96.6$}&{$\approx0$}
            &{$96.3$}&{$\approx 0$}\\    
BOAT     &{$78.5$}&{$0.07$}
         &{$61.3$}&{$0.24$}
         &{$52.4$}&{$0.33$}
         &{$91.8$}&{$0.1$}
         &{$91.9$}&{$0.09$}\\
BANANA1  &{$72.5$}&{$0.10$}
         &{$74.0$}&{$0.06$}
         &{$50.5$}&{$0.27$}
         &{$86.8$}&{$0.08$}
         &{$86.8$}&{$0.07$}\\
BANANA2   &{$84.8$}&{$0.11$}
          &{$53.6$}&{$0.18$}
          &{$72.4$}&{$0.26$}
          &{$90.3$}&{$0.07$}
          &{$92.3$}&{$0.01$}\\
\shline
\end{tabular}
\end{table*}

In Fig.~\ref{fig_ClutterDemos}, we demonstrate the qualitative comparison with the CombPaths model  and the RSF segmentation model on four images interrupted by different numbers of straight segments of white color. The number of interrupting segments increases from rows $1$ to $4$. These straight segments are able to produce strong edge saliency features, thus introducing unexpected influence to the results associated with the CombPaths model and the RSF segmentation model. In row $1$, all the compared segmentation models have achieved accurate results since only few interrupting segments are added to the images. However, as the number of the interrupting segments increasing, the results from the CombPaths model and the RSF segmentation model, as depicted in columns $1$ and $2$, pass through some interrupting segments exhibiting strong edge-based saliency features, such that some portions of the target boundaries are missed. In contrast,  the proposed RegionGeo-Polygon and RegionGeo-SC models are capable of generating favorable segmentation contours, which are insensitive to the influence from the interrupting segments,  as shown in columns $3$ and $4$, due to the use of the region-based appearance model for the construction of the Randers geodesic paths. In this experiment, the Bhattacharyya coefficient of two histograms is taken as the region-based appearance term for the proposed models.

Fig.~\ref{fig_texture} depicts the qualitative comparison between the CombPaths model and the proposed RegionGeo-Polygon model on two synthetic texture images, where the segmentation results are respectively shown in columns $2$ and $3$. The original images obtained from~\citep{jung2012nonlocal} are illustrated in column $1$. One can see that the CombPaths model suffers from the shortcuts problem while the proposed RegionGeo-Polygon model is able to find the correct segmentation results. In this experiment, we exploit the Bhattacharyya coefficient of histograms as the region-based appearance model for constructing the Randers geodesic metrics.

\subsubsection{Quantitative Comparison}
In the proposed interactive segmentation models, the landmark points  serve as a hard constraint in the course of the image segmentation process.  In order to evaluate its performance against the locations of the landmark points, we implement the proposed models $30$ times per test image.   For this purpose, we apply $30$ times a randomly point sampling procedure\footnote{The point sampling procedure used here is similar to the strategy introduced in~\citep[Section 6.5]{mille2015combination}.} to the ground truth boundary of each test image, thus yielding $30$ different sets of landmark points. In each test, a set of landmark points is leveraged for initializing the models. The landmark points in each set are distributed in a counterclockwise order along the ground truth boundary of the test image. Note that a constraint is imposed to the scale between a pair of successive landmark points, such that the euclidean length of the portion of ground truth boundary between two successive sampled points should be higher than a given threshold value~\citep{mille2015combination}.  In our experiments, we set this threshold value as $0.3\kL_{\rm gt}/m$, where $\kL_{\rm gt}$ represents the euclidean length of the ground truth boundary and where $m$ is the number of landmark points associated to the test image. For fair comparison, in each test all the compared models are initialized using the identical landmark points. 

Table~\ref{table_NatureImages} summarizes the mean and the standard deviation values of the Jaccard scores over the $30$ runs per image. The evaluation results in rows $1$ to $11$ of Table~\ref{table_NatureImages} are tested in the images used in Figs.~\ref{fig:DemoComb},~\ref{fig:AsyTube},~\ref{fig:tubeParas},~\ref{fig_EffectsTensorField} and~\ref{fig_NatureImagesComparison}, respectively. These images are sampled from the Grabcut dataset~\citep{rother2004grabcut} and the BSD500 dataset~\citep{arbelaez2011contour}. We can see that the proposed RegionGeo-Polygon and RegionGeo-SC models, have obtained non-trivial improvement to the CombPaths, Riverbed and RSF segmentation models on  most of the test images. The quantitative evaluation results exhibit that the integration of the region-based appearance models and the edge-based features indeed encourages more satisfactory segmentations, especially when the magnitude values of the image gradients are not salient enough along the target boundaries.  The RegionGeo-SC model slightly outperforms, due to the initial contours from the simple closed contour construction method being more relevant than those from the polygon construction method. One can point that  all the compared models have achieved very high Jaccard scores in the FLOWER and STONE2 images, as shown in the second and fourth rows of Table~\ref{table_NatureImages},  since the target boundaries in these images are  well defined by the edge saliency features. Hence the computed optimal paths can appropriately extract the target boundaries. We further observe that the Riverbed model and the RSF segmentation model obtain low Jaccard scores in the OCELOT and FLOWER2  images, whose edge saliency features are not even along the target boundaries. In addition, the absence of the simplicity constraint for the segmentation contours from the RSF segmentation model may greatly increase the risk of the self-tangency problem.

Figs.~\ref{fig_boxPlot_Synthetic} and~\ref{fig_boxPlot_Clutter} visualize the statistics of the Jaccard scores of different models,  evaluated on the images shown in Figs.~\ref{fig:unreliableEdges} and~\ref{fig_ClutterDemos}, respectively. We again randomly sample $30$ groups of landmark points per test image, and then collect all the Jaccard scores of all the compared models. Among these tested models, the proposed RegionGeo-Polygon and RegionGeo-SC models indeed obtain the best performance in both accuracy and robustness.

\begin{figure}[t]
\centering
\includegraphics[width=0.95\linewidth]{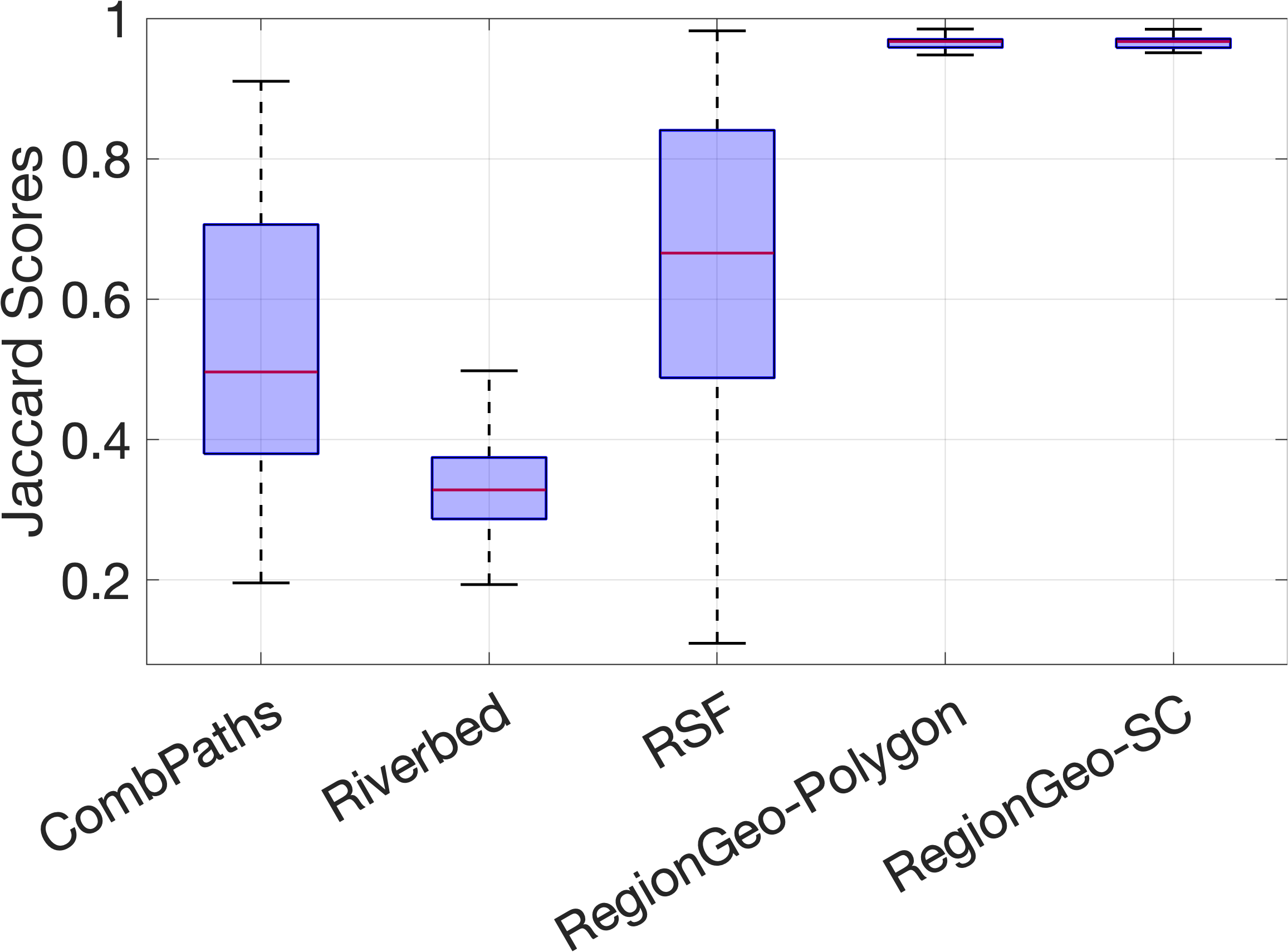}
\caption{Box plot for the Jaccard scores evaluated in the synthetic images shown in Fig.~\ref{fig:unreliableEdges} over $30$ runs per image.}
\label{fig_boxPlot_Synthetic}
\end{figure}

\begin{figure}[t]
\centering
\includegraphics[width=0.95\linewidth]{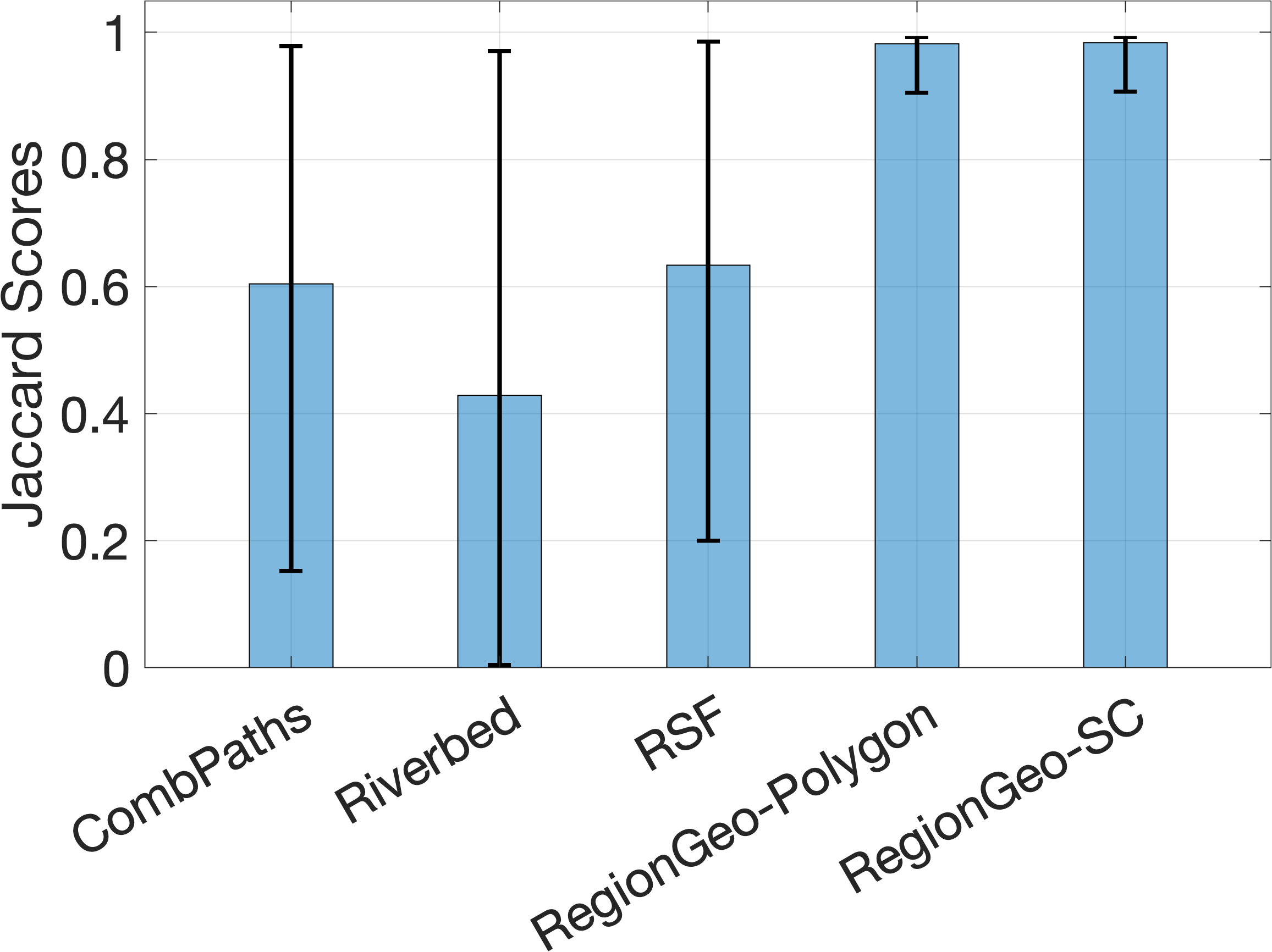}
\caption{The visualization for the statistics (i.e.\ the average, maximum and minimum values) of the Jaccard scores evaluated in the blurred images shown in Fig.~\ref{fig_ClutterDemos}.}
\label{fig_boxPlot_Clutter}
\end{figure}

\subsubsection{Computation Time}
The computation complexity in our model mainly involve four parts: the construction of the initial contour,  the estimation of the region-based energy gradients, the construction of tubular domains, and the extraction of geodesic paths. The last two parts can be efficiently solved by the fast marching methods whose computation complexity has been analyzed in Section~\ref{subsec_FMSolver}. Here we report here the computation time for the image (with $450\times 600$ grid points) in row $4$ of Fig.~\ref{fig_ClutterDemos}. We apply the polygon construction method for the initial contour which cost $0.18$ seconds with three landmark points. The contour evolution process requires around $3$ seconds to stabilize. This test was conducted on an Intel Core i9 $3.6$GHz architecture with $96$GB RAM, using a C++ implementation for the FM-ASR method and a MATLAB implementation for the computation of regional energy gradients of the Bhattacharyya coefficient-based term.

\section{Conclusion}
\label{sec_conclusion}
In this paper, we propose a new region-based Randers geodesic model based on the eikonal PDE framework, which is an efficient solution to the active contour problem involving a regional appearance term. The core contributions of this paper lie at (i) the theoretical convergence analysis of the proposed geodesic model, (ii) the construction of  Randers geodesic metrics, by which the region-based appearance term can be transformed to a weighted curve length in terms of a tubular domain, and (iii) a practical implementation of the introduced Randers geodesic model enhanced by various heuristics for practical  interactive image segmentation, in conjunction with a collection of landmark points placed at the target boundary.      
 The comparison results against state-of-the-art optimal paths-based segmentation models illustrate the  advantages of taking the region-based appearance models for the computation of minimal geodesic paths and for image segmentation.
 
 Future works will be devoted to (i) extending the proposed region-based Randers geodesic model to 3D volume segmentation in conjunction with the surface reconstruction method~\citep{ardon2006fast}, and (ii) integrating the proposed geodesic model with geometric shape priors for solving various medical image segmentation problems. 
 
\section*{Acknowledgment}
The authors would like to thank all the anonymous reviewers for their invaluable  suggestions to improve this manuscript. This work is in part supported by the Shandong Provincial Natural Science Foundation (NO.~ZR2022YQ64), a project in QLUT (NO.~2022PY11), the National Natural Science Foundation of China (NOs.~62171125, 62172243), the Intergovernmental Cooperation Project of the National Key Research and Development Program of China (NO.~2022YFE0116700), and by the French government under management of Agence Nationale de la Recherche as part of the ``Investissements d'avenir'' program, reference ANR-19-P3IA-0001 (PRAIRIE~3IA~Institute).

\appendices

\section{Proof of \cref{th:geodesic_tubular}}
\label{Appendix:geodesic_tubular}

This appendix is devoted to the proof of \cref{th:geodesic_tubular}. We first estimate in \cref{prop:bounded_curvature} the curvature of a minimal path, which is bounded (but no necessarily continuous) under our assumptions. The Lagrangian formalism is used to describe the geometry, rather than the metric formalism \eqref{eq_RandersLength}, in order to benefit from the strict convexity property \eqref{eq:Lag_UCvx}. 
In the application of this result, at the end of this section and following common practice, the Lagrangian is eventually defined as half the squared metric.  
We then show in \cref{prop:curve_in_tube} that a path with bounded curvature lying within a thin band can be parametrized as desired.

\begin{proposition}
\labelx{prop:bounded_curvature}
Let $\Omega \subset \bR^d$ be an open convex domain, equipped with a Lagrangian $\cL : \overline \Omega \times \bR^d \to [0,\infty[$, with a continuous derivative $\partial_2 \cL$ w.r.t.\ the second variable, and such that: 
\begin{align}
	\label{eq:Lag_Lip}
	&\cL(\fx,\dfx) \leq (1+C_0\|\fx-\fy\|) \cL(\fy,\dfx),\\
	\label{eq:DLag_Lip}
	&\|\partial_2 \cL(\fx,\dfx) - \partial_2 \cL(\fy,\dfx)\| \leq C_1 \|\fx-\fy\| \|\dfx\|, \\
	\label{eq:DLag_bounded}
	&\| \partial_2 \cL(\fx,\dfx)\| \leq C_2 \|\dfx\|,\\
	\label{eq:Lag_UCvx}
	&\cL(\fx,\dfx+\dfy) \geq \cL(\fx,\dfx) + \<\partial_2 \cL(\fx,\dfx), \dfy\> + \tfrac 1 2 c_0 \|\dfy\|^2, \\
	\label{eq:Lag_Homog}
	&\cL(\fx,\lambda \dfx) = \lambda^2 \cL(\fx,\dfx).
\end{align}
for all $\fx,\fy \in \overline \Omega$, $\dfx,\dfy\in \bR^d$, $\lambda\geq 0$, where $C_0,C_1,C_2,c_0>0$ are constants. 
Define  
\begin{align}
\label{eq:action}
	\cS(\gamma) &:= \int_0^1 \cL(\gamma(t), \gamma'(t)) \diff t,\\
\label{eq:action_min}
	\cS_{\min}(\fx_0,\fx_1) &:= 
	\min\big\{ \cS(\gamma) 
	\mid 
	\gamma(0)=\fx_0, \gamma(1)=\fx_1
	\}
\end{align}
where $\gamma \in H^1([0,1],\overline \Omega)$ and $\fx_0,\fx_1 \in \overline\Omega$.
Then this minimum is attained, and any minimizer $\gamma$ satisfies
\begin{align}
\label{eq:action_x2}
c_0 \|\fx_0-\fx_1\|^2 &\leq \cS_{\min}(\fx_0,\fx_1) \leq \tfrac 1 2 C_2 \|\fx_0-\fx_1\|^2\\
\label{eq:cst_speed}
	\cL(\gamma(t),\gamma'(t)) &= \cS_{\min}(\fx_0,\fx_1),\\
\label{eq:speed_bounded}
	c_1 \|\fx_0-\fx_1\|  &\leq \|\gamma'(t)\|   \leq c_1^{-1} \|\fx_0-\fx_1\|, \\
\label{eq:curvature_bounded}
	\Lip(\gamma')  & \leq K \|\fx_0-\fx_1\|^2,
\end{align}
for all $t \in [0,1]$, 
where the constants 
$K,c_1>0$ depend on $c_0,C_0,C_1,C_2$.
In particular, the curvature $\kappa := \det(\gamma',\gamma'')/\|\gamma'\|^3$ satisfies $\|\kappa\|_\infty \leq K/c_1^2$. 
\end{proposition}

\begin{proof}
We first observe that, for all $\fx \in \overline \Omega$ and $\dfx \in \bR^2$,
\begin{equation}
\label{eq:Lag_UL}
	c_0 \|\dfx\|^2 \leq \cL(\fx,\dfx) = \tfrac 1 2 \<\partial_2 \cL(\fx,\dfx),\dfx\> \leq \tfrac 1 2 C_2 \|\dfx\|^2,
\end{equation}
using (i) the uniform convexity~\eqref{eq:Lag_UCvx} with $\dfx=0$ and noting that $\partial_2 \cL(\fx,\bzero)=\bzero$ by~\cref{eq:DLag_bounded}, and (ii) Euler's identity for $2$-homogeneous functions. 
It follows that $c_0\|\gamma'\|^2_2 \leq \cS(\gamma) \leq \tfrac 1 2 C_2 \|\gamma'\|_2^2$
for any $\gamma \in H^1([0,1],\overline \Omega)$. This implies \eqref{eq:action_x2} using (i) $\|\gamma(0)-\gamma(1)\|^2\leq \|\gamma'\|_2^2$ by e.g.\ Jensen's inequality for the left estimate, and (ii) choosing the straight line $\gamma_0(t) := (1-t) \fx_0 + t\fx_1$ for the right estimate, which is a valid path since $\overline \Omega$ is convex.

From the compact embedding $H^1([0,1]) \subset C^0([0,1])$ \citep[Theorem 1.6 due to Sobolev and Rellich]{dacorogna2007direct}, and the weak compactness of the unit ball of $H^1$ (Banach-Alaoglu theorem), we obtain that there is a minimizing sequence $(\gamma_n)_{n \geq 0}$ for~\cref{eq:action_min} which is converging, both weakly and uniformly, to a limit $\gamma_* \in H^1([0,1],\overline \Omega)$. Define
\begin{equation}
	\cS(\gamma,\sigma) := \int_0^1 \<\gamma'(t),\sigma(t)\>  - \cH(\gamma(t),\sigma(t)) \diff t,
\end{equation}
where $\cH(\fx,\cfx):=\sup\{\<\cfx,\dfx\> - \cL(\fx,\dfx)\mid \dfx \in \bR^d\}$ denotes the Legendre-Fenchel dual of $\cL$, referred to as the Hamiltonian, which is continuous. (Indeed, the $\cH$ inherits the Lipchitz-like regularity~\eqref{eq:Lag_Lip} w.r.t.\ $\fx$, and by~\cref{eq:Lag_UCvx} it has a $c_0$-Lipschitz gradient w.r.t.\ $\dfx$.)
By construction $\cS(\gamma) = \sup\{\cS(\gamma,\sigma)\mid \sigma \in C^0([0,1],\bR^d)\}$, and by weak and uniform convergence $\cS(\gamma_n,\sigma)\to \cS(\gamma_*,\sigma)$ for any continuous $\sigma$, hence $\cS(\gamma_*) \leq \lim \cS(\gamma_n) = \cS_{\min}(\fx_0,\fx_1)$ and therefore $\gamma_*$ minimizes \eqref{eq:action_min}.
		
We now fix a minimizer $\gamma$, and observe that by the $2$-homogeneity property \eqref{eq:Lag_Homog} it must be parametrized in such way that the Lagrangian $\cL(\gamma,\gamma')$ is constant along the trajectory, establishing \eqref{eq:cst_speed} for a.e.\ $t\in [0,1]$. Thus~\cref{eq:speed_bounded} follows,
using Eqs.~\eqref{eq:action_x2} and~\eqref{eq:Lag_UL}, with $c_1 := \sqrt{2c_0/C_2}$.
	
Let $\eta := \gamma-\gamma_0$, and note that $\eta(0)=\eta(1)=0$ and $\gamma_0'\equiv \fx_1-\fx_0$ on $[0,1]$. Then 
\begin{align*}
	\int_0^1 \cL(\gamma_0, \gamma_0')\diff t 
	&\geq \int_0^1 \cL(\gamma,\gamma')\diff t \\ &
	= \int_0^1 \cL(\gamma_0 + \eta, \gamma_0' + \eta') \diff t\\
	&\geq (1+C_0\|\eta\|_\infty)^{-1}\int_0^1 \cL(\gamma_0,\gamma_0' + \eta') \diff t,
\end{align*}
using successively (i) the optimality of $\gamma$, (ii) the definitions of $\gamma_0$ and $\eta$, and (iii) the Lipschitz-like property \eqref{eq:Lag_Lip}. Denoting $v(t) := \partial_2 \cL(\gamma_0(t),\fx_1-\fx_0)$ and $v_* := v(1/2)$ we obtain
\begin{align*}
	\int_0^1 \cL(\gamma_0,\gamma_0' + \eta') - \int_0^1 \cL(\gamma_0,\gamma_0') 
	\geq \int_0^1 \<v(t),\eta'\> + c|\eta'|^2 \\
	\geq - \|v-v_*\|_\infty \|\eta'\|_1 + c_0 \|\eta'\|_2^2.
\end{align*}
using successively (i) the uniform convexity \eqref{eq:Lag_UCvx} of $\cL$, and (ii) the fact that $\int_0^1 \<v_*,\eta'\>=\<v_*,\eta(1)-\eta(0)\>=0$. Thus
\begin{align*}
	c_0\|\eta'\|_2^2 &\leq \|v-v_*\|_\infty \|\eta'\|_1 + C_0 \|\eta\|_\infty \int_0^1 \cL(\gamma_0,\gamma_0')\diff t\\
	&\leq \tfrac 1 2 (C_1 + C_0 C_2) \|\fx_0-\fx_1\|^2 \|\eta'\|_2
\end{align*}
using (i) the two previous estimates, and (ii) the H\"{o}lder semi-norm inequality $\|\eta\|_\infty \leq \|\eta'\|_1 \leq \|\eta'\|_2$ on the domain $[0,1]$,  the estimates $\cL(\gamma_0,\gamma_0') \leq \frac 1 2 C_2 \|\fx_1-\fx_0\|^2$ by Eq.~(\ref{eq:Lag_UL}, right), and $\|v(t)-v_*\| \leq \frac 1 2 C_1 \|\fx_1-\fx_0\|^2$ by~\cref{eq:DLag_Lip}.
It follows that $\|\eta'\|_2 \leq C_* \|\fx_0-\fx_1\|^2$ with $C_3 := (C_1 + C_0 C_2)/(2 c_0)$. 

Denote by $m_{[s,t]}$ the averaging operator over the interval $[s,t]$, and observe that $\eta' = \gamma'-m_{[0,1]}(\gamma')$ hence the previous estimate amounts to
$\|\gamma'-m_{[0,1]}(\gamma')\|_{L^2([0,1])} \leq C_3 \|\gamma(0)-\gamma(1)\|^2$. Therefore, for any $0 \leq s < t \leq 1$,
\begin{equation*}
	\|\gamma'-m_{[s,t]}(\gamma')\|_{L^2([s,t])} \leq C_3 \frac{\|\gamma(s)-\gamma(t)\|^2}{\sqrt{t-s}}  \leq C_4 |t-s|^\frac 3 2
\end{equation*}
using (i) the previous estimate applied to $\tilde \gamma(u) := \gamma( (1-u) s + u t )$, (ii) the Lipschitz bound (\ref{eq:speed_bounded}, right) with $C_4 := C_3 \|\fx_0-\fx_1\|^2/c_1^2$. By \cite[Theorem 2.3, using $\alpha=1$]{garcia2005lipschitz}, the mapping $\gamma'$ has Lipschitz regularity, with a Lipschitz constant equivalent to $C_4$, which proves \eqref{eq:curvature_bounded}. The curvature bound follows from its definition \eqref{eqdef:curvature}, which concludes the proof.
\qed
\end{proof}

\begin{proposition}
\labelx{prop:curve_in_tube}
Let $\Omega := \bR \times ]-U,U[$ and let $\gamma \in \Lip( [0,T], \overline \Omega)$ whose absolute curvature is bounded by $\kappa_\rmax$, with $0 < U < (\kappa_\rmax \sqrt 2)^{-1}$ and $\kappa_\rmax \geq \pi$. Denote $\gamma= (\alpha,\beta)$ and assume that that $\alpha(0) = 0$ and $\alpha(T)=1$. Then $\gamma$ admits a reparametrization of the form $\tilde \gamma : t\in [0,1] \mapsto (t,\mu(t))$, with $\Lip(\mu) \leq \tan \theta_\rmax$ where $\theta_\rmax := 2\arcsin(\kappa_\rmax U)$.
\end{proposition}

\begin{proof}
We may assume w.l.o.g.\ that $\gamma$ is parametrized at unit Euclidean speed, thus $T \geq \|\gamma(T) - \gamma(0)\| \geq |\alpha(T)-\alpha(0)| = 1$, and in addition there exists $\theta : [0,T] \to \bR/2 \pi \bZ$ such that $(\alpha'(t),\beta'(t)) = (\cos\theta(t), \sin \theta(t))$ for all $t \in [0,T]$.
By assumption $\Lip(\theta) = \|\theta'\|_\infty \leq \kappa_\rmax$.

Now let $t \in [0,T]$. Assuming that $t \leq T/2$ and $\theta_* := \theta(t) \in [0,\pi/2]$, we obtain
\begin{equation*}
2 U \geq \beta(t+\tfrac{\theta_*}{\kappa_\rmax}) - \beta(t) \geq \int_0^{\tfrac{\theta_*}{\kappa_\rmax}} \sin (\theta_* - s) \diff s = \frac{1 - \cos(\theta_*)} {\kappa_\rmax}.
\end{equation*}
Note that $\theta_*/\kappa_\rmax \leq \pi/(2\kappa_\rmax) \leq 1/2 \leq T/2$, hence $t+\tfrac{\theta_*}{\kappa_\rmax}\in [0,T]$.
It follows that $0 \leq \theta_* \leq \theta_\rmax$, and by assumption $\theta_\rmax < 2 \arcsin(1/\sqrt 2) = \pi/2$. 
In general, distinguishing cases depending on the signs of $t-T/2$ and $\theta(t)$, we obtain $\theta(t) \in [-\theta_\rmax,\theta_\rmax] \cup [\pi-\theta_\rmax,\pi+\theta_\rmax]$. 
Observe that these two connected components of $\bR/(2\pi \bZ)$ are disjoint, that $\theta$ is Lipschitz hence continuous, and that $\int_0^T \cos \theta(t) \diff t = \alpha(T)-\alpha(0) = 1 >0$. Therefore $\theta(t) \in [-\theta_\rmax,\theta_\rmax]$ for all $t \in [0,T]$, and the result follows. 
\qed
\end{proof}

\noindent\emph{Conclusion of the proof of~\cref{th:geodesic_tubular}.}
We let $\Omega_0 := \bR \times ]-U,U[$, and note that $\Phi_0 : (t,u) \in \overline \Omega_0 \mapsto \Phi(t \mod 1, u) \in \rT$ is a local $C^2$ diffeomorphism.
Defining 
$\cF_0(\fx,\dfx) := \cF(\Phi_0(\fx), \diff \Phi_0(\fx) \cdot \dfx)$, we obtain a Randers metric on $\Omega_0$, often referred to as the pull-back of the metric $\cF$ by $\Phi_0$, with Lipschitz coefficients (transformed as follows: $\cM_0(\fx) = \diff \Phi_0(\fx)^\top \cM(\Phi_0(\fx)) \diff \Phi_0(\fx)$ and $\omega_0(\fx) = \diff \Phi_0(\fx)^\top \omega(\Phi_0(\fx))$).

Define the Lagrangian $\cL(\fx,\dfx) := \frac 1 2 \cF_0(\fx,\dfx)^2$ on $\Omega_0$. One easily checks that \eqref{eq:Lag_Lip} to \eqref{eq:Lag_Homog} hold on $\Omega_0$. In particular, the uniform convexity \eqref{eq:Lag_UCvx} follows from the fact that the ball $\{\dfx \in \bR^2 \mid \cF_0(\fx,\dfx)\leq 1\}$ associated to the Randers metric at a given point $\fx_0 \in \overline \Omega_0$ is an ellipsoid containing the origin in its iterior (but possibly not centered on the origin, contrary to the Riemannian case), hence is a uniformly convex set. In addition, paths minimizing the Finslerian length $\Length_{\cF_0}(\gamma)$ or the action \eqref{eq:action} between two points are identical, up to a time reparametrization enforcing \eqref{eq:cst_speed}.

Denote by $\fx = \Phi_0(t_\fx,s_\fx)$ the endpoint in \cref{th:geodesic_tubular}, and let $\fx_0 := (t_\fx,\mu_\fx),\ \fx_1 := (t_\fx+1,\mu_\fx)\in \overline \Omega_0$.
Any curve $\cC\in \Gamma_1$ with $\cC(0) = \fx$ can be expressed as a composition $\cC = \Phi_0 \circ \gamma$ where $\gamma \in \Lip([0,1],\overline \Omega_0)$ is such that $\gamma (0) = \fx_0$ and $\gamma (1) = \fx_1$. Furthermore $\Length_\cF(\cC) = \Length_{\cF_0}(\gamma)$ by definition of the pullback metric.

Denote by $\gamma$ a minimal geodesic path for the Lagrangian $\cL$ from $\fx_0$ to $\fx_1$ within $ \overline \Omega_0$. By~\cref{prop:bounded_curvature}, the curvature of $\gamma$ is bounded by a constant, independent of the tube width $0 < U \leq \lfs(\rC)/3$. 
By~\cref{prop:curve_in_tube}, up to reducing $U$, one has $\gamma(t) = (t_\fx+t,\mu(t))$ with $\Lip(\mu) \leq 1$. The curve $\cC(t) := \Phi_0(\gamma(t)) = \rC(t_\fx+t) + \mu(t) \rN(t_\fx+t)$ is the solution to~\cref{eq:argmin_path}, and the proof is complete.

\section{Computation of Edge-based Features}
\label{Appendix_EdgeFeatures}

\noindent\emph{Edge Saliency Features}.
Let $\mathbf{I}=(\cI_1,\cI_2,\cI_3):\overline\Omega\to\bR^3$ be a vector-valued image in the RGB color space. The mollified Jacobian matrix $\mathbf{J}_\sigma(\cdot)$ of the image $\mathbf{I}$, with shape $2\times3$, is
\begin{equation*}
\mathbf{J}_\sigma(\fx):=
\begin{pmatrix}
\partial_x \cK_\sigma\ast	\cI_1,\,&\partial_x \cK_\sigma\ast	\cI_2,\,&\partial_x \cK_\sigma\ast	\cI_3\\
\partial_y \cK_\sigma\ast	\cI_1,\,&\partial_y \cK_\sigma\ast	\cI_2,\,&\partial_y \cK_\sigma\ast	\cI_3
\end{pmatrix}(\fx),
\end{equation*}
where $\cK_\sigma$ is a Gaussian kernel with standard deviation $\sigma$ and $\partial_x \cK_\sigma$ (resp. $\partial_y \cK_\sigma$) denotes the first-order derivative of the kernel $\cK_\sigma$ along the axis $x$ (resp. the axis $y$).

The edge saliency features can be described  through a scalar-valued function $g:\overline\Omega\to\bR^+_0$ which assigns to each point $\fx$ the Frobenius norm of the matrix $\mathbf{J}_\sigma(\fx)$
\begin{equation}
\label{eq_imageGradsColor}
g(\fx):=\sqrt{\sum_{i=1}^3\Big( (\partial_x \cK_\sigma\ast \cI_i)(\fx)^2+(\partial_y \cK_\sigma\ast \cI_i)(\fx)^2\Big)}.
\end{equation}

\noindent\emph{Edge Anisotropy Feature}. As introduced in~\citep{sochen1998general}, the edge anisotropy feature at each point $\fx$ is carried out by a matrix $\mathcal{Q}(\fx)$ defined as 
\begin{equation}
\label{eqdef:cQ}
\cQ(\fx):=\mathbf{J}_\sigma(\fx)\mathbf{J}_\sigma(\fx)^\top+\Id,
\end{equation}
where $\Id$ denotes the identity matrix of shape $d \times d$ ($d$ is the dimension of the image domain). For each point $\fx\in\overline\Omega$, the image gradient vector $\kg(\fx)$ is derived as the eigenvector of  $\mathcal{Q}(\fx)$ which corresponds to its largest eigenvalue.

We finally discuss the special case of a gray level image $\cI:\overline\Omega\to\bR$. Then the function $g$ is obtained as 
\begin{equation}
\label{eq_imageGradsGray}
g=\|\nabla{\cK_{\sigma}}\ast{\cI}\|=\sqrt{(\partial_x \cK_\sigma\ast \cI)^2+(\partial_y \cK_\sigma\ast \cI)^2},
\end{equation}	
where $\nabla{\cK_{\sigma}}$ stands for the standard Euclidean gradient of the Gaussian kernel $\cK_\sigma$. 

At each point $\fx$, the edge anisotropy feature matrix $\mathcal{Q}(\fx):=J_\sigma(\fx)J_\sigma(\fx)^\top+\Id$ is expressed in terms of the vector field $J_\sigma(\fx):=(\nabla{\cK_{\sigma}}\ast{\cI})(\fx)$. Then the gradient vector field $\kg(\fx)$ of a gray level image is taken as the eigenvector of $\mathcal{Q}(\fx)$ corresponding to the largest eigenvalue. For the points $\fx$ such that $\|J_\sigma(\fx)\|\neq 0$, one has $\kg(\fx) = \pm J_\sigma(\fx)/\|J_\sigma(\fx)\|$.

\section{RSF Geodesic Segmentation Model}
\label{appendix_RSF}
The RSF geodesic model proposed by~\citep{duits2018optimal} is a curvature-penalized model, that is defined over an orientation-lifted space. Let $\tilde\bS^1:=\bR/(2\pi\bZ)$ denote the angular space. A point $\fx_{\rm L}=(\fx,\theta)\in\bM:=\overline\Omega\times\tilde\bS^1$ consists of a position $\fx\in\overline\Omega$ and of an angular coordinate  $\theta\in\tilde\bS^1$. In this way, a smooth curve $\gamma:[0,1]\to\overline\Omega$  can be lifted to the space $\bM$, i.e.\ $\gamma_{\rm L}(u)=(\gamma(u),\varpi(u))$, where $\varpi:[0,1]\to\tilde\bS^1$ is a parametric function such that $\gamma^\prime(u)=\dft_{\varpi(u)}\|\gamma^\prime(u)\|$ for any $u\in[0,1]$ and $\dft_{\varpi(u)}:=(\cos\varpi(u),\sin\varpi(u))$ is the unit vector associated to the angle $\varpi(u)\in\tilde\bS^1$. The curve $\gamma$ is referred to as the physical projection of $\gamma_{\rm L}$.

The RSF model takes into account an orientation-lifted geodesic metric $\cF_{\rm RSF}:\mathbb{M}\times\bR^3\to[0,\infty]$, which reads
\begin{equation*}
\cF_{\rm RSF}(\fx_{\rm L},\dot{\fx}_{\rm L})=
\begin{cases}
\sqrt{\|\dfx\|^2+(\nu\,\dot\theta)^2},&\text{if~}\dfx=\dft_\theta\|\dfx\|,\\
\infty,&\text{otherwise},
\end{cases}	
\end{equation*}
where $\dot{\fx}_{\rm L}:=(\dfx,\dot{\theta})\in\bR^3$ is a tangent vector to $\bM$ at  point $\fx_{\rm L}\in\bM$. In this case, the length of a smooth curve $\gamma_{\rm L}$ can be formulated as
\begin{equation}
\int_0^1\cP_{\rm RSF}(\gamma(u),\varpi(u))\cF_{\rm RSF}	(\gamma_{\rm L}(u),\gamma^\prime_{\rm L}(u))du,
\end{equation}
where $\cP_{\rm RSF}:\bM\to\bR^+$ is a data-driven cost function. The value of $\cP_{\rm RSF}(\fx,\theta)$  is low if $\fx$ is close to an edge and the vector $\dft_\theta$ is tangent to that edge. As implemented in~\citep{chen2017global}, the computation of $\cP_{\rm RSF}$ is carried out by the Canny-like steerable filter~\citep{jacob2004design} for the following experiments.
\vspace{-0.5\baselineskip}\\
 
\noindent\emph{RSF Segmentation Method.}
We leverage the RSF geodesic paths for interactive image segmentation, by connecting  a set of $m$ landmark points $\fp_k$  with known order as input. This is implemented by exploiting a slight variant of the  closed contour detection scheme introduced in~\citep{chen2017global}.  Firstly, each landmark point $\fp_k\in\overline\Omega$ is lifted to the orientation-lifted space $\bM$ such that $\fp_k$ corresponds to two points $(\fp_k,\theta_k),\,(\fp_k,\theta_k+\pi)\in\bM$.  The angle $\theta_k\in\tilde\bS^1$ characterizes the orientation that an image edge should have at the point $\fp_k$, which can be estimated by solving 
\begin{equation*}
\theta_k:=\min_{\theta\in[0,\pi)}\cP_{\rm RSF}(\fp_k,\theta).
\end{equation*}
At the initialization stage (i.e. $k=1$),  we can get four RSF geodesic paths by choosing one point from $\{(\fp_1,\theta_1), (\fp_1,\theta_1+\pi)\}$ as source point and choosing one point from $\{(\fp_{2},\theta_{2}),(\fp_{2},\theta_{2}+\pi)\}$ as target point, respectively. Among them, we select the geodesic path with the smallest geodesic distance, whose source and target points are respectively denoted by $(\fp_1,\theta_1^*)$ and $(\fp_{2},\theta_{2}^*)$. This can be done by performing the Hamiltonian FMM~\citep{mirebeau2018fast} just once, as introduced in~\citep{chen2017global}. When $k=2$, we can obtain two curvature-penalized geodesic paths joining the point $(\fp_{2},\theta_{2}^*)$ to the  orientation-lifted points $(\fp_{3},\theta_{3})$ and $(\fp_{3},\theta_{3}+\pi)$, respectively. From them, we again choose the minimal one in terms of  geodesic distance associated to the RSF metric. In case $k\geq3$, we iterate the similar step with $k=2$ until a geodesic path is tracked which takes the point $(\fp_{1},\theta_{1}^*)$ as target point. In this way, we obtain $m$ orientation-lifted geodesic paths, whose physical projection curves  form  a closed contour passing through all the landmark points $\fp_k$ with $1\leq k\leq m$.

\bibliographystyle{spbasic}      
\bibliography{eikonalActiveContours.bib}   

\end{document}